\def\year{2023}\relax
\documentclass[letterpaper]{article} 
\usepackage{aaai22}  
\usepackage{times}  
\usepackage{helvet}  
\usepackage{courier}  
\usepackage[hyphens]{url}  
\usepackage{graphicx} 
\usepackage{natbib}  
\usepackage{caption} 
\DeclareCaptionStyle{ruled}{labelfont=normalfont,labelsep=colon,strut=off} 
\usepackage{algorithm}
\usepackage{longtable}
\usepackage{algorithmic}
\usepackage{float}
\usepackage{amsmath}\allowdisplaybreaks
\usepackage{amssymb}
\usepackage{amsthm}
\DeclareMathOperator*{\argmax}{argmax}

\newtheorem{assumption}{Assumption}
\newtheorem{lemma}{Lemma}
\newtheorem{theorem}{Theorem}

\usepackage{thm-restate}

\usepackage{xcolor}
\usepackage[colorlinks]{hyperref}
\AtBeginDocument{\hypersetup{
    citecolor=blue,
	linkcolor=blue,
	filecolor=magenta,      
	urlcolor=blue,
}}

\newcommand{\compilehidecomments}{false}
\newcommand{\Pa}{\boldsymbol{\it Pa}}
\newcommand{\pa}{\boldsymbol{\it pa}}
\newcommand{\Ch}{\boldsymbol{\it Ch}}

\newcommand{\doi}{{\it do}}
\newcommand{\pdf}{{\it pdf}}

\newcommand{\Sphere}{{\it Sphere}}

\newcommand{\E}{\mathbb{E}}
\newcommand{\I}{\mathbb{I}}

\newcommand{\bb}{\boldsymbol{b}}

\newcommand{\bS}{\boldsymbol{S}}
\newcommand{\bs}{\boldsymbol{s}}

\newcommand{\bU}{\boldsymbol{U}}

\newcommand{\bV}{\boldsymbol{V}}
\newcommand{\bv}{\boldsymbol{v}}

\newcommand{\bW}{\boldsymbol{W}}
\newcommand{\bX}{\boldsymbol{X}}

\newcommand{\bx}{\boldsymbol{x}}

\newcommand{\bz}{\boldsymbol{z}}

\newcommand{\btheta}{\boldsymbol{\theta}}
\newcommand{\bbvarepsilon}{\boldsymbol{\varepsilon}}
\newcommand{\bgamma}{\boldsymbol{\gamma}}

\newcommand{\cP}{{\mathcal{P}}}

\usepackage{newfloat}
\usepackage{listings}
\floatstyle{ruled}
\newfloat{listing}{tb}{lst}{}
\floatname{listing}{Listing}

\title{Combinatorial Causal Bandits}
\author{
    Shi Feng,\textsuperscript{\rm 1}
    Wei Chen\textsuperscript{\rm 2}
}
\affiliations{
    \textsuperscript{\rm 1}Institute for Interdisciplinary Information Sciences, Tsinghua University, Beijing, China\\
    \textsuperscript{\rm 2}Microsoft Research, Beijing, China\\

    shifeng-thu@outlook.com, weic@microsoft.com
}

\begin{document}

\ifthenelse{ \equal{\compilehidecomments}{false} }{%
	\newcommand{\shi}[1]{}
	\newcommand{\wei}[1]{}
}{
	\newcommand{\shi}[1]{{\color{red} [\text{Shi:} #1]}}
	\newcommand{\wei}[1]{{\color{blue} [\text{Wei:} #1]}}
}

\newcommand{\compilefullversion}{true}
\ifthenelse{\equal{\compilefullversion}{false}}{%
	\newcommand{\OnlyInFull}[1]{}
	\newcommand{\OnlyInShort}[1]{#1}
}{%
	\newcommand{\OnlyInFull}[1]{#1}%
	\newcommand{\OnlyInShort}[1]{}%
}%

\maketitle

\begin{abstract}
In combinatorial causal bandits (CCB), the learning agent chooses at most $K$ variables in each round to intervene, collects feedback from the observed variables, 
	with the goal of minimizing expected regret on the target variable $Y$.
We study under the context of binary generalized linear models (BGLMs) with a succinct parametric representation of the causal models.
We present the algorithm BGLM-OFU for Markovian BGLMs (i.e. no hidden variables) based on the maximum likelihood estimation method, 
	and show that it achieves $O(\sqrt{T}\log T)$ regret, where $T$ is the time horizon.
For the special case of linear models with hidden variables, we apply causal inference techniques such as the do-calculus to convert the original model into a Markovian model, and
	then show that our BGLM-OFU algorithm and another algorithm based on the linear regression both solve such linear models with hidden variables.
Our novelty includes (a) considering the combinatorial intervention action space and
	the general causal models including ones with hidden variables,
	(b) integrating and adapting techniques from diverse studies such as generalized linear bandits and online influence maximization, and
	(c) avoiding unrealistic assumptions (such as knowing the joint distribution of the parents of $Y$ under all interventions) and regret factors exponential to causal graph size in
		prior studies.
\end{abstract}

\section{Introduction}

Causal bandit problem is first introduced in \cite{lattimore2016causal}.
It consists of a causal graph $G=({\bX}\cup\{Y\},E)$ indicating the causal relationship among the observed variables, where
	the structure of the graph is known but the underlying probability distributions governing the causal model are unknown.
In each round, the learning agent selects one or a few variables in $\bX$
	to intervene, gains the reward as the output of $Y$, and observes the values of all
	variables in $\bX\cup \{Y\}$.
The goal is to either maximize the cumulative reward over $T$ rounds, or find the intervention closest to the optimal one after $T$ rounds.
The former setting, which is the one our study focuses on, 
	is typically translated to minimizing cumulative regret, which is defined as the difference in reward between always playing the optimal intervention
	and playing interventions according to a learning algorithm.
Causal bandits can be applied in many settings that include causal relationships, 
	such as medical drug testing, policy making, scientific experimental process, performance tuning, etc.
Causal bandit is a form of multi-armed bandit (cf. \cite{LS20}), with the main difference being that causal bandits may use the causal relationship and more observed feedback
	to achieve a better regret. 

All the causal bandit studies so far~\cite{lattimore2016causal,sen2017identifying,nair2020budgeted,lu2020regret,maiti2021causal} focus on the case where the number of possible
	interventions is small.
However, in many scenarios we need to intervene on a set of variables and the number of possible choices of sets is large.
For example, in tuning performance of a large system, one often needs to tune a large number of parameters simultaneously to achieve the best system performance, and in drug
	testing, a combination of several drugs with a number of possible dosages needs to be tested for the best result.
Intervening on a set of variables raises new challenges to the learning problem, since the number of possible interventions is exponentially large to the size of the
	causal graph.
In this paper, we address this challenge and propose the new framework of combinatorial causal bandits (CCB) and its solutions.
In each round of CCB, the learning agent selects a set of at most $K$ observed variables in $\bX$ to intervene instead of one variable.
Other aspects including the feedback and reward remain the same.

We use the binary generalized linear models (BGLMs) to give the causal model a succinct parametric representation where all variables are binary.
Using the maximum likelihood estimation (MLE) method, we design an online learning algorithm BGLM-OFU for the causal models without hidden variables
	(called Markovian models), and the algorithm achieves 
	$O(\sqrt{T}\log T)$ regret, where $T$ is the time horizon.
The algorithm is based on the earlier study on generalized linear bandits~\cite{li2017provably}, but our BGLM model is more general and thus requires new
	techniques (such as a new initialization phase) to solve the problem.
Our regret analysis also integrates results from online influence maximization~\cite{li2020online,zhang2022online} in order to obtain the final regret bound.

Furthermore, for the binary linear models (BLMs), we show how to transform a BLM with hidden variables into one without hidden variables by 
	utilizing the tools in causal inference such as do-calculus, and thus we can handle causal
	model even with hidden variables in linear models.
Then, for BLMs, we show (a) the regret bound when applying the BGLM-OFU algorithm to the linear model, and (b) a new algorithm and its regret bound based on the linear regression method.
We show a tradeoff between the MLE-based BGLM-OFU algorithm and the linear-regression-based algorithm on BLMs: 
	the latter removes the assumption needed by the former but has an extra factor in the regret bound. We demonstrate effectiveness of our algorithms by experimental evaluations in \OnlyInFull{Appendix~\ref{app.simulation}}\OnlyInShort{appendix}. Besides BLMs, we give similar results for linear models with continuous variables. Due to space limits, this part is put in \OnlyInFull{Appendix~\ref{app.exttocontinuous}}\OnlyInShort{appendix}.
	\OnlyInShort{Due to space limits, our appendix is included in the full version \cite{fullversion} on arXiv.}
	
In summary, our contribution includes (a) proposing the new combinatorial causal bandit framework, 
	(b) considering general causal models including ones with hidden variables, 
	(c) integrating and adapting techniques from diverse studies such as generalized linear bandits and online influence maximization, and
	(d) achieving competitive online learning results without unrealistic assumptions and regret factors exponential to the causal graph size as appeared in prior studies.
The intrinsic connection between causal bandits and online influence maximization revealed by this study may further benefit researches in both directions.

\section{Related Works}
\label{sec:related}

{\bf Causal Bandits.} The causal bandit problem is first defined in \cite{lattimore2016causal}. 
The authors introduce algorithms for a specific parallel model and 
	more general causal models to minimize the simple regret, defined as the gap between the optimal reward and the reward of the action obtain after $T$ rounds. 
A similar algorithm to minimize simple regret has also been proposed in \cite{nair2020budgeted} for a graph with no backdoor. 
To optimally trade-off observations and interventions, they have also discussed the budgeted version of causal bandit when interventions are more expensive than observations.
\citet{sen2017identifying} use the importance sampling method to propose an algorithm to optimize simple regret 
	when only soft intervention on a single node is allowed. 
These studies all focus on simple regret for the pure exploration setting, while 
	our work focuses on cumulative regret. 
There are a few studies on the cumulative regret of causal bandits~\cite{lu2020regret,nair2020budgeted,maiti2021causal}. 
However, the algorithms in \cite{lu2020regret} and \cite{nair2020budgeted} are both under an unrealistic assumption that for every intervention, 
	the joint distribution of the parents of the reward node $Y$ is known. 
The algorithm in \cite{maiti2021causal} does not use this assumption, but it only works on Markovian models without hidden variables.
The regrets in~\cite{lu2020regret,maiti2021causal} also have a factor exponential to the causal graph size.
\citet{yabe2018causal} designs an algorithm by estimating each $P(X|\Pa(X))$ for $X\in\bX\cup\{Y\}$ that focuses on simple regret but works for combinatorial settings. However, it requires the round number $T\ge \sum_{X\in \bX} 2^{|\Pa(X)|}$, which is still unrealistic.
We avoid this by working on the BGLM model with a linear number of parameters, and it results in completely different techniques. \citet{lee2018structural, lee2019structural, lee2020characterizing} also consider the combinatorial settings, but they focus on empirical studies, while we provide theoretical regret guarantees. Furthermore, studying causal bandit problem without the full casual structure is an important future direction.
One such study, \cite{lu2021causal}, exists but it is only for the atomic setting and 
	has a strong assumption that $\left|\Pa(Y)\right|=1$.

\noindent{\bf Combinatorial Multi-Armed Bandits.} 
CCB can be viewed as a type of combinatorial multi-armed bandits (CMAB)~\cite{pmlr-v28-chen13a,ChenWYW16}, but the feedback model is not the semi-bandit model studied
	in~\cite{pmlr-v28-chen13a,ChenWYW16}. In particular, in CCB individual random variables cannot be treated as base arms,
	because each intervention action changes the distribution of the remaining variables, 
	violating the i.i.d assumption for base arm random variables across different rounds.
Interestingly, it has a closer connection to recent studies on online influence maximization (OIM) with node-level feedback~\cite{li2020online,zhang2022online}.
These studies consider influence cascades in a social network following either the independent cascade (IC) or the linear threshold (LT) model.
In each round, one selects a set of at most $K$ seed nodes, observes the cascade process as which nodes are activated in which steps, and obtains the reward as the total
	number of nodes activated.
The goal is to learn the cascade model and minimize the cumulative regret.
Influence propagation is intrinsically a causal relationship, and thus OIM has some intrinsic connection with CCB, and our study does borrow techniques from
	 \cite{li2020online,zhang2022online}.
However, there are a few key differences between OIM and CCB, such that we need adaptation and integration of OIM techniques into our setting:
	(a) OIM does not consider hidden variables, while we do consider hidden variables for causal graphs; 
	(b) OIM allows the observation of node activations at every time step, but in CCB we only observe the final result of the variables; and
	(c) current studies in \cite{li2020online,zhang2022online} only consider IC and LT models, while we consider the more general BGLM model, which includes IC model (see \cite{zhang2022online} for transformation from IC model to BGLM) and LT model as two special
		cases.

\noindent{\bf Linear and Generalized Linear Bandits.} Our work is also based on some previous online learning studies on linear and generalized linear bandits. 
\citet{abbasi2011improved} propose an improved theoretical regret bound for the linear stochastic multi-armed bandit problem. 
Some of their proof techniques are used in our proofs. 
\citet{li2017provably} propose an online algorithm and its analysis based on MLE for generalized linear contextual bandits, and our MLE method is adapted from this study.
However, our setting is much more complicated, in that (a) we have a combinatorial action space, and (b) we have a network of generalized linear relationships while they
	only consider one generalized linear relationship.
As the result, our algorithm and analysis are also more sophisticated.

\section{Model and Preliminaries}
\label{sec.model}

Following the convention in causal inference literature (e.g. \cite{Pearl09}), 
we use capital letters ($U,X,Y\ldots$) to represent variables, and their corresponding lower-case letters to represent their values.
We use boldface letters such as $\bX$ and $\bx$ to represent a set or a vector of variables or values.

A {\em causal graph} $G=({\bU} \cup {\bX}\cup \{Y\},E)$ is a directed acyclic graph where ${\bU} \cup {\bX}\cup \{Y\}$ are sets of nodes with $\bU$ being the
	set of unobserved or hidden nodes, ${\bX}\cup \{Y\}$ being the set of observed nodes, $Y$ is a special target node with no outgoing edges, and $E$ is the set of directed edges 
	connecting nodes in ${\bU} \cup {\bX}\cup \{Y\}$.
For simplicity, in this paper we consider all variables in ${\bU} \cup {\bX}\cup \{Y\}$ are $(0,1)$-binary random variables.
For a node $X$ in the graph $G$, we call its in-neighbor nodes in $G$ the {\em parents} of $X$, denoted as $\Pa(X)$, and the values taken by these parent random variables
are denoted $\pa(X)$.

Following the causal Bayesian model, the causal influence from the parents of $X$ to $X$ is modeled as the probability distribution $P(X | \Pa(X))$ for every possible value
	combination of $\Pa(X)$.
Then, for each $X$, the full non-parametric characterization of $P(X | \Pa(X))$ requires $2^{|\Pa(X)|}$ values, which would be difficult for learning.
Therefore, we will describe shortly the succinct parametric representation of $P(X | \Pa(X))$ as a generalized linear model to be used in this paper.

The causal graph $G$ is {\em Markovian} if there are no hidden variables in $G$ and every observed variable $X$ has
	certain randomness not caused by other observed parents.
To model this effect of the Markovian model, in this paper we dedicate random variable $X_1$ to be a special variable
	that always takes the value $1$ and is a parent of all other observed random variables.
We study the Markovian causal model first, and in Section~\ref{sec.hidden} we will consider causal models with more general hidden variable forms.

An {\em intervention} in the causal model $G$ is to force a subset of observed variables ${\bS} \subseteq {\bX}$ to take some predetermined values ${\bs}$, 
	to see its effect on the target variable $Y$.
The intervention on ${\bS}$ to $Y$ is denoted as $\E[Y | \doi({\bS}={\bs})]$.
In this paper, we study the selection of ${\bS}$ to maximize the expected value of $Y$, and our parametric model would have the monotonicity property such that
	setting a variable $X$ to $1$ is always better than setting it to $0$ in maximizing $\E[Y]$, so our intervention would always be setting $\bS$ to all $1$'s,
	for which we simply denote as $\E[Y | \doi({\bS})]$.

In this paper, we study the online learning problem of {\em combinatorial causal bandit}, as defined below.
A learning agent runs an algorithm $\pi$ for $T$ rounds.
Initially, the agent knows the observed part of the causal graph $G$ induced by observed variables ${\bX} \cup \{Y\}$, but does not know the 
	probability distributions $P(X | \Pa(X))$'s.
In each round $t=1,2,\ldots, T$, the agent selects at most $K$ observed variables in $\bX$ for intervention, obtains the observed $Y$ value as the reward, and
	observes all variable values in ${\bX}\cup \{Y\}$ as the feedback.
The agent needs to utilize the feedback from the past rounds to adjust its action in the current round, with the goal of maximizing the cumulative reward from all $T$ rounds.

The performance of the agent is measured by the {\em regret} of the algorithm $\pi$, which is defined as the difference between the expected cumulative reward of the 
	best action ${\bS}^*$ and the cumulative reward of algorithm $\pi$, where ${\bS}^* \in \argmax_{{\bS}\subseteq {\bX}, |{\bS}|=K} \E[Y | \doi({\bS})]$, as given
	below:
\begin{align} \label{eq:regret}
R^\pi(T)&= \mathbb{E}\left[\sum_{t=1}^T(\E[Y | \doi({\bS}^*)] - \E[Y | \doi({\bS}_t^\pi)])\right],
\end{align}
where ${\bS}_t^\pi$ is the intervention set selected by algorithm $\pi$ in round $t$, and the expectation is taking from the randomness of the causal model as well
	as the possible randomness of the algorithm $\pi$.
In some cases our online learning algorithm uses a standard offline oracle that takes the causal graph $G$ and
	the distributions $P(X | \Pa(X))$'s as inputs and outputs a set of nodes $\bS$ that achieves an $\alpha$-approximation
	with probability $\beta$ with $\alpha,\beta \in (0,1]$. 
In this case, we consider the $(\alpha,\beta)$-approximation regret, as in~\cite{ChenWYW16}:
\begin{align} \label{eq:regret1}
\small
R^\pi_{\alpha,\beta}(T)&= \mathbb{E}\left[\sum_{t=1}^T(\alpha\beta\E[Y | \doi({\bS}^*)] - \E[Y | \doi({\bS}_t^\pi)])\right].
\end{align}

As pointed out earlier, the non-parametric form of distributions $P(X | \Pa(X))$'s needs an exponential number of quantities
	and is difficult to learn.
In this paper, we adopt the general linear model as the parametric model, which is widely used in causal inference
	literature \cite{hernan2010causal,garcia2017squeezing,han2017evaluating,arnold2020reflection,Vansteelandt2020AssumptionleanIF}. 
Since we consider binary random variables, we refer to such models as binary general linear models (BGLMs). 
In BGLM, the functional relationship between a node $X$ and its parents $\Pa(X)$ in $G$ 
	is $P(X=1|\Pa(X)=\pa(X)) =f_{X}(\btheta^*_{X}\cdot \pa(X))+\varepsilon_{X}$, 
	where $\btheta^*_{X}$ is the unknown weight vector in $[0,1]^{|\Pa(X)|}$, $\varepsilon_{X}$ is a 
	zero-mean sub-Gaussian noise that ensures that the probability does not exceed $1$, 
	$\pa(X)$ here is the vector form of the values of parents of $X$, and $f_X$ is a scalar and monotonically non-decreasing
	function.
It is worth noticing that our BGLM here is a binary version of traditional generalized linear models \cite{aitkin2009statistical, hilbe2011logistic, sakate2014comparison, hastie2017generalized}. 
Instead of letting $X=f_{X}(\btheta^*_{X}\cdot\pa(X))+\varepsilon_{X}$ directly, we take $f_{X}(\btheta^*_{X}\cdot\pa(X))+\varepsilon_{X}$ as the probability for $X$ to be $1$. 
The vector of all the weights in a BGLM is denoted by $\btheta^*$ and the feasible domain for the 
	weights is denoted by $\Theta$.  
We use the notation $\theta^*_{Z,X}$ to denote the parameter in $\btheta^*$ that corresponds to the edge $(Z,X)$, or equivalently the entry in vector $\btheta^*_X$ that corresponds to
	node $Z$.
We also use notation $\bbvarepsilon$ to represent all noise random variables $(\varepsilon_X)_{X\in \bX\cup Y}$.

A special case of BGLM is the linear model where function $f_X$ is the identity function for all $X$'s, then
	$P(X=1|\Pa(X)=\pa(X)) = \btheta^*_{X}\cdot \pa(X)+\varepsilon_{X}$.
We refer to this model as BLM.
Moreover, when we remove the noise variable $\varepsilon_{X}$, BLM coincides with the {\em linear threshold (LT)} model for influence cascades~\cite{kempe03}
	in a DAG.
In the LT model, each node $X$ has a random threshold $\lambda_X$ uniformly drawn from $[0,1]$, and each edge  $(Z,X)$ has a weight $w_{Z,X} \in [0,1]$, such that
	node $X$ is activated (equivalent to $X$ being set to $1$) when the cumulative weight of its active in-neighbors is at least $\lambda_X$.
It is easy to see that when we set $\theta^*_{Z,X} = w_{Z,X}$, the activation condition is translated to the probability of $X=1$ being exactly $\btheta^*_{X}\cdot \pa(X)$.
It is not surprising that a linear causal model is equivalent to an influence cascade model, since the influence relationship is intrinsically a causal relationship.


\section{Algorithm for BGLM CCB}
\label{sec.glm}

In this section, we present an algorithm that solves the online CCB problem for the Markovian BGLM. 
The algorithm requires three assumptions. 
\begin{assumption}
	\label{asm:glm_3}
	For every $X \in \bX\cup\{Y\}$,
	$f_X$ is twice differentiable. Its first and second order derivatives are upper-bounded by $L_{f_X}^{(1)} >0$ and $L_{f_X}^{(2)} >0$. 
\end{assumption}
Let $\kappa=\inf_{X \in \bX\cup\{Y\}, \bv\in[0,1]^{|\Pa(X)|},||\btheta-\btheta^*_X||\leq 1} \dot{f}_X(\bv\cdot \btheta)$.
\begin{assumption}
\label{asm:glm_2}
We have $\kappa > 0$.
\end{assumption}
\begin{assumption}
\label{asm:glm_4}
There exists a constant $\zeta>0$ such that for any $X\in {\bX}\cup\{Y\}$ and $X'\in \Pa(X)$, 
for any value vector $\bv \in \{0,1\}^{|\Pa(X)\setminus \{X',X_1\} |}$,
the following inequalities hold:
\begin{align}
    &\Pr_{\bbvarepsilon,\bX,Y}\{X'=1|\Pa(X)\setminus \{X',X_1\}=\bv\} \geq \zeta , \label{eq:parentsZero}\\
    &\Pr_{\bbvarepsilon,\bX,Y}\{X'=0|\Pa(X)\setminus \{X',X_1\}=\bv\}\geq \zeta.  \label{eq:parentsOne}
\end{align} 
\end{assumption}
The first two assumptions for our BGLM are also adopted in a previous work on GLM \cite{li2017provably}, 
which ensure that the change of $P(X=1|\pa(X))$ is not abrupt. 
It is worth noting that Assumption \ref{asm:glm_2} only needs the lower bound of the first derivative in the neighborhood of $\theta^*_X$, which is weaker than Assumption 1 in \cite{filippi2010parametric}. 
Finally, Assumption~\ref{asm:glm_4} makes sure that each parent node of $X$ still has a constant probability of taking either $0$ or $1$ even when all other parents of $X$
	already fix their values.
This means that each parent has some independence and is not fully determined by other parents. 
In \OnlyInFull{Appendix~\ref{app:asm_just}}\OnlyInShort{appendix} we give some further justification of this assumption.

We first introduce some notations. 
Let $n,m$ be the number of nodes and edges in $G$ respectively.
Let $D = \max_{X\in{\bX}\cup\{Y\}}|\Pa(X)|$, 
	$L^{(1)}_{\max} = \max_{X\in {\bX}\cup\{Y\}}L_{f_X}^{(1)}$, 
	and $c$ be the constant in Lecu\'{e} and Mendelson's inequality \cite{nie2021matrix} (see Lemma~\OnlyInFull{\ref{lemma:LMInequality}}\OnlyInShort{11} in \OnlyInFull{Appendix~\ref{app:regret_glm}}\OnlyInShort{appendix}). 
Let $(\bX_t,Y_t)$ be the propagation result in the $t^{th}$ round of intervention. It contains $\bV_{t,X}$ and $X^t$ for each node $X\in{\bX}\cup\{Y\}$ where $X^t$ is the propagating result of $X$ and $\bV_{t,X}$ is the propagating result of parents of $X$. Additionally, our estimation of the weight vectors is denoted by $\hat{\btheta}$. 

\begin{algorithm}[t]
\caption{BGLM-OFU for BGLM CCB Problem}
\label{alg:glm-ucb}
\begin{algorithmic}[1]
\STATE {\bfseries Input:}
Graph $G=({\bX}\cup\{Y\},E)$, intervention budget $K\in\mathbb{N}$, parameter $L_{f_X}^{(1)},L_{f_X}^{(2)},\kappa,\zeta$ in Assumption \ref{asm:glm_3} , \ref{asm:glm_2} and \ref{asm:glm_4}.
\STATE Initialize $M_{0,X}\leftarrow {\bf 0}\in\mathbb{R}^{|\Pa(X)|\times |\Pa(X)|}$ for all $X\in {\bX}\cup\{Y\}$, $\delta\leftarrow\frac{1}{3n\sqrt{T}}$, $R\leftarrow\lceil\frac{512D(L_{f_X}^{(2)})^2}{\kappa^4}(D^2+\ln\frac{1}{\delta})\rceil$, $T_0\leftarrow \max\left\{\frac{c}{\zeta^2}\ln\frac{1}{\delta},\frac{(8n^2-16n+2)R}{\zeta}\right\}$ and $\rho\leftarrow\frac{3}{\kappa}\sqrt{\log(1/\delta)}$.
\STATE /* Initialization Phase: */
\STATE Do no intervention on BGLM $G$ for $T_0$ rounds and observe feedback $(\bX_t,Y_t),1\leq t\leq T_0$.
\STATE /* Iterative Phase: */
\FOR{$t=T_0+1,T_0+2,\cdots,T$}
\STATE $\{\hat{\btheta}_{t-1,X},M_{t-1,X}\}_{X\in {\bX}\cup\{Y\}}=
	\text{BGLM-Estimate}((\bX_1,Y_1),\cdots,(\bX_{t-1},Y_{t-1}))$ (see Algorithm \ref{alg:glm-est}).
\STATE Compute the confidence ellipsoid $\mathcal{C}_{t,X}=\{\btheta_X'\in[0,1]^{|\Pa(X)|}:\left\|\btheta_X'-\hat{\btheta}_{t-1,X}\right\|_{M_{t-1,X}}\leq\rho\}$ for any node $X\in{\bX}\cup\{Y\}$.
\STATE \label{alg:OFUargmax} $({\bS_t},\tilde{\btheta}_t) = \argmax_{\bS\subseteq \bX, |\bS|\le K, \btheta'_{t,X} \in \mathcal{C}_{t,X} } \E[Y| \doi({\bS})]$.
\STATE Intervene all the nodes in ${\bS}_t$ to $1$ and observe the feedback $(\bX_t,Y_t)$.
\ENDFOR
\end{algorithmic}
\end{algorithm}

\begin{algorithm}[t]
\caption{BGLM-Estimate}
\label{alg:glm-est}
\begin{algorithmic}[1]
\STATE {\bfseries Input:}
All observations $((\bX_1,Y_1),\cdots,(\bX_{t},Y_{t}))$ until round $t$.
\STATE {\bfseries Output:}
$\{\hat{\btheta}_{t,X},M_{t,X}\}_{X\in {\bX}\cup\{Y\}}$
\STATE For each $X\in {\bX}\cup\{Y\}$, $i \in [t]$, construct data pair $(\bV_{i,X},X^i)$
	with $\bV_{i,X}$ the parent value vector of $X$ in round $i$, and $X^i$ the value of $X$ in round $i$ if $X\not\in S_i$.
\FOR{$X\in{\bX}\cup\{Y\}$}
\STATE Calculate the maximum-likelihood estimator $\hat{\btheta}_{t,X}$ by solving the equation $ \sum_{i=1}^{t}(X^i-f_X(\bV_{i,X}^\intercal \btheta_X))\bV_{i,X}=0$.
	\label{alg:pseudolikelihood}
\STATE $M_{t,X}=\sum_{i=1}^t \bV_{i,X} \bV_{i,X}^\intercal$.
\ENDFOR
\end{algorithmic}
\end{algorithm}

We now propose the algorithm BGLM-OFU in Algorithm \ref{alg:glm-ucb}, where OFU stands for optimism in the face of uncertainty.
The algorithm contains two phases. The first phase is the initialization phase with only pure observations without doing any intervention to ensure that our maximum likelihood 
	estimation of $\btheta^*$ is accurate enough. Based on  Lecu\'{e} and Mendelson's inequality \cite{nie2021matrix}, it is designed to ensure Eq.~\eqref{eq:lambdamin} in Lemma~\ref{thm.learning_glm} holds. In practice, one alternative implementation of the initialization phase is doing no intervention until Eq.~\eqref{eq:lambdamin} holds for every $X\in\bX\cup\{Y\}$. The required number of rounds is usually much less than $T_0$.
Then in the second iterative phase, we use maximum likelihood estimator (MLE) method to
	estimate $\btheta^*$ and can therefore create a confidence region that contains the real parameters $\btheta^*$ with high probability around it to balance the exploration and exploitation. 
More specifically, in each iteration of intervention selections, we find an optimal intervention set together with a set of parameters $\tilde{\btheta}$ from the region around our unbiased estimation $\hat{\btheta}$ and take the found intervention set in this iteration. Intuitively, this method to select intervention sets follows the OFU spirit: the $\argmax$ operator in line~\ref{alg:OFUargmax} of Algorithm~\ref{alg:glm-ucb} selects the best (optimistic) solution in a
	confidence region (to address uncertainty).
The empirical mean $\hat{\btheta}$ calculated corresponds to exploration while the confidence region surrounding it is for exploration.

The regret analysis of Algorithm~\ref{alg:glm-ucb} requires two technical components to support the main analysis.
The first component indicates that when the observations are sufficient, we can get a good estimation $\hat{\btheta}$ for $\btheta^*$,
	while the second component shows that a small change in the parameters should not lead to a big change in the reward $\mathbb{E}[Y | \doi(\bS)]$. 

The first component is based on the result of maximum-likelihood estimation.
In the studies of \cite{filippi2010parametric} and \cite{li2017provably}, a standard log-likelihood function used in the updating process should be $L^{std}_{t,X}(\btheta_X)=\sum_{i=1}^t[X^{i}\ln f_X(\bV_{i,X}^\intercal\btheta_X)+(1-X^i)\ln(1-f_X(\bV_{i,X}^\intercal\btheta_X))]$. 
However, the analysis in their work needs the gradient of the log-likelihood function to have 
	the form $\sum_{i=1}^t\left[X^i-f_X(\bV_{i,X}^\intercal\btheta_X)\right]\bV_{i,X}$, which is not true here. Therefore, using the same idea in \cite{zhang2022online}, we use the pseudo log-likelihood function $L_{t,X}(\btheta_X)$ instead, which is constructed by integrating the gradient of it defined by $\triangledown_{\btheta_X} L_{t,X}(\btheta_X)=\sum_{i=1}^t\left[X^i-f_X(\bV_{i,X}^\intercal\btheta_X)\right]\bV_{i,X}$. 
Actually, this pseudo log-likelihood function is used in line~\ref{alg:pseudolikelihood} of Algorithm \ref{alg:glm-est}. 
The following lemma presents the result for the learning problem as the first technical component of the regret analysis.
Let $M_{t,X}$ and $\hat{\btheta}_{t,X}$ be as defined in Algorithm \ref{alg:glm-est}, and also note that the definition of $\hat{\btheta}_{t,X}$ is equivalent to
$\hat{\btheta}_{t,X}=\argmax_{\btheta_X} L_{t,X}(\btheta_X)$.
Let $\lambda_{\min}(M)$ denote the minimum eigenvalue of matrix $M$.
\begin{restatable}[Learning Problem for BGLM]{lemma}{thmlearningglm}
\label{thm.learning_glm}
Suppose that Assumptions~\ref{asm:glm_3} and~\ref{asm:glm_2} hold.
Moreover, given $\delta\in(0,1)$, assume that
\begin{small}
\begin{align}
    \lambda_{\min}(M_{t,X})\geq \frac{512|\Pa(X)|\left(L_{f_X}^{(2)}\right)^2}{\kappa^4}\left(|\Pa(X)|^2+\ln\frac{1}{\delta}\right). \label{eq:lambdamin}
\end{align}
\end{small}
Then with probability at least $1-3\delta$, the maximum-likelihood estimator satisfies , for any $\bv\in\mathbb{R}^{|\Pa(X)|}$, 
\begin{align*}
    \left|\bv^\intercal(\hat{\btheta}_{t,X}-\btheta^*_X)\right|\leq\frac{3}{\kappa}\sqrt{\log(1/\delta)}\left\|\bv\right\|_{M_{t,X}^{-1}},
\end{align*}
where the probability is taken from the randomness of all data collected from round $1$ to round $t$.
\end{restatable}

The proof of the above lemma is adapted from \cite{li2017provably}, and is included in \OnlyInFull{Appendix~\ref{app:learning_glm}}\OnlyInShort{appendix}.
Note that the initialization phase of the algorithm together with Assumption~\ref{asm:glm_4} and the Lecu\'{e} and Mendelson's inequality 
	would show the condition on $\lambda_{\min}(M_{t,X})$ in Eq.\eqref{eq:lambdamin}, and the design of the initialization phase,
	the summarization of Assumption~\ref{asm:glm_4} and the analysis to show Eq.\eqref{eq:lambdamin} together form one of our key technical
	contributions in this section.

For the second component showing that a small change in parameters leads to a small change in the reward, we adapt 
	the group observation modulated (GOM) bounded smoothness property for the LT model \cite{li2020online}
	to show a GOM bounded smoothness property for BGLM. 
To do so, we define an equivalent form of BGLM as a threshold model as follows.
For each node $X$, we randomly sample a threshold $\gamma_X$ uniformly from $[0,1]$, i.e. $\gamma_X\sim\mathcal{U}[0,1]$, 
	and if $f_X(\Pa(X)\cdot \btheta^*_X)+\varepsilon_X\geq \gamma_X$, $X$ is activated (i.e. set to $1$); 
	if not, $X$ is not activated (i.e. set to $0$). 
Suppose $X_1,X_2,\cdots,X_{n-1},Y$ is a topological order of nodes in ${\bX}\cup\{Y\}$, then at time step $1$, only $X_1$ is tried to be activated; 
	at time step $2$, only $X_2$ is tried to be activated by $X_1$; $\ldots$; at time step $n$, only $Y$ is tried to be activated by activated nodes in $\Pa(Y)$.
The above view of the propagating process is equivalent to BGLM, but it shows that BGLM is a general form of the LT model~\cite{kempe03} on DAG.
Thus we can show a result below similar to Theorem 1 in \cite{li2020online} for the LT model. 
For completeness, we include the proof in \OnlyInFull{Appendix~\ref{app:gom_glm}}\OnlyInShort{appendix}.
Henceforth, we use $\sigma(\bS,\btheta)$ to represent the reward function $\E[Y | \doi(\bS)]$ under parameter $\btheta$, to make the parameters of the reward explicit.
We use $\bgamma$ to represent the vector $(\gamma_X)_{X\in \bX\cup Y}$.

\begin{restatable}[GOM Bounded Smoothness of BGLM]{lemma}{thmgomglm}
\label{thm.gom_glm}
For any two weight vectors $\btheta^1,\btheta^2\in \Theta$ for a BGLM $G$, the difference of their expected reward for any intervened set $\bS$ can be bounded as
\begin{small}\begin{align}
    &\left|\sigma({\bS},\btheta^1)-\sigma({\bS},\btheta^2)\right| \leq\mathbb{E}_{\bbvarepsilon, \bgamma}\left[ \sum_{X\in {\bX}_{{\bS},Y}}
    	\!\!\! \left|\bV_{X}^\intercal (\btheta^1_X-\btheta^2_X)\right|L_{f_X}^{(1)}\right], \label{eq:GOMcond}
\end{align}\end{small}
where ${\bX}_{{\bS},Y}$ is the set of nodes in paths from ${\bS}$ to $Y$ excluding $\bS$, and $\bV_X$ is the propagation result of the parents of $X$ under parameter $\btheta^2$. The expectation is taken over the randomness of the thresholds $\bgamma$ and the noises $\bbvarepsilon$.
\end{restatable}

We can now prove the regret bound of Algorithm \ref{alg:glm-ucb}. 
In the proof of Theorem \ref{thm.regret_glm}, we use Lecu\'{e} and Mendelson's inequality \cite{nie2021matrix} to prove that our initialization step has a very high probability to meet the needs of Lemma \ref{thm.learning_glm}. 
Then we use Lemma \ref{thm.gom_glm} to transform the regret to the sum of $\left\|\bV_{t,X}\right\|_{M_{t-1,X}^{-1}}$, which can be bounded using a similar lemma of Lemma 2 in \cite{li2017provably}. The details of the proof is put in \OnlyInFull{Appendix~\ref{app:regret_glm}}\OnlyInShort{appendix} for completeness. 

\begin{restatable}[Regret Bound of BGLM-OFU]{theorem}{thmregretglm}
\label{thm.regret_glm}
Under Assumptions \ref{asm:glm_3}, \ref{asm:glm_2} and \ref{asm:glm_4}, the regret of BGLM-OFU (Algorithm~\ref{alg:glm-ucb} and~\ref{alg:glm-est}) is bounded as
\begin{equation}\begin{aligned}
    \label{equation.regret_glm}
    R(T)&=O\left(\frac{1}{\kappa} n L^{(1)}_{\max} \sqrt{DT}\log T\right),
\end{aligned}
\end{equation}
where the terms of $o(\sqrt{T})$ are omitted.
\end{restatable}

{\bf Remarks. } The leading term of the regret in terms of $T$ is in the order of $O(\sqrt{T}\log T )$, which is commonly
	seen in confidence ellipsoid based bandit or combinatorial bandit algorithms (e.g.~\cite{abbasi2011improved,li2020online,zhang2022online}). Also, it matches the regret of previous causal bandits algorithm, C-UCB in \cite{lu2020regret}, which works on the atomic setting.
The term $L^{(1)}_{\max}$ reflects the rate of changes in $f_X$'s, and intuitively, the higher rate of changes in $f_X$'s, the larger the regret since 
	the online learning algorithm inevitably leads to some error in the parameter estimation, which will be amplified by the rate of changes in $f_X$'s.
Technically, $L^{(1)}_{\max}$ comes from the $L_{f_X}^{(1)}$ term in the GOM condition (Eq.\eqref{eq:GOMcond}).
Term $n$ is to relax the sum over ${\bX}_{\bS,Y}$ in Eq.\eqref{eq:GOMcond}, and it could be made tighter in causal graphs where $|{\bX}_{\bS,Y}|$ is significantly smaller than $n$,
	and intuitively it means that all nodes on the path from $\bS$ to $Y$ would contribute to the regret.
Term $\sqrt{D}$ implies that the regret depends on the number of parents of nodes, and technically it is because 
	the scale of $\left\|\bV_{t,X}\right\|_{M_{t-1,X}^{-1}}$ is approximately $\sqrt{|\Pa(X)|} \le \sqrt{D}$, 
	and we bound the regret as the sum of $\left\|\bV_{t,X}\right\|_{M_{t-1,X}^{-1}}$'s as we explained earlier.
Term $\frac{1}{\kappa}$ comes from the learning problem (Lemma~\ref{thm.learning_glm}), which is also adopted in the regret bound of UCB-GLM 
	of \cite{li2017provably} using a similar learning problem. For the budget $K$, it does not appear in our regret because it is not directly related to the number of parameters we are estimating.
	
While our algorithm and analysis are based on several past studies~\cite{li2017provably,zhang2022online,li2020online}, our innovation includes
	(a) the initialization phase and its corresponding Assumption~\ref{asm:glm_4} and its more involved analysis, because in our model we do not have direct observations
	of one-step immediate causal effect; and
	(b) the integration of the techniques from these separate studies, such as the maximum likelihood based analysis of~\cite{li2017provably}, the pseudo log-likelihood function
	of~\cite{zhang2022online}, and the GOM condition analysis of \cite{li2020online}, whereas each of these studies alone is not enough to achieve our result.

In line~\ref{alg:OFUargmax} of Algorithm~\ref{alg:glm-ucb}, the $\argmax$ operator needs a simultaneous optimization over both the intervention set $\bS$ and parameters $\btheta'$.
This could be a computationally hard problem for large-scale problems, but since we focus on the online learning aspect of the problem, we leave the computationally-efficient 
	solution for the full problem as future work, and such treatment is often seen in other combinatorial online learning problems (e.g.~\cite{combes2015combinatorialfeedback,li2020online}).

\section{Algorithms for BLM with Hidden Variables}
\label{sec.hidden}

Previous sections consider only Markovian models without hidden variables.
In many causal models, hidden variables exist to model the latent confounding factors.
In this section, we present results on CCB under the linear model BLM with hidden variables.

\subsection{Transforming the Model with Hidden Variables to the one without Hidden Variables}
\label{sec:transformBLM}
To address hidden variables, we first show how to reduce the BLM with hidden variables to a corresponding one without
	the hidden variables.
Suppose the hidden variables in $G=(\bU\cup \bX \cup \{Y\},E)$ are ${\bf U}=\{U_0,U_1,U_2,\cdots\}$, and
	we use $X_i,X_j$'s to represent observed variables.
Without loss of generality, we let $U_0$ always be $1$ and it only has out-edges pointing to other observed or unobserved nodes, to model
the self-activations of nodes.
We allow various types of connections involving the hidden nodes, including edges from observed nodes to hidden nodes and edges among
	hidden nodes, but we disallow the situation where a hidden node $U_s$ with $s>0$ has two paths to $X_i$ and $X_i$'s descendant $X_j$ and the
	paths contain only hidden nodes except the end points $X_i$ and $X_j$. 
Figure \ref{Fig.hidden} is an example of a causal graph allowed for this section.

\begin{figure}[htb] 
\centering 
\includegraphics[width=0.40\textwidth]{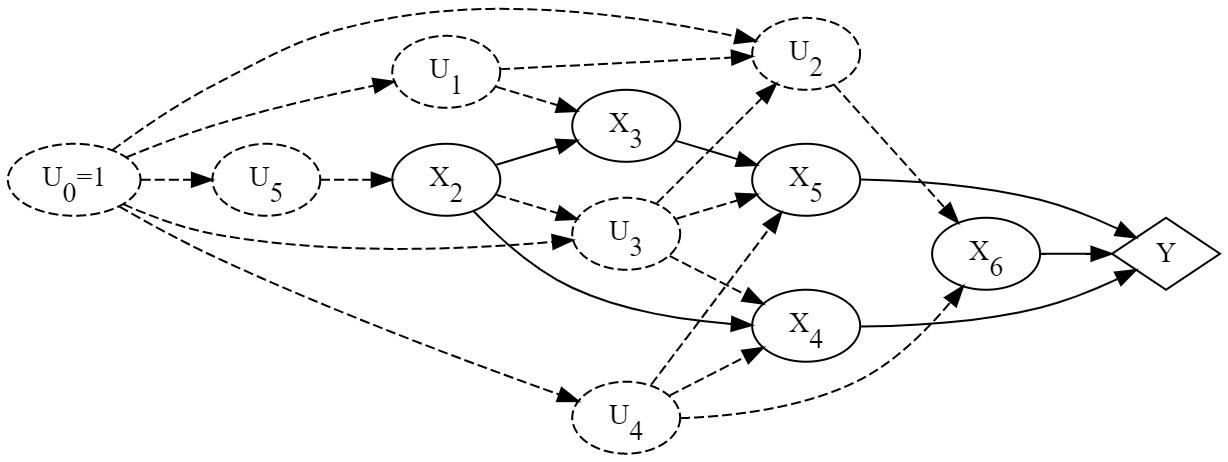}
\caption{An Example of BLM with Hidden Variables}
\label{Fig.hidden} 
\end{figure}

Our idea is to transform such a general causal model into an equivalent Markovian model $G'=(\{X_1\}\cup {\bf X}\cup\{Y\},E')$. 
For convenience, we assume $X_1$ is not in the original observed variable set $\bX\cup \{Y\}$.
Henceforth, all notations with~$'$, such as $\Pr',\E', \btheta^{*'}, \Pa'(X)$ correspond to the new Markovian model $G'$, and
	the notations $\Pr, \E$ without $'$ refer to the original model $G$.
For any two observed nodes $X_i,X_j$, a hidden path $P$ from $X_i$ to $X_j$ is a directed path from $X_i$ to $X_j$ where all intermediate nodes are hidden or
	there are no intermediate nodes.
We define $\theta^*_P$ to be the multiplication of weights on path $P$.
Let $\cP^u_{X_i,X_j}$ be the set of hidden paths from $X_i$ to $X_j$.
If $\cP^u_{X_i,X_j} \ne \emptyset$, then we add edge $(X_i,X_j)$ into $E'$, and its weight
	$\theta^{*'}_{X_i,X_j}=\sum_{P\in \cP^u_{X_i,X_j}}\theta^*_P$.
As in the Markovian model, $X_1$ is always set to $1$, and for each $X_i \in \bX\cup \{Y\}$, we add edge
	$(X_1, X_i)$ into $E'$, with weight $\theta^{*'}_{X_1,X_i}= \Pr\left\{X_i=1|\doi\left({\bf X} \cup \{Y\} \setminus\{X_i\}={\bf 0}\right)\right\}$.
The noise variables $\bbvarepsilon$ are carried over to $G'$ without change.

The following lemma shows that the new model $G'$ has the same parent-child conditional probability as the original model $G$.
The proof of this lemma utilizes the do-calculus rules for causal inference~\cite{Pearl09}.
\begin{restatable}{lemma}{lemmahidden}
\label{lemma.hidden}
For any $X\in{\bf X}\cup\{Y\}$, any $\bS \subseteq \bX$, any value $\pa'(X) \in \{0,1\}^{|\Pa'(X)|}$, any value $\bs \in \{0,1\}^{|\bS|}$
	($\bs$ is consistent with $\pa'(X)$ on values in $\bS \cap \Pa'(X)$), we have
\begin{align*}
    &\Pr\left\{X=1|\Pa'(X)\setminus\{X_1\}=\pa'(X)\setminus\{x_1\},\doi(\bS=\bs)\right\}\\
    &=\Pr{'}\left\{X=1|\Pa'(X)=\pa'(X),\doi(\bS=\bs)\right\}.
\end{align*}
\end{restatable}
The next lemma shows that the objective function is also the same for $G'$ and $G$.

\begin{restatable}{lemma}{thmlinearproperty}
\label{thm:property_of_linear_model}
For any $\bS \subseteq \bX$, any value $\bs \in \{0,1\}^{|\bS|}$, we have
$\mathbb{E}[Y|\doi({\bf S}={\bf s})]=\mathbb{E}'[Y|\doi({\bf S}={\bf s})]$.
\end{restatable}

The above two lemmas show that from $G$ to $G'$ the parent-child conditional probability and the reward function are all remain unchanged.

\subsection{Applying BGLM-OFU on BLM}

With the transformation described in Section~\ref{sec:transformBLM} and its properties summarized in Lemmas~\ref{lemma.hidden} and~\ref{thm:property_of_linear_model},
	we can apply Algorithm~\ref{alg:glm-ucb} on $G'$ to achieve the learning result for $G$.
More precisely, we can use the observed data of $X$ and $\Pa'(X)$ to estimate parameters $\btheta^{*'}_X$ and minimize the regret on the reward 
	$\mathbb{E}'[Y|\doi({\bf S})]$, which is the same as $\mathbb{E}[Y|\doi({\bf S})]$.
Furthermore, under the linear model BLM, Assumptions~\ref{asm:glm_3} and~\ref{asm:glm_2} hold with $L_{f_X}^{(1)} = \kappa = 1$ and $L_{f_X}^{(2)}$ could be
	any constant greater than $0$.
We still need Assumption~\ref{asm:glm_4}, but we change the $\Pa(X)$'s in the assumption to $\Pa'(X)$'s.
Then we can have the following regret bound.
\begin{theorem}[Regret Bound of Algorithm \ref{alg:glm-ucb} for BLM]
\label{thm.blm_regret}
Under Assumption~\ref{asm:glm_4}, Algorithm \ref{alg:glm-ucb} has the following regret bound when running on BLM with hidden variables:
\begin{align}
    R(T)=O\left(n\sqrt{DT}\log T\right), 
\end{align}
where $n$ is the number of nodes in $G'$, and $D = \max_{X\in \bX\cup\{Y\}} |\Pa'(X)|$ is the maximum in-degree in $G'$.
\end{theorem}

{\bf Remarks.} We compare our regret bound to the one in \cite{li2020online} for the online influence maximization problem under the LT model, 
which is a special case of our BLM model with $\varepsilon_X = 0$ for all $X$'s.
Our regret bound is $O(n^{\frac{3}{2}}\sqrt{T}\ln T)$ while theirs is $O(n^{\frac{9}{2}}\sqrt{T}\ln T)$.
The $n^3$ factor saving comes from three sources:
	(a) our reward is of scale $[0,1]$ while theirs is $[0,n]$, and this saves one $n$ factor;
	(b) our BLM is a DAG, and thus we can equivalently fix the activation steps as described before Lemma~\ref{thm.gom_glm}, causing our
	GOM condition (Lemma~\ref{thm.gom_glm}) to save another factor of $n$; and
	(c) we use MLE, which requires an initialization phase and Assumption~\ref{asm:glm_4} to give an accurate estimate of $\btheta^*$, while
	their algorithm uses linear regression without an initialization phase, which means we tradeoff a factor of $n$ with an additional Assumption~\ref{asm:glm_4}.

\subsection{Algorithm for BLM based on Linear Regression}

Now we have already introduced how to use Algorithm~\ref{alg:glm-ucb} and \ref{alg:glm-est} to solve the online BLM CCB problem with hidden variables. However, in order to meet the needs of Lemma~\ref{thm.learning_glm}, we have to process an initialization phase. Assumption~\ref{asm:glm_4} needs to hold for the Markovian model $G'$ after the transformation. 
We can remove the initialization phase and the dependency on Assumption~\ref{asm:glm_4}, by
	noticing that our MLE based on pseudo log-likelihood maximization is equivalent to linear regression when adopted on BLMs.
In particular, we use Lemma 1 in \cite{li2020online} to replace Lemma~\ref{thm.learning_glm} for MLE. We rewrite it as Lemma~\OnlyInFull{\ref{lemma.learning_glm}}\OnlyInShort{11} in \OnlyInFull{Appendix~\ref{app.blmlr}}\OnlyInShort{appendix}.

\begin{algorithm}[t]
\caption{BLM-LR for BLM CCB Problem}
\label{alg:linear-lr}
\begin{algorithmic}[1]
\STATE {\bfseries Input:}
Graph $G=({\bX}\cup\{Y\},E)$, intervention budget $K\in\mathbb{N}$.
\STATE Initialize $M_{0,X}\leftarrow \mathbf{I}\in\mathbb{R}^{|\Pa(X)|\times |\Pa(X)|}$, $\bb_{0,X}\leftarrow {\bf 0}^{|\Pa(X)|}$ for all $X\in {\bX}\cup\{Y\}$, $\hat{\btheta}_{0,X}\leftarrow 0\in\mathbb{R}^{|\Pa(X)|}$ for all $X\in {\bX}\cup\{Y\}$, $\delta\leftarrow\frac{1}{n\sqrt{T}}$ and $\rho_t\leftarrow\sqrt{n\log(1+tn)+2\log\frac{1}{\delta}}+\sqrt{n}$ for $t=0,1,2,\cdots,T$.
\FOR{$t=1,2,\cdots,T$}
\STATE Compute the confidence ellipsoid $\mathcal{C}_{t,X}=\{\btheta_X'\in[0,1]^{|\Pa(X)|}:\left\|\btheta_X'-\hat{\btheta}_{t-1,X}\right\|_{M_{t-1,X}}\leq\rho_{t-1}\}$ for any node $X\in{\bX}\cup\{Y\}$.
\STATE \label{alg:LRargmax} $({\bS_t},\tilde{\btheta}_t) = \argmax_{\bS\subseteq \bX, |\bS|\le K, \btheta'_{t,X} \in \mathcal{C}_{t,X} } \E[Y| \doi({\bS})]$.
\STATE Intervene all the nodes in ${\bS}_t$ to $1$ and observe the feedback $(\bX_t,Y_t)$.
\FOR{$X\in\bX\cup\{Y\}$}
\STATE Construct data pair $(\bV_{t,X},X^t)$ with $\bV_{t,X}$ the parent value vector of $X$ in round $t$, and $X^t$ the value of $X$ in round $t$ if $X\not\in S_t$.
\STATE $M_{t,X}=M_{t-1,X}+\bV_{t,X}\bV_{t,X}^\intercal$, $\bb_{t,X}=\bb_{t-1,X}+X^t\bV_{t,X}$, $\hat{\btheta}_{t,X}=M^{-1}_{t,X}\bb_{t,X}$.
\ENDFOR
\ENDFOR
\end{algorithmic}
\end{algorithm}

Based on the above result, we introduce Algorithm~\ref{alg:linear-lr} using the linear regression.
Algorithm~\ref{alg:linear-lr} is designed for Markovian BLMs. 
For a BLM $G$ with more general hidden variables, we can apply the transformation described in Section~\ref{sec:transformBLM} to 
	transform the model into a Markovian model $G'$ first. 
The following theorem shows the regret bound of Algorithm~\ref{alg:linear-lr}.

\begin{restatable}[Regret Bound of Algorithm~\ref{alg:linear-lr}]{theorem}{thmlinearlr}
\label{thm.blmlr}
The regret of BLM-LR (Algorithm~\ref{alg:linear-lr}) running on BLM with hidden variables is bounded as
\begin{align*}
    R(T)=O\left(n^2\sqrt{DT}\log T\right).
\end{align*}
\end{restatable}

The proof of this theorem is adapted from the proof of Theorem 2 in \cite{li2020online}. For completeness, the proof is put in \OnlyInFull{Appendix~\ref{app.blmlr}}\OnlyInShort{appendix}. 
Comparing the algorithm and the regret bound of BLM-LR with those of BGLM-OFU, we can see a tradeoff between using the MLE method and the linear regression method:
When we use the linear regression method with the BLM-LR algorithm, we do not need to have an initialization phase so Assumption~\ref{asm:glm_4} is not required anymore. 
However, the regret bound of BLM-LR (Theorem~\ref{thm.blmlr}) has an extra factor of $n$ in regret bound, comparing to the regret bound of BGLM-OFU on BLM (Theorem~\ref{thm.blm_regret}).

\section{Conclusion and Future Work}

In this paper, we propose the combinatorial causal bandit (CCB) framework, and provide a solution for CCB under the BGLM.
We further study a special model, the linear model. We show that our algorithm would work for models with many types of hidden variables. We further provide an algorithm for linear model not relying on Assumption~\ref{asm:glm_4} based on the linear regression. 
	
There are many open problems and future directions to extend this work.
For the BGLM, one could study how to remove some assumptions (e.g. Assumption~\ref{asm:glm_4}), how to include hidden variables, or how to make
	the computation more efficient.
For the linear model, one could consider how to remove the constraint on the hidden variable structure that does not allow a hidden variable to
	connect to an observed variable and its observed descendant via hidden variables.
More generally, one could consider classes of causal models other than the BGLM. 

\section*{Acknowledgements}
The authors would like to thank the anonymous reviewers for their valuable comments and constructive feedback.

\bibliography{main}

\begin{thebibliography}{38}
\providecommand{\natexlab}[1]{#1}

\bibitem[{Abbasi-Yadkori, P{\'a}l, and
  Szepesv{\'a}ri(2011)}]{abbasi2011improved}
Abbasi-Yadkori, Y.; P{\'a}l, D.; and Szepesv{\'a}ri, C. 2011.
\newblock Improved algorithms for linear stochastic bandits.
\newblock \emph{Advances in Neural Information Processing Systems}, 24.

\bibitem[{Aitkin et~al.(2009)Aitkin, Francis, Hinde, and
  Darnell}]{aitkin2009statistical}
Aitkin, M.; Francis, B.; Hinde, J.; and Darnell, R. 2009.
\newblock \emph{Statistical Modelling in R}.
\newblock Oxford University Press Oxford.

\bibitem[{Arnold et~al.(2020)Arnold, Davies, de~Kamps, Tennant, Mbotwa, and
  Gilthorpe}]{arnold2020reflection}
Arnold, K.~F.; Davies, V.; de~Kamps, M.; Tennant, P.~W.; Mbotwa, J.; and
  Gilthorpe, M.~S. 2020.
\newblock Reflection on modern methods: generalized linear models for prognosis
  and intervention—theory, practice and implications for machine learning.
\newblock \emph{International Journal of Epidemiology}, 49(6): 2074--2082.

\bibitem[{Chen, Wang, and Yuan(2013)}]{pmlr-v28-chen13a}
Chen, W.; Wang, Y.; and Yuan, Y. 2013.
\newblock Combinatorial multi-armed bandit: General framework and applications.
\newblock In \emph{International Conference on Machine Learning}, 151--159.
  PMLR.

\bibitem[{Chen et~al.(2016)Chen, Wang, Yuan, and Wang}]{ChenWYW16}
Chen, W.; Wang, Y.; Yuan, Y.; and Wang, Q. 2016.
\newblock Combinatorial multi-armed bandit and its extension to
  probabilistically triggered arms.
\newblock \emph{The Journal of Machine Learning Research}, 17(1): 1746--1778.

\bibitem[{Combes et~al.(2015)Combes, Shahi, Proutiere
  et~al.}]{combes2015combinatorialfeedback}
Combes, R.; Shahi, M. S. T.~M.; Proutiere, A.; et~al. 2015.
\newblock Combinatorial bandits revisited.
\newblock In \emph{Advances in Neural Information Processing Systems},
  2116--2124.

\bibitem[{De~la Fuente(2000)}]{de2000mathematical}
De~la Fuente, A. 2000.
\newblock \emph{Mathematical Methods and Models for Economists}.
\newblock Cambridge University Press.

\bibitem[{Filippi et~al.(2010)Filippi, Cappe, Garivier, and
  Szepesv{\'a}ri}]{filippi2010parametric}
Filippi, S.; Cappe, O.; Garivier, A.; and Szepesv{\'a}ri, C. 2010.
\newblock Parametric bandits: The generalized linear case.
\newblock \emph{Advances in Neural Information Processing Systems}, 23.

\bibitem[{Fisher(1992)}]{fisher1992statistical}
Fisher, R.~A. 1992.
\newblock Statistical methods for research workers.
\newblock In \emph{Breakthroughs in Statistics}, 66--70. Springer.

\bibitem[{Garcia-Huidobro and Michael~Oakes(2017)}]{garcia2017squeezing}
Garcia-Huidobro, D.; and Michael~Oakes, J. 2017.
\newblock Squeezing observational data for better causal inference: Methods and
  examples for prevention research.
\newblock \emph{International Journal of Psychology}, 52(2): 96--105.

\bibitem[{Han, Yu, and Friedberg(2017)}]{han2017evaluating}
Han, B.; Yu, H.; and Friedberg, M.~W. 2017.
\newblock Evaluating the impact of parent-reported medical home status on
  children's health care utilization, expenditures, and quality: a
  difference-in-differences analysis with causal inference methods.
\newblock \emph{Health Services Research}, 52(2): 786--806.

\bibitem[{Hastie and Pregibon(2017)}]{hastie2017generalized}
Hastie, T.~J.; and Pregibon, D. 2017.
\newblock Generalized linear models.
\newblock In \emph{Statistical Models in S}, 195--247. Routledge.

\bibitem[{Hern{\'a}n and Robins(2010)}]{hernan2010causal}
Hern{\'a}n, M.~A.; and Robins, J.~M. 2010.
\newblock Causal inference.

\bibitem[{Hilbe(2011)}]{hilbe2011logistic}
Hilbe, J.~M. 2011.
\newblock Logistic regression.
\newblock \emph{International Encyclopedia of Statistical Science}, 1: 15--32.

\bibitem[{Hoeffding(1994)}]{hoeffding1994probability}
Hoeffding, W. 1994.
\newblock Probability inequalities for sums of bounded random variables.
\newblock In \emph{The Collected Works of Wassily Hoeffding}, 409--426.
  Springer.

\bibitem[{Karp(1972)}]{karp1972reducibility}
Karp, R.~M. 1972.
\newblock Reducibility among combinatorial problems.
\newblock In \emph{Complexity of Computer Computations}, 85--103. Springer.

\bibitem[{Kempe, Kleinberg, and Tardos(2003)}]{kempe03}
Kempe, D.; Kleinberg, J.; and Tardos, {\'E}. 2003.
\newblock Maximizing the spread of influence through a social network.
\newblock In \emph{Proceedings of the Ninth ACM SIGKDD International Conference
  on Knowledge Discovery and Data Mining}, 137--146.

\bibitem[{Lattimore, Lattimore, and Reid(2016)}]{lattimore2016causal}
Lattimore, F.; Lattimore, T.; and Reid, M.~D. 2016.
\newblock Causal bandits: learning good interventions via causal inference.
\newblock In \emph{Proceedings of the 30th International Conference on Neural
  Information Processing Systems}, 1189--1197.

\bibitem[{Lattimore and Szepesv\'{a}ri(2020)}]{LS20}
Lattimore, T.; and Szepesv\'{a}ri, C. 2020.
\newblock \emph{Bandit Algorithms}.
\newblock Cambridge University Press.

\bibitem[{Lee and Bareinboim(2018)}]{lee2018structural}
Lee, S.; and Bareinboim, E. 2018.
\newblock Structural causal bandits: where to intervene?
\newblock \emph{Advances in Neural Information Processing Systems}, 31.

\bibitem[{Lee and Bareinboim(2019)}]{lee2019structural}
Lee, S.; and Bareinboim, E. 2019.
\newblock Structural causal bandits with non-manipulable variables.
\newblock In \emph{Proceedings of the Thirty-Third AAAI Conference on
  Artificial Intelligence and Thirty-First Innovative Applications of
  Artificial Intelligence Conference and Ninth AAAI Symposium on Educational
  Advances in Artificial Intelligence}, 4146--4172.

\bibitem[{Lee and Bareinboim(2020)}]{lee2020characterizing}
Lee, S.; and Bareinboim, E. 2020.
\newblock Characterizing optimal mixed policies: Where to intervene and what to
  observe.
\newblock \emph{Advances in Neural Information Processing Systems}, 33:
  8565--8576.

\bibitem[{Li, Lu, and Zhou(2017)}]{li2017provably}
Li, L.; Lu, Y.; and Zhou, D. 2017.
\newblock Provably optimal algorithms for generalized linear contextual
  bandits.
\newblock In \emph{International Conference on Machine Learning}, 2071--2080.
  PMLR.

\bibitem[{Li et~al.(2020)Li, Kong, Tang, Li, and Chen}]{li2020online}
Li, S.; Kong, F.; Tang, K.; Li, Q.; and Chen, W. 2020.
\newblock Online influence maximization under linear threshold model.
\newblock \emph{Advances in Neural Information Processing Systems}, 33:
  1192--1204.

\bibitem[{Lu, Meisami, and Tewari(2021)}]{lu2021causal}
Lu, Y.; Meisami, A.; and Tewari, A. 2021.
\newblock Causal bandits with unknown graph structure.
\newblock \emph{Advances in Neural Information Processing Systems}, 34.

\bibitem[{Lu et~al.(2020)Lu, Meisami, Tewari, and Yan}]{lu2020regret}
Lu, Y.; Meisami, A.; Tewari, A.; and Yan, W. 2020.
\newblock Regret analysis of bandit problems with causal background knowledge.
\newblock In \emph{Conference on Uncertainty in Artificial Intelligence},
  141--150. PMLR.

\bibitem[{Maiti, Nair, and Sinha(2021)}]{maiti2021causal}
Maiti, A.; Nair, V.; and Sinha, G. 2021.
\newblock Causal Bandits on General Graphs.
\newblock \emph{arXiv preprint arXiv:2107.02772}.

\bibitem[{Nair, Patil, and Sinha(2021)}]{nair2020budgeted}
Nair, V.; Patil, V.; and Sinha, G. 2021.
\newblock Budgeted and non-budgeted causal bandits.
\newblock In \emph{International Conference on Artificial Intelligence and
  Statistics}, 2017--2025. PMLR.

\bibitem[{Nie(2022)}]{nie2021matrix}
Nie, Z. 2022.
\newblock Matrix anti-concentration inequalities with applications.
\newblock In \emph{Proceedings of the 54th Annual ACM SIGACT Symposium on
  Theory of Computing}, 568--581.

\bibitem[{Pearl(2009)}]{Pearl09}
Pearl, J. 2009.
\newblock \emph{Causality}.
\newblock Cambridge University Press.
\newblock 2nd Edition.

\bibitem[{Pearl(2012)}]{pearl2012calculus}
Pearl, J. 2012.
\newblock The do-calculus revisited.
\newblock In \emph{Proceedings of the Twenty-Eighth Conference on Uncertainty
  in Artificial Intelligence}, 3--11.

\bibitem[{Russell and Norvig(2021)}]{russell2021artificial}
Russell, S.; and Norvig, P. 2021.
\newblock Artificial Intelligence: A Modern Approach, Global Edition 4th.
\newblock \emph{Foundations}, 19: 23.

\bibitem[{Sakate and Kashid(2014)}]{sakate2014comparison}
Sakate, D.; and Kashid, D. 2014.
\newblock Comparison of Estimators in GLM with Binary Data.
\newblock \emph{Journal of Modern Applied Statistical Methods}, 13(2): 10.

\bibitem[{Sen et~al.(2017)Sen, Shanmugam, Dimakis, and
  Shakkottai}]{sen2017identifying}
Sen, R.; Shanmugam, K.; Dimakis, A.~G.; and Shakkottai, S. 2017.
\newblock Identifying best interventions through online importance sampling.
\newblock In \emph{International Conference on Machine Learning}, 3057--3066.
  PMLR.

\bibitem[{Vansteelandt and Dukes(2020)}]{Vansteelandt2020AssumptionleanIF}
Vansteelandt, S.; and Dukes, O. 2020.
\newblock Assumption-lean inference for generalised linear model parameters.
\newblock \emph{arXiv preprint arXiv:2006.08402}.

\bibitem[{Wu et~al.(2018)Wu, Du, Wang, Wu, Duan, and Tian}]{wu2018general}
Wu, W.; Du, H.; Wang, H.; Wu, L.; Duan, Z.; and Tian, C. 2018.
\newblock On general threshold and general cascade models of social influence.
\newblock \emph{Journal of Combinatorial Optimization}, 35(1): 209--215.

\bibitem[{Yabe et~al.(2018)Yabe, Hatano, Sumita, Ito, Kakimura, Fukunaga, and
  Kawarabayashi}]{yabe2018causal}
Yabe, A.; Hatano, D.; Sumita, H.; Ito, S.; Kakimura, N.; Fukunaga, T.; and
  Kawarabayashi, K.-i. 2018.
\newblock Causal bandits with propagating inference.
\newblock In \emph{International Conference on Machine Learning}, 5512--5520.
  PMLR.

\bibitem[{Zhang et~al.(2022)Zhang, Chen, Sun, and Zhang}]{zhang2022online}
Zhang, Z.; Chen, W.; Sun, X.; and Zhang, J. 2022.
\newblock Online influence maximization with node-level feedback using standard
  offline oracles.
\newblock In \emph{Proceedings of the AAAI Conference on Artificial
  Intelligence}, volume~36, 9153--9161.

\end{thebibliography}

\clearpage

\appendix
\onecolumn

\section*{Appendix}

\section{Notations}

The main notations used in this paper are listed below.

\begin{center}
	\begin{longtable}{cc}
		\hline
		Notations&Meanings\\
		\hline
		$c$& the constant in Lecu\'{e} and Mendelson's inequality\\
		$\Ch(X)$& set of children of $X$\\
		$D$& the maximum in-degree of nodes in $G$ \\
		$D_{\text{out}}$& the maximum out-degree of nodes in Markovian BGLM $G$ except $X_1$ \\
		$E$& the set of edges in $G$ \\
		$f_X$& the function that defines the activating probability of $X$ in BGLM \\
		$G$& a causal model $G=\{{\bf X}\cup\{Y\},E\}$ with other hidden nodes \\
		$G'$& the transformed Markovian model from $G$ in Section \ref{sec.hidden} \\
		$J_{i,X}$& the number of created data pairs in the $i^{th}$ round for node $X$\\
		$K$& the maximal number of nodes that can be chosen to intervene in a round\\
		$L_{f_X}^{(1)}$& the upper bound of $\dot{f}_X$\\
		$L^{(1)}_{\max}$& $\max_{X\in {\bf X}\cup\{Y\}}L_{f_X}^{(1)}$\\
		$L_{f_X}^{(2)}$& the upper bound of $\ddot{f}_X$\\
		$M_{t,X}$& the observation matrix ($M_{t,X}=\sum_{i=1}^t\sum_{j=1}^{J_{i,X}}\bV_{i,j,X}\bV_{i,j,X}^\intercal$)\\
		$m$& number of edges in $G$\\
		$n$& number of observed nodes in $G$\\
		$\pdf$& probability density function\\
		$P$& used to define the propagating rule of BGLMs by $P(X|\Pa(X))$\\ 
		$\Pa(X)$& set of parents of $X$\\
		$\cP_{X_i,X_j}$& set of paths from $X_i$ to $X_j$\\
		$\cP^u_{X_i,X_j}$& set of hidden paths from $X_i$ to $X_j$\\
		$R$& a constant in online algorithms\\
		$R_t$& the regret in the $t^{th}$ round\\
		${\bf S}_t$& the set of nodes we perform $\doi({\bf S}_t={\bf s}_t)$ in round $t$\\
		$T$& total number of rounds in the online CCB problem\\
		${\bf U}$& the set of hidden variables in $G$\\
		$\bV_X$& the propagating result of $\Pa(X)$\\
		${\bf X}$& node set $\{X_1,X_2,\cdots,X_{n-1}\}$ such that $X_1=1$ for Markovian models\\
		$({\bf X}_t,Y_t)$& the observed propagating result of nodes in round $t$\\
		$X^i$& propagating result of $X$ at the $i^{th}$ round\\
		$Y$& the reward node in $G$\\
		$\bgamma$& random thresholds in BGLMs\\
		$\delta$& a constant in the online algorithms\\
		$\varepsilon_X$& the noise for activating $X$\\
		$\varepsilon_X'$& a variable used in the proof of Lemma \ref{thm.learning_glm}\\
		$\zeta$& a constant in Assumption~\ref{asm:glm_4} for BGLMs\\
		$\btheta^*$& the real parameters\\
		$\hat{\btheta}$& estimated parameters\\
		$\tilde{\btheta}$& parameter with the upper confidence bound or from the OFU ellipsoid\\
		$\lambda_{\min}(M)$& minimum eigenvalue of matrix $M$\\
		$\kappa$& a constant in Assumption \ref{asm:glm_2}\\
		$\rho$& a constant in the input of our online algorithms\\
		$\upsilon$& a constant in the justification of Assumption \ref{asm:glm_4} (Appendix \ref{app:asm_just})\\
		$\sigma({\bf S}, \btheta)$& the expected reward under parameter $\btheta$ and intervention $\doi({\bf S})$\\
		\hline
	\end{longtable}
\end{center}

\section{A Justification of Assumption~\ref{asm:glm_4}}
\label{app:asm_just}

In order to show that our assumption is reasonable, we give a possible valuation of $\zeta$ here for general Markovian BGLMs under certain conditions. 
First, we use the following alternative assumption:
\begin{assumption}
	\label{asm:parent2child}
	There exists a constant $\upsilon>0$ such that for every $X\in\bX\cup\{Y\}\setminus\{X_1\}$ and $\pa(X)\in\{0,1\}^{|\Pa(X)|}$, 
	\begin{align}
	&\E_{\bbvarepsilon}[P(X=1|\Pa(X)=\pa(X)) | \bbvarepsilon] \geq \upsilon, \label{eq:parentChildOne}\\
	&\E_{\bbvarepsilon}[P(X=0|\Pa(X)=\pa(X)) | \bbvarepsilon] \geq \upsilon.  \label{eq:parentChildZero}
	\end{align} 
\end{assumption}
This assumption means that even if all the parents of $X$ have fixed their values, $X$ still has a constant probability of being either $0$ or $1$.
The inequality for $X=0$ is already adopted as Assumption 1 in \citep{zhang2022online}. 
In Markovian BGLMs, since 
	$\Pr_{\bbvarepsilon,X}\{X=1|\Pa(X)=\pa(X)\}= \E_{\bbvarepsilon,X}[X | \Pa(X)=\pa(X)] = 
	\E_{\bbvarepsilon}[\E_{X}[f(\btheta^*_X \cdot \pa(X))+\varepsilon_X | \bbvarepsilon]]
	= f(\btheta^*_X \cdot \pa(X)) \ge f(\theta^*_{X_1,X})$, 
	condition $f(\theta^*_{X_1,X}) \ge \upsilon$ is suffice to satisfy Inequality~\eqref{eq:parentChildOne}.
Similarly, $\Pr_{\bbvarepsilon,X}\{X=0|\Pa(X)=\pa(X)\} = 1 - \E_{\bbvarepsilon,X}[X | \Pa(X)=\pa(X)] = 1 - f(\btheta^*_X \cdot \pa(X)) \ge 1- f(||\btheta^*_X||_1)$.
Thus, condition $f(||\btheta^*_X||_1) \le 1-\upsilon$ is suffice to satisfy Inequality~\eqref{eq:parentChildZero}.
This means that the existence of $\upsilon$ can be reasonably achieved.

Denote the children of $X$ by $\Ch(X)$. 
The following lemma shows that Assumption~\ref{asm:parent2child} leads to the following general result:
\begin{lemma}
\label{lemma:compute_zeta}
For every node $X\in\bX\cup\{Y\}\setminus\{X_1\}$ and every value $\bv \in \{0,1\}^{|\bX|}$, we have
\begin{align}
    \left(\frac{\upsilon}{1-\upsilon}\right)^{\left|\Ch(X)\right|+1}\leq \frac{\Pr_{\bbvarepsilon,\bX,Y}(X=1|\bX\cup\{Y\}\setminus\{X\}=\bv)}{\Pr_{\bbvarepsilon,\bX,Y}(X=0|\bX\cup\{Y\}\setminus\{X\}=\bv)}
    \leq\left(\frac{1-\upsilon}{\upsilon}\right)^{\left|\Ch(X)\right|+1}.
\end{align}
\end{lemma}
\begin{proof}
Without loss of generality, we only prove the left hand side inequality then the other side can be proved symmetrically. 
Let $v_Z$ denote the entry in $\bv$ corresponding to $Z$, and $\bv_{\Pa(Z)}$ denote the sub-vector of $\bv$ corresponding to $\Pa(Z)$.
Concretely, we have
\begin{align}
	&\frac{\Pr_{\bbvarepsilon,\bX,Y}\{X=1|\bX\cup\{Y\}\setminus\{X\}=\bv\}}{\Pr_{\bbvarepsilon,\bX,Y}\{X=0|\bX\cup\{Y\}\setminus\{X\}=\bv\}} \nonumber \\
    &=\frac{\Pr_{\bbvarepsilon,\bX,Y}\{X=1,\bX\cup\{Y\}\setminus\{X\}=\bv\}}{\Pr_{\bbvarepsilon,\bX,Y}\{X=0,\bX\cup\{Y\}\setminus\{X\}=\bv\}} \nonumber \\
    &=\frac{\E_{\bbvarepsilon}[\Pr_{\bX,Y}\{X=1,\bX\cup\{Y\}\setminus\{X\}=\bv\} | \bbvarepsilon]}{\E_{\bbvarepsilon}[\Pr_{\bX,Y}\{X=0,\bX\cup\{Y\}\setminus\{X\}=\bv\}| \bbvarepsilon]} \nonumber \\
    &=\frac{\prod_{Z\in \bX\cup\{Y\}\setminus(\{X\} \cup \Ch(X))} \E_{\varepsilon_Z}[P(Z=v_Z|\Pa(Z)=\bv_{\Pa(Z)})| \varepsilon_Z]}{\prod_{Z\in \bX\cup\{Y\}\setminus(\{X\} \cup \Ch(X))}{\E_{\varepsilon_Z}[ P(Z=v_Z|\Pa(Z)=\bv_{\Pa(Z)})| \varepsilon_Z]}}\times\frac{ \E_{\varepsilon_X}[ P(X=1|\Pa(X)=\bv_{\Pa(X)})| \varepsilon_X]}{ \E_{\varepsilon_X}[ P(X=0|\Pa(X)=\bv_{\Pa(X)})| \varepsilon_X]} \nonumber \\
    &\quad \quad \times\frac{\prod_{Z\in \Ch(X)}\E_{\varepsilon_Z}[ P(Z=v_Z|\Pa(Z)\setminus\{X\}=\bv_{\Pa(Z)\setminus X},X=1)| \varepsilon_Z]}{\prod_{Z\in \Ch(X)} \E_{\varepsilon_Z}[ P(Z=v_Z|\Pa(Z)\setminus\{X\}=\bv_{\Pa(Z)\setminus X},X=0)| \varepsilon_Z]}  \label{eq:decomposition}\\
    &\geq \frac{\upsilon}{1-\upsilon}\times \left(\frac{\upsilon}{1-\upsilon}\right)^{|\Ch(X)|} \nonumber \\
    &= \left(\frac{\upsilon}{1-\upsilon}\right)^{\left|\Ch(X)\right|+1}, \nonumber
\end{align}
where Eq.~\eqref{eq:decomposition} applies the decomposition rule of the Bayesian causal model \cite{russell2021artificial}, and then the mutual independence of $\varepsilon_X$'s for all $X\in \bX\cup \{Y\}$.
\end{proof}

Now we can prove the following lemma which shows that the $\zeta$ constant in Assumption~\ref{asm:glm_4} exists.

\begin{lemma}
\label{lemma:prove_zeta}
With Assumption~\ref{asm:parent2child} in Markovian BGLM $G=(\bX\cup\{Y\},E)$, Assumption \ref{asm:glm_4} holds for\begin{align*}
    \zeta=\frac{\upsilon^{D_{\text{out}}+1}}{\upsilon^{D_{\text{out}}+1}+(1-\upsilon)^{D_{\text{out}}+1}},
\end{align*}
where $D_{\text{out}}=\max_{X\in\bX\cup\{Y\}\setminus\{X_1\}}|\Ch(X)|$.
\end{lemma}

\begin{proof}
With Lemma~\ref{lemma:compute_zeta}, we can deduce that for any $X\in\bX\cup\{Y\}\setminus\{X_1\}$ and $X'\in \Pa(X)$, we have
\begin{align*}
    &\frac{\Pr_{\bbvarepsilon, \bX, Y}\left\{X'=1|\Pa(X)\setminus\{X'\}=\pa(X)\setminus\{x'\}\right\}}{\Pr_{\bbvarepsilon, \bX, Y}\left\{X'=0|\Pa(X)\setminus\{X'\}=\pa(X)\setminus\{x'\}\right\}}\\
    &=\frac{\Pr_{\bbvarepsilon, \bX, Y}\left\{X'=1,\Pa(X)\setminus\{X'\}=\pa(X)\setminus\{x'\}\right\}}{\Pr_{\bbvarepsilon, \bX, Y}\left\{X'=0,\Pa(X)\setminus\{X'\}=\pa(X)\setminus\{x'\}\right\}}\\
    &=\frac{\sum_{\bx\cup\{y\}\setminus\pa(X)\in\{0,1\}^{n-|\Pa(X)|}}\Pr_{\bbvarepsilon, \bX, Y}\left\{X'=1,\bX\cup\{Y\}\setminus\{X'\}=\bx\cup\{y\}\setminus\{x'\}\right\}}
    {\sum_{\bx\cup\{y\}\setminus\pa(X)\in\{0,1\}^{n-|\Pa(X)|}}\Pr_{\bbvarepsilon, \bX, Y}\left\{X'=0,\bX\cup\{Y\}\setminus\{X'\}=\bx\cup\{y\}\setminus\{x'\}\right\}}\\
    &\geq\min_{\bx\cup\{y\}\setminus\pa(X)\in\{0,1\}^{n-|\Pa(X)|}}\frac{\Pr_{\bbvarepsilon, \bX, Y}\left\{X'=1,\bX\cup\{Y\}\setminus\{X'\}=\bx\cup\{y\}\setminus\{x'\}\right\}}
    {\Pr_{\bbvarepsilon, \bX, Y}\left\{X'=0,\bX\cup\{Y\}\setminus\{X'\}=\bx\cup\{y\}\setminus\{x'\}\right\}}\\
    &\geq \left(\frac{\upsilon}{1-\upsilon}\right)^{\left|\Ch(X')\right|+1}.
\end{align*}
Similarly, we can prove the upper bound $\left(\frac{1-\upsilon}{\upsilon}\right)^{\left|\Ch(X')\right|+1}$ of it. Now we can deduce that $\Pr_{\bbvarepsilon, \bX, Y}\left\{X'=1|\Pa(X)\setminus\{X'\}=\pa(X)\setminus\{x'\}\right\}$ and $\Pr_{\bbvarepsilon, \bX, Y}\left\{X'=0|\Pa(X)\setminus\{X'\}=\pa(X)\setminus\{x'\}\right\}$ are both no less than $\frac{\upsilon^{|\Ch(X')|+1}}{\upsilon^{|\Ch(X')|+1}+(1-\upsilon)^{|\Ch(X')|+1}}$. This means that we can take $\zeta=\frac{\upsilon^{D_{\text{out}}+1}}{\upsilon^{D_{\text{out}}+1}+(1-\upsilon)^{D_{\text{out}}+1}}$ in Assumption \ref{asm:glm_4}.
\end{proof}

$D_{\text{out}}$ is the maximum out-degree in the causal graph $G$.
Lemma~\ref{lemma:prove_zeta} shows that when $D_{\text{out}}$ is small, we can have a reasonable $\zeta$ value for Assumption~\ref{asm:glm_4}.
Note that small $D_{\text{out}}$ is a sufficient but not necessary condition for a reasonable $\zeta$ value in Assumption~\ref{asm:glm_4}.
Intuitively, as long as no node dominates or is dominated by other nodes, Assumption~\ref{asm:glm_4} is likely to hold.

\section{Proofs for the Online Algorithm of BGLM CCB Problem  (Section~\ref{sec.glm})}
\label{app:glm}

\subsection{A Technical Lemma}
\label{app:lemma_init}
We first proof the following technical lemma, as a consequence Assumption~\ref{asm:glm_4}.
The lemma enables the use of Lecu\'{e} and Mendelson's inequality in our later theoretical analysis. 

Let $\Sphere(d)$ denote the sphere of the $d$-dimensional unit ball.
\begin{restatable}{lemma}{lemmainit}
	\label{lemma:init_condition}
	For any $\bv=(v_1,v_2,\cdots,v_{|\Pa(X)|})\in \Sphere(|\Pa(X)|)$ and any $X\in{\bX}\cup\{Y\}$ in a BGLM that 
	satisfies Assumption \ref{asm:glm_4}, we have
	\begin{align*}
	\Pr_{\bbvarepsilon, \bX, Y}\left\{\left|\Pa(X)\cdot \bv\right|\geq\frac{1}{\sqrt{4n^2-8n+1}}\right\}&\geq\zeta,
	\end{align*}
	where $\Pa(X)$ is the random vector generated by the natural Bayesian propagation in BGLM $G$ with no interventions
	(except for setting $X_1$ to 1).
\end{restatable}

\begin{proof}
	The lemma is proved using the idea of Pigeonhole principle.
	Let $\Pa(X)=(X_{i_1}=X_1,X_{i_2},X_{i_3},\cdots, X_{i_{|\Pa(X)|}})$ as the random vector and  
	$\pa(X)=(x_1=1,x_{i_1},x_{i_2},x_{i_3},\cdots, x_{i_{|\Pa(X)|}})$ as a possible valuation of $\Pa(X)$. 
	Without loss of generality, we suppose that $\left|v_2\right|\geq \left|v_3\right|\geq\cdots\ge\left|v_{|\Pa(X)|}\right| $. 
	For simplicity, we denote $n_0=\sqrt{n-2}+\frac{1}{2\sqrt{n-2}}$. If $\left|v_1\right|\geq\frac{n_0}{\sqrt{n_0^2+1}}$, 
	we can deduce that 
	\begin{align}
	|\pa(X)\cdot \bv|&\geq\left|v_1\right|-\left|v_2\right|-\left|v_3\right|-\cdots - \left|v_{|\Pa(X)|}\right| \nonumber\\
	&\geq\frac{n_0}{\sqrt{n_0^2+1}}-\sqrt{(n-2)\left(\left|v_2\right|^2+\left|v_3\right|^2+\cdots \left|v_{|\Pa(X)|}\right|^2\right)} \label{eq:cauchy}\\
	&\ge \frac{n_0}{\sqrt{n_0^2+1}}-\sqrt{(n-2)\left(1-\frac{n_0^2}{n_0^2+1}\right)} \label{eq:sphere}\\
	&=\frac{1}{2\sqrt{(n_0^2+1)(n-2)}} = \frac{1}{\sqrt{4n^2-8n+1}}, \nonumber 
	\end{align}
	where Inequality~\eqref{eq:cauchy} is by the Cauchy-Schwarz inequality and the fact that $|\Pa(X)|\leq n-1$, and 
	Inequality~\eqref{eq:sphere} uses the fact that $\bv \in \Sphere(|\Pa(X)|)$.
	Thus, when $\left|v_1\right|\geq\frac{n_0}{\sqrt{n_0^2+1}}$, the event $\left|\Pa(X)\cdot \bv\right|\geq\frac{1}{\sqrt{4n^2-8n+1}}$
	holds deterministically.
	Otherwise, when $\left|v_1\right| < \frac{n_0}{\sqrt{n_0^2+1}}$, 
	we use the fact that $|v_2|$ is the largest among $|v_2|, |v_3|, \ldots$ and deduce that
	\begin{align} \label{eq:inequalityv2}
	\left|v_2\right|\geq\frac{1}{\sqrt{n-2}}\sqrt{\left|v_2\right|^2+\left|v_3\right|^2+\cdots}
	\geq\frac{\sqrt{1-\left(\frac{n_0}{\sqrt{n_0^2+1}}\right)^2}}{\sqrt{n-2}}
	=\frac{2}{\sqrt{4n^2-8n+1}}.
	\end{align}
	
	Therefore, using the fact that 
	\begin{align*}
	&\Pr_{\bbvarepsilon, \bX,Y}\left\{X_{i_1}=1,X_{i_2}=x_{i_2},X_{i_3}=x_{i_3},\cdots\right\}= \\
	&\Pr_{\bbvarepsilon, \bX,Y}\left\{X_{i_2}=x_{i_2}|X_{i_1}=1,X_{i_3}=x_{i_3},\cdots\right\}\cdot\Pr_{\bbvarepsilon, \bX,Y}\left\{(X_{i_1}=1,X_{i_3}=x_{i_3},\cdots\right\}\geq \zeta \Pr_{\bbvarepsilon, \bX,Y}\left\{X_{i_1}=1,X_{i_3}=x_{i_3},\cdots\right\}
	\end{align*}
	
	and $\sum_{x_{i_3},x_{i_4},\cdots} \Pr_{\bbvarepsilon, \bX,Y}\left\{X_{i_1}=1,X_{i_3}=x_{i_3},\cdots\right\}=1$, we have
	
	\begin{align}
		&\Pr_{\bbvarepsilon, \bX,Y}\left\{\left|\Pa(X)\cdot \bv\right|\geq\frac{1}{\sqrt{4n^2-8n+1}}\right\} \nonumber\\
		&=\sum_{x_{i_3},x_{i_4},\cdots}\Pr\{X_{i_1}=1,X_{i_2}=1,X_{i_3}=x_{i_3},\cdots\}\cdot  \I\left\{ |(1,1,x_{i_3},x_{i_4},\cdots)\cdot (v_1,v_2,v_3,\cdots)| \geq\frac{1}{\sqrt{4n^2-8n+1}}\right\} \nonumber\\
		&\quad\quad+\sum_{x_{i_3},x_{i_4},\cdots}\Pr\{X_{i_1}=1,X_{i_2}=0,X_{i_3}=x_{i_3},\cdots\}\cdot \I\left\{|(1,0,x_{i_3},x_{i_4},\cdots)\cdot (v_1,v_2,v_3,\cdots)| \geq\frac{1}{\sqrt{4n^2-8n+1}}\right\} \nonumber\\
		&\geq \sum_{x_{i_3},x_{i_4},\cdots}\zeta \Pr\{X_{i_1}=1,X_{i_3}=x_{i_3},X_{i_4}=x_{i_4}\cdots\}
		\cdot \I\left\{|(1,1,x_{i_3},x_{i_4},\cdots)\cdot (v_1,v_2,v_3,\cdots)| \geq\frac{1}{\sqrt{4n^2-8n+1}}\right\} \nonumber\\
		&\quad\quad+\sum_{x_{i_3},x_{i_4},\cdots}\zeta \Pr\{X_{i_1}=1,X_{i_3}=x_{i_3},X_{i_4}=x_{i_4},\cdots\}
		\cdot \I\left\{|(1,0,x_{i_3},x_{i_4},\cdots)\cdot (v_1,v_2,v_3,\cdots)| \geq\frac{1}{\sqrt{4n^2-8n+1}}\right\} \nonumber\\
		&=\zeta \cdot \sum_{x_{i_3},x_{i_4},\cdots}\Pr\{X_{i_1}=1,X_{i_3}=x_{i_3},X_{i_4}=x_{i_4},\cdots\} \left(\I\left\{|(1,1,x_{i_3},x_{i_4},\cdots)\cdot (v_1,v_2,v_3,\cdots)| \geq\frac{1}{\sqrt{4n^2-8n+1}}\right\} \right. \nonumber \\
		&\quad\quad \left. +\I\left\{|(1,0,x_{i_3},x_{i_4},\cdots)\cdot (v_1,v_2,v_3,\cdots)| \geq\frac{1}{\sqrt{4n^2-8n+1}}\right\} \right) \nonumber\\
		&\geq \zeta\sum_{x_{i_3},x_{i_4},\cdots}\Pr\{X_{i_1}=1,X_{i_3}=x_{i_3},X_{i_4}=x_{i_4},\cdots\} \label{eq:identityrelax} \\
		&=\zeta, \nonumber
		\end{align}
	which is exactly what we want to prove. 
	Inequality~\eqref{eq:identityrelax} holds because 
	otherwise, at least for some $x_{i_3}, x_{i_4}, \ldots$, both indicators on the left-hand side of the inequality have to be $0$, which implies that
	\begin{align}
	\left|(1,1,x_{i_3},x_{i_4},\cdots)\cdot (v_1,v_2,v_3,\cdots)-(1,0,x_{i_3},x_{i_4},\cdots)\cdot (v_1,v_2,v_3,\cdots)\right|=\left|v_2\right|< \frac{2}{\sqrt{4n^2-8n+1}},
	\end{align}
	but this contradicts to Inequality~\eqref{eq:inequalityv2}.
\end{proof}

\subsection{Proof of the Learning Problem for BGLM CCB Problem (Lemma \ref{thm.learning_glm})}
\label{app:learning_glm}
\thmlearningglm*

\begin{proof}
Note that $\hat{\btheta}_{t,X}$ satisfies that $\triangledown L_{t,X}(\hat{\btheta}_{X})=0$, where the gradient is\begin{align*}
    \triangledown L_{t,X}(\btheta_X)=\sum_{i=1}^t[X^t-f_X(\bV_{i,X}^\intercal\btheta_X)]\bV_{i,X}.
\end{align*}

Define $G(\btheta_X)=\sum_{i=1}^t(f_X(\bV_{i,X}^\intercal\btheta_X)-f_X(\bV_{i,X}^\intercal\btheta_X^*))\bV_{i,X}$. Then we have $G(\btheta_X^*)=0$ and $G(\hat{\btheta}_{t,X})=\sum_{i=1}^t\varepsilon'_{i,X}\bV_{i,X}$, where $\varepsilon_{i,X}'=X^i-f_X(\bV_{i,X}^\intercal\btheta_X^*)$. Note that $\mathbb{E}[\varepsilon'_{i,X}|\bV_{i,X}]=0$ and $\varepsilon'_{i,X}=X^i-f_X(\bV_{i,X}^\intercal\btheta^*_X)\in[-1,1]$ since $X^i\in\{0,1\}$ and $f_X(\bV_{i,X}^\intercal\btheta^*_X)\in[0,1]$. Therefore, $\varepsilon'_{i,X}$ is $1$-sub-Gaussian. Furthermore, define $Z=G(\hat{\btheta}_{t,X})=\sum_{i=1}^t\varepsilon'_{i,X}\bV_{i,X}$ for convenience. All the remaining notations in this proof are corresponding to some $X\in {\bf X}\cup\{Y\}$.

{\bf Step 1: Consistency of $\hat{\btheta}_{t,X}$.}  We first prove the consistency of $\hat{\btheta}_{t,X}$. For any $\btheta_1,\btheta_2\in\mathbb{R}^{|\Pa(X)|}$, by the mean value theorem, $\exists\overline{\btheta}=s\btheta_1+(1-s)\btheta_2,0<s<1$ such that\begin{align*}
    G(\btheta_1)-G(\btheta_2)&=\left[\sum_{i=1}^t\dot{f}_X(\bV_{i,X}^\intercal\overline{\btheta})\bV_{i,X}\bV_{i,X}^\intercal\right](\btheta_1-\btheta_2)\\
    &\triangleq F(\overline{\btheta})(\btheta_1-\btheta_2).
\end{align*}
Since $\dot{f}_X\geq \kappa>0$ as we state in Assumption \ref{asm:glm_3} and $\lambda_{\min}(M_{t,X})>0$, we have $F(\overline{\btheta})\succ\kappa M_{t,X}$ and for any $\btheta_1,\btheta_2$, we have $(\btheta_1-\btheta_2)^\intercal(G(\btheta_1)-G(\btheta_2))\geq (\btheta_1-\btheta_2)^\intercal(\kappa M_{t,X})(\btheta_1-\btheta_2)>0$. Now we deduce that $G(\btheta)$ is an injection from $\mathbb{R}^{|\Pa(X)|}$ to $\mathbb{R}^{|\Pa(X)|}$ and therefore $G^{-1}$ is well defined. We have $\hat{\btheta}_{t,X}=G^{-1}(Z)$.

For any $\btheta\in\Theta$, now we have a $\overline{\btheta}_X$ such that \begin{align*}
    \left\|Z\right\|^2_{M_{t,X}^{-1}}&=\left\|G(\hat{\btheta}_{t,X})\right\|^2_{M_{t,X}^{-1}}\\
    &=\left\|G(\hat{\btheta}_{t,X})-G(\btheta^*_X)\right\|^2_{M_{t,X}^{-1}}\\
    &=(\hat{\btheta}_{t,X}-\btheta^*)^\intercal F(\overline{\btheta})M_{t,X}^{-1}F(\overline{\btheta})(\hat{\btheta}_{t,X}-\btheta^*_X)\\
    &\geq \kappa^2\lambda_{\min}(M_{t,X})\left\|\hat{\btheta}_{t,X}-\btheta^*_X\right\|^2.
\end{align*}

Then we need a bound for $\left\|Z\right\|^2_{M_{t,X}^{-1}}$ so that we can get a bound for $\left\|\hat{\btheta}_{t,X}-\btheta_X^*\right\|$. The next lemma solves this.

\begin{lemma}[\cite{zhang2022online}]
For any $\delta>0$, define the following event $\mathcal{E}_G=\{\left\|Z\right\|_{M_{t,X}^{-1}}\leq4\sqrt{|\Pa(X)|+\log\frac{1}{\delta}}\}$. Then $\mathcal{E}_G$ holds with probability at least $1-\delta$.
\end{lemma}

\begin{proof}
The proof of this lemma is a simple rewriting of the proof of Lemma 10 in \cite{zhang2022online}.
\end{proof}

Combining this lemma and the deduction we made just now, we can get\begin{align*}
    \left\|\hat{\btheta}_{t,X}-\btheta^*_X\right\|&\leq\frac{\left\|Z\right\|_{M_{t,X}^{-1}}}{\kappa\sqrt{\lambda_{\min}(M_{t,X})}}\leq\frac{4}{\kappa}\sqrt{\frac{|\Pa(X)|+\log\frac{1}{\delta}}{\lambda_{\min}(M_{t,X})}}\leq 1.
\end{align*}

{\bf Step 2: Normality of $\hat{\btheta}_X$.} Now we assume that $\mathcal{E}$ holds in the following proof. Define $\Delta=\hat{\btheta}_{t,X}-\btheta^*_X$, then $\exists s\in[0,1]$ such that $Z=G(\hat{\btheta}_{t,X})-G(\btheta_X^*)=(H+E)\Delta$, where $\overline{\btheta}=s\btheta^*_X+(1-s)\hat{\btheta}_{t,X}$, $H=F(\btheta^*_X)=\sum_{i=1}^t\dot{f}_X(\bV_{i,X}^\intercal\btheta^*_X)\bV_{i,X}\bV_{i,X}^\intercal$ and $E=F(\overline{\btheta})-F(\btheta^*_X)$. According to the mean value theorem, we have\begin{align*}
    E&=\sum_{i=1}^t(\dot{f}_X(\bV_{i,X}\cdot\overline{\btheta})-\dot{f}_X(\bV_{i,X}\cdot \btheta^*_X))\bV_{i,X}\bV_{i,X}^\intercal\\
    &=\sum_{i=1}^t\ddot{f}_X(r_i)\bV_{i,X}^\intercal\Delta \bV_{i,X}\bV_{i,X}^\intercal
\end{align*}
for some $r_i\in\mathbb{R}$. Because $\ddot{f}_X\leq L_{f_X}^{(2)}$ and $s\in[0,1]$, for any $\bv\in \mathbb{R}^{|\Pa(X)|}\setminus \{\mathbf{0}\}$, we have\begin{align*}
    \bv^TH^{-1/2}EH^{-1/2}\bv
    &=(1-s)\sum_{i=1}^t\ddot{f}_X(r_i)\bV_{i,X}^\intercal\Delta\left\|\bv^\intercal H^{-1/2}\bV_{i,X}\right\|\\
    &\leq\sum_{i=1}^tL_{f_X}^{(2)}\left\|\bV_{i,X}\right\|\left\|\Delta\right\|\left\|\bv^\intercal H^{-1/2}\bV_{i,X}\right\|^2\\
    &\leq L_{f_X}^{(2)}\sqrt{|\Pa(X)|}\left\|\Delta\right\|(\bv^\intercal H^{-1/2}(\sum_{i=1}^t\bV_{i,X}\bV_{i,X}^\intercal)H^{-1/2}\bv)\\
    &\leq\frac{L_{f_X}^{(2)}\sqrt{|\Pa(X)|}}{\kappa}\left\|\Delta\right\|\left\|\bv\right\|^2,
\end{align*}
where we have use the fact that $\left\|\bV_{i,X}\right\|\leq\sqrt{|\Pa(X)|}$ for the second inequality. Therefore, we have\begin{align*}
    \left\|H^{-1/2}EH^{-1/2}\right\|
    &\leq\frac{L_{f_X}^{(2)}\sqrt{|\Pa(X)|}}{\kappa}\left\|\Delta\right\|\\
    &\leq\frac{4L_{f_X}^{(2)}\sqrt{|\Pa(X)|}}{\kappa^2}\sqrt{\frac{|\Pa(X)|+\ln\frac{1}{\delta}}{\lambda_{\min}(M_{t,X})}}.
\end{align*}
When $\lambda_{\min}(M_{t,X})\geq 64\left(L_{f_X}^{(2)}\right)^2|\Pa(X)|\left(|\Pa(X)|+\ln\frac{1}{\delta}\right)/\kappa^4$, we have $\left\|H^{-1/2}EH^{-1/2}\right\|\leq\frac{1}{2}$. 

Now we can prove this theorem. For any $\bv\in\mathbb{R}^{|\Pa(X)|}$, we have\begin{align*}
    \bv^\intercal(\hat{\btheta}_{t,X}-\btheta^*_X)&=\bv^\intercal(H+E)^{-1}Z\\
    &=\bv^\intercal H^{-1}Z-\bv^\intercal H^{-1}E(H+E)^{-1}Z.
\end{align*}
Note that $(H+E)^{-1}$ exists since $H+E=F(\overline{\btheta})\succ\kappa M_{t,X}\succ 0$.

For the first term, define $D\triangleq(\bV_{1,X},\bV_{2,X},\cdots,\bV_{t,X})^\intercal\in\mathbb{R}^{t\times|\Pa(X)|}$. Note that $D^\intercal D=\sum_{i=1}^t\bV_{i,X}\bV_{i,X}^\intercal=M_{t,X}$. By the Hoeffding's inequality \cite{hoeffding1994probability}, \begin{align*}
    \Pr(\left|\bv^\intercal H^{-1}Z\geq a\right|)&\leq\exp\left(-\frac{a^2}{2\left\|\bv^\intercal H^{-1}D^\intercal\right\|^2}\right)\\
    &=\exp\left(-\frac{a^2}{2\bv^\intercal H^{-1}D^\intercal DH^{-1}\bv}\right)\\
    &\leq \exp\left(-\frac{a^2\kappa^2}{2\left\|\bv\right\|^2_{M_{t,X}^{-1}}}\right).
\end{align*} 
The last inequality holds because $H\succeq \kappa M_{t,X}=\kappa D^\intercal D$. From this we deduce that with probability at least $1-2\delta$, $\left| \bv^\intercal H^{-1}Z\right|\leq\frac{\sqrt{2\ln{1/\delta}}}{\kappa}\left\|\bv\right\|_{M_{t,X}^{-1}}$.

For the second term, we have \begin{align*}
    \left|\bv^\intercal H^{-1}E(H+E)^{-1} Z\right|
    &\leq \left\|\bv\right\|_{H^{-1}}\left\|H^{-\frac{1}{2}}E(H+E)^{-1}Z\right\|\\
    &\leq \left\|\bv\right\|_{H^{-1}}\left\|H^{-\frac{1}{2}}E(H+E)^{-1}H^{\frac{1}{2}}\right\|\left\|Z\right\|_{H^{-1}}\\
    &\leq \frac{1}{\kappa}\left\|\bv\right\|_{M^{-1}_{t,X}}\left\|H^{-\frac{1}{2}}E(H+E)^{-1}H^{\frac{1}{2}}\right\|\left\|Z\right\|_{M^{-1}_{t,X}}.
\end{align*}
where the last inequality is due to the fact that $H\succeq \kappa M_{t,X}$. Since $(H+E)^{-1}=H^{-1}-H^{-1}E(H+E)^{-1}$, we have\begin{align*}
    \left\|H^{-\frac{1}{2}}E(H+E)^{-1}H^{\frac{1}{2}}\right\|
    &=\left\|H^{-\frac{1}{2}}E(H^{-1}-H^{-1}E(H+E)^{-1})^{-1}H^{\frac{1}{2}}\right\|\\
    &=\left\|H^{-\frac{1}{2}}EH^{\frac{1}{2}}+H^{-\frac{1}{2}}EH^{-1}E(H+E)^{-1}H^{\frac{1}{2}}\right\|\\
    &\leq \left\|H^{-\frac{1}{2}}EH^{\frac{1}{2}}\right\|+\left\|H^{-\frac{1}{2}}EH^{-\frac{1}{2}}\right\|\left\|H^{-\frac{1}{2}}E(H+E)^{-1}H^{\frac{1}{2}}\right\|.
\end{align*}
By solving this inequality we get\begin{align*}
    \left\|H^{-\frac{1}{2}}E(H+E)^{-1}H^{\frac{1}{2}}\right\|
    &\leq\frac{\left\|H^{-\frac{1}{2}}EH^{-\frac{1}{2}}\right\|}{1-\left\|H^{-\frac{1}{2}}EH^{-\frac{1}{2}}\right\|}\\
    &\leq2\left\|H^{-\frac{1}{2}}EH^{-\frac{1}{2}}\right\|\\
    &\leq\frac{8L_{f_X}^{(2)}}{\kappa^2}\sqrt{\frac{|\Pa(X)|(|\Pa(X)+\log\frac{1}{\delta}|)}{\lambda_{\min}(M_{t,X})}}.
\end{align*} 
Therefore, we have\begin{align*}
    \left|v^\intercal H^{-1}E(H+E)^{-1} Z\right|
    &\leq\frac{32L_{f_X}^{(2)}\sqrt{|\Pa(X)|}(|\Pa(X)|+\log\frac{1}{\delta})}{\kappa^3\sqrt{\lambda_{\min}(M_{t,X})}}\left\|\bv\right\|_{M_{t,X}^{-1}}.
\end{align*}

Thus, we have\begin{align*}
    \left|\bv^\intercal(\hat{\btheta}_{t,X}-\btheta^*_X)\right|
    &\leq\left(\frac{32L_{f_X}^{(2)}\sqrt{|\Pa(X)|}(|\Pa(X)|+\log\frac{1}{\delta})}{\kappa^3\sqrt{\lambda_{\min}(M_{t,X})}}
    +\frac{\sqrt{2\ln{1/\delta}}}{\kappa}\right)\left\|\bv\right\|_{M_{t,X}^{-1}}\\
    &\leq\frac{3}{\kappa}\sqrt{\log(1/\delta)}\left\|\bv\right\|_{M_{t,X}^{-1}},
\end{align*}
when\begin{align*}
    \lambda_{\min}(M_{t,X})\geq \frac{512|\Pa(X)|(L_{f_X}^{(2)})^2}{\kappa^4}\left(|\Pa(X)|^2+\ln\frac{1}{\delta}\right).
\end{align*}
\end{proof}

\subsection{Proof of the GOM Bounded Smoothness Condition for BGLM (Lemma \ref{thm.gom_glm})}
\label{app:gom_glm}
\thmgomglm*

\begin{proof}
Firstly, we have\begin{align*}
    \left|\sigma({\bf S},\btheta^1)-\sigma({\bf S},\btheta^2)\right|
    &=\mathbb{E}_{\bgamma\sim(\mathcal{U}[0,1])^n,{\bbvarepsilon}}\left[\I\{Y\text{ is influenced under } \btheta^1,\bgamma\}
    \neq \I\{Y\text{ is influenced under } \btheta^2,\bgamma\}\right].
\end{align*}
Then we define the following event $\mathcal{E}_0^{\bbvarepsilon}(X)$ as below:\begin{align*}
    \mathcal{E}_0^{\bbvarepsilon}(X)=\left\{\bgamma|\I\{X\text{ is activated under }\bgamma,\btheta^1\}\neq \I\{X\text{ is activated under }\bgamma,\btheta^2\}\right\}.
\end{align*}
Thus we have\begin{align*}
    \left|\sigma({\bf S},\btheta^1)-\sigma({\bf S},\btheta^2)\right|&\leq \mathbb{E}_{{\bbvarepsilon}}\left[ \Pr_{\bgamma\sim(\mathcal{U}[0,1])^n}\{\mathcal{E}_0^{\bbvarepsilon}(Y)\}\right].
\end{align*}

Let $\phi^{\bbvarepsilon}(\btheta, \bgamma) = (\phi^{\bbvarepsilon}_0(\btheta, \bgamma) = S, \phi^{\bbvarepsilon}_1(\btheta, \bgamma),\cdots ,\phi^{\bbvarepsilon}_n(\btheta, \bgamma))$ be the sequence of activation sets given weight vector $\btheta$, $0$-mean noise ${\bbvarepsilon}$ and threshold factor $\bgamma$. More specifically, $\phi_i(\btheta, \bgamma)$ is the set of nodes activated by time step $i$, i.e. $\phi^{\bbvarepsilon}_i(\btheta, \bgamma)\setminus\phi^{\bbvarepsilon}_{i-1}(\btheta, \bgamma)\in\{\emptyset, \{X_i\}\}$. For every node $X\in {\bf X}_{{\bf S}, Y}$, we define the event that $X$ is the first node that has different activation under $\btheta^1$ and $\btheta^2$ as below:\begin{align*}
    \mathcal{E}_1^{\bbvarepsilon}(X)=\{\bgamma|\exists\tau\in[n],\forall\tau'<\tau,\phi^{\bbvarepsilon}_{\tau'}(\btheta^1,\bgamma)=\phi^{\bbvarepsilon}_{\tau'}(\btheta^2,\bgamma),X\in(\phi^{\bbvarepsilon}_\tau(\btheta^1,\bgamma)\setminus\phi^{\bbvarepsilon}_\tau(\btheta^2,\bgamma)\cup (\phi_{\tau}^{\bbvarepsilon}(\btheta^2,\bgamma)\setminus\phi^{\bbvarepsilon}_\tau(\btheta^1,\bgamma)))\}.
\end{align*}
Then we have $\mathcal{E}_0^{\bbvarepsilon}(Y)\subseteq \cup_{X\in {\bf X}_{{\bf S}, Y}}\mathcal{E}_1^{\bbvarepsilon}(X)$. Now we define some other events as below:\begin{align*}
    &\mathcal{E}_{2,0}^{\bbvarepsilon}(X,\tau)=\{\bgamma|\forall\tau'<\tau,\phi_{\tau'}^{\bbvarepsilon}(\btheta^1,\bgamma)=\phi_{\tau'}^{\bbvarepsilon}(\btheta^2,\bgamma),X\not\in\phi_{\tau-1}^{\bbvarepsilon}(\btheta^1,\bgamma)\},\\
    &\mathcal{E}_{2,1}^{\bbvarepsilon}(X,\tau)=\{\bgamma|\forall\tau'<\tau,\phi_{\tau'}^{\bbvarepsilon}(\btheta^1,\bgamma)=\phi_{\tau'}^{\bbvarepsilon}(\btheta^2,\bgamma),X\in\phi_\tau^{\bbvarepsilon}(\btheta^1,\bgamma)\setminus\phi_\tau^{\bbvarepsilon}(\btheta^2,\bgamma)\},\\
    &\mathcal{E}_{2,2}^{\bbvarepsilon}(X,\tau)=\{\bgamma|\forall\tau'<\tau,\phi_{\tau'}^{\bbvarepsilon}(\btheta^1,\bgamma)=\phi_{\tau'}^{\bbvarepsilon}(\btheta^2,\bgamma),X\in\phi_\tau^{\bbvarepsilon}(\btheta^2,\bgamma)\setminus\phi_\tau^{\bbvarepsilon}(\btheta^1,\bgamma)\},\\
    &\mathcal{E}_{3,1}^{\bbvarepsilon}(X,\tau)=\{\bgamma|X\in\phi_\tau^{\bbvarepsilon}(\btheta^1,\bgamma)\setminus\phi_\tau^{\bbvarepsilon}(\btheta^2,\bgamma)\},\mathcal{E}_{3,2}^{\bbvarepsilon}(X,\tau)=\{\bgamma|X\in\phi_\tau^{\bbvarepsilon}(\btheta^2,\bgamma)\setminus\phi_\tau^{\bbvarepsilon}(\btheta^1,\bgamma)\}.
\end{align*}
Because $\mathcal{E}^{\bbvarepsilon}_{2,1}, \mathcal{E}^{\bbvarepsilon}_{2,2}$ are mutually exclusive, naturally we have\begin{align*}
    \Pr_{\bgamma\sim(\mathcal{U}[0,1])^n}\{\mathcal{E}^{\bbvarepsilon}_1(X)\}=\sum_{\tau=1}^n\Pr_{\bgamma\sim(\mathcal{U}[0,1])^n}\{\mathcal{E}^{\bbvarepsilon}_{2,1}(X,\tau)\}+\sum_{\tau=1}^n\Pr_{\bgamma\sim(\mathcal{U}[0,1])^n}\{\mathcal{E}^{\bbvarepsilon}_{2,2}(X,\tau)\}.
\end{align*}

We first bound $\Pr_{\bgamma\sim(\mathcal{U}[0,1])^n}\{\mathcal{E}^{\bbvarepsilon}_{2,1}(X,\tau)\}$. Now suppose that $\bgamma_{-X}$ is the vector of $\bgamma$ such that all the entries are fixed by entries of $\gamma$ except $\gamma_X$. Furthermore, the corresponding sub-event of $\mathcal{E}^{\bbvarepsilon}_{2,1}(u,\tau)$ is defined as $\mathcal{E}^{\bbvarepsilon}_{2,1}(X,\tau,\bgamma_{-X})\subseteq\mathcal{E}^{\bbvarepsilon}_{2,1}(X,\tau)$. Similarly, we can define $\mathcal{E}^\epsilon_{2,0}(X,\tau,\bgamma_{-X})\subseteq\mathcal{E}^{\bbvarepsilon}_{2,0}(X,\tau)$ and $\mathcal{E}^{\bbvarepsilon}_{3,1}(X,\tau,\bgamma_{-X})\subseteq\mathcal{E}^{\bbvarepsilon}_{3,1}(X,\tau)$. 

According to the definitions, it is easy to observe that $\mathcal{E}^{\bbvarepsilon}_{2,1}(X,\tau,\bgamma_{-X})=\mathcal{E}^{\bbvarepsilon}_{3,1}(X,\tau,\bgamma_{-X})\cup\mathcal{E}^{\bbvarepsilon}_{2,0}(X,\tau,\bgamma_{-X})$ and $\mathcal{E}^{\bbvarepsilon}_{2,2}(X,\tau,\bgamma_{-X})=\mathcal{E}^{\bbvarepsilon}_{3,2}(X,\tau,\bgamma_{-X})\cup\mathcal{E}^{\bbvarepsilon}_{2,0}(X,\tau,\bgamma_{-X})$. So,\begin{align*}
    \Pr_{\bgamma\sim(\mathcal{U}[0,1])^n}\{\mathcal{E}^{\bbvarepsilon}_{2,1}(X,\tau)\}=\Pr_{\bgamma\sim(\mathcal{U}[0,1])^n}\{\mathcal{E}^{\bbvarepsilon}_{2,0}(X,\tau)\}\cdot \Pr_{\bgamma\sim(\mathcal{U}[0,1])^n}\{\mathcal{E}^{\bbvarepsilon}_{3,1}(X,\tau)|\mathcal{E}^{\bbvarepsilon}_{2,0}(X,\tau)\}.
\end{align*}
Then we also have \begin{align}
    \Pr_{\bgamma\sim(\mathcal{U}[0,1])^n}\{\mathcal{E}^{\bbvarepsilon}_{2,1}(X,\tau,\bgamma_{-X})\}=\Pr_{\bgamma\sim(\mathcal{U}[0,1])^n}\{\mathcal{E}^{\bbvarepsilon}_{2,0}(X,\tau,\bgamma_{-X})\}\cdot \Pr_{\bgamma\sim(\mathcal{U}[0,1])^n}\{\mathcal{E}^{\bbvarepsilon}_{3,1}(X,\tau,\bgamma_{-X})|\mathcal{E}^{\bbvarepsilon}_{2,0}(X,\tau,\bgamma_{-X})\}.\label{equation.1}
\end{align}

Similar equations also holds for $\mathcal{E}_{2,2}^{\bbvarepsilon}(X,\tau,\bgamma_{-X})$. By the monotonicity of BGLM, obviously, in $\mathcal{E}^{\bbvarepsilon}_{2,0}(X,\tau,\bgamma_{-X})$, the entry on $\gamma_X$ must be an interval from some lowest value to $1$. Let $\omega^{\bbvarepsilon}_{X,2,0}(\tau,\bgamma_{-X})$ to be the lowest value of this interval, then we have\begin{align}
    \label{equation.2}
    \Pr_{\bgamma_X\sim\mathcal{U}[0,1]}\{\mathcal{E}^{\bbvarepsilon}_{2,0}(X,\tau,\bgamma_{-X})\}=1-\omega^{\bbvarepsilon}_{X,2,0}(\tau,\bgamma_{-X}).
\end{align}
Then we denote that the set of nodes activated by time step $i$ under $\mathcal{E}^{\bbvarepsilon}_{2,0}(X,\tau,\bgamma_{-X})$ as $\phi^{\bbvarepsilon}_{i}(\mathcal{E}^{\bbvarepsilon}_{2,0}(X,\tau,\bgamma_{-X}))$. Moreover, we first assume that $\omega^{\bbvarepsilon}_{X,2,0}(\tau,\bgamma_{-X})<1$ or $\mathcal{E}^{\bbvarepsilon}_{2,0}(X,\tau,\bgamma_{-X})\neq\emptyset$.

Now we consider the value of \begin{align*}\Pr_{\bgamma\sim(\mathcal{U}[0,1])^n}\{\mathcal{E}^{\bbvarepsilon}_{3,1}(X,\tau,\bgamma_{-X})|\mathcal{E}^{\bbvarepsilon}_{2,0}(X,\tau,\bgamma_{-X})\}.\end{align*} This conditional probability means that conditioned on $\bgamma_X\geq \omega^{\bbvarepsilon}_{X,2,0}(\tau,\bgamma_{-X})$ and a fixed activated set $\phi^{\bbvarepsilon}_{\tau-1}(\mathcal{E}^{\bbvarepsilon}_{2,0}(X,\tau,\bgamma_{-X}))$
by time $\tau-1$, the probability that $X$ is activated at step $\tau$ under one of $\btheta^1$ and $\btheta^2$ but not both. Then if the event of this conditional probability holds, we have the following inequalities:\begin{align*}
    f_X\left(\sum_{X'\in \phi^{\bbvarepsilon}_{\tau-1}(\mathcal{E}^{\bbvarepsilon}_{2,0}(X,\tau,\bgamma_{-X}))\cap N(X)}\btheta^1_{X',X}\right)+\epsilon_X<\gamma_X
    \leq f_X\left(\sum_{X'\in \phi^{\bbvarepsilon}_{\tau-1}(\mathcal{E}^{\bbvarepsilon}_{2,0}(X,\tau,\bgamma_{-X}))\cap N(X)}\btheta^2_{X',X}\right)+\epsilon_X
\end{align*}
or\begin{align*}
    f_X\left(\sum_{X'\in \phi^{\bbvarepsilon}_{\tau-1}(\mathcal{E}^{\bbvarepsilon}_{2,0}(X,\tau,\bgamma_{-X}))\cap N(X)}\btheta^1_{X',X}\right)+\epsilon_X\geq\gamma_X
    > f_X\left(\sum_{X'\in \phi^{\bbvarepsilon}_{\tau-1}(\mathcal{E}^{\bbvarepsilon}_{2,0}(X,\tau,\bgamma_{-X}))\cap N(X)}\btheta^2_{X',X}\right)+\epsilon_X.
\end{align*}
Therefore, we can get\begin{align*}
    &\Pr_{\gamma_X\sim\mathcal{U}[0,1]}\{\mathcal{E}^{\bbvarepsilon}_{3,1}(X,\tau,\bgamma_{-X})\cup\mathcal{E}^{\bbvarepsilon}_{3,2}(X,\tau,\bgamma_{-X})|\mathcal{E}^{\bbvarepsilon}_{2,0}(X,\tau,\bgamma_{-X})\}\\
    &=\frac{\left|f_X\left(\sum_{X'\in \phi^{\bbvarepsilon}_{\tau-1}(\mathcal{E}^{\bbvarepsilon}_{2,0}(X,\tau,\bgamma_{-X}))\cap N(X)}\btheta^1_{X',X}\right)-f_X\left(\sum_{X'\in \phi^{\bbvarepsilon}_{\tau-1}(\mathcal{E}^{\bbvarepsilon}_{2,0}(X,\tau,\bgamma_{-X}))\cap N(X)}\btheta^2_{X',X}\right)\right|}
    {1-\omega^{\bbvarepsilon}_{X,2,0}(\tau,\bgamma_{-X})}
\end{align*}
so according to Eq.~\eqref{equation.1} and Eq.~\eqref{equation.2}, we get\begin{align}
    \label{equation.est_of_events}
    \Pr_{\gamma_X\sim\mathcal{U}[0,1]}\{\mathcal{E}^{\bbvarepsilon}_{2,1}(X,\tau,\bgamma_{-X})\cup\mathcal{E}^{\bbvarepsilon}_{2,2}(X,\tau,\bgamma_{-X})\}\leq L_{f_X}^{(1)}\left|\sum_{X'\in \phi^{\bbvarepsilon}_{\tau-1}(\mathcal{E}^{\bbvarepsilon}_{2,0}(X,\tau,\bgamma_{-X}))\cap N(X)}(\btheta^1_{X',X}-\btheta^2_{X',X})\right|.
\end{align}
When $\mathcal{E}^{\bbvarepsilon}_{2,0}(X,\tau,\bgamma_{-X})=\emptyset$, both the right side and the left side of the above inequality is zero, so we have this inequality holds in general.

Now we define $\mathcal{E}_{4,0}^{\bbvarepsilon}(X,\tau,\bgamma_{-X})=\{\bgamma=(\bgamma_{-X},\gamma_X)|X\not\in\phi_{\tau-1}^{{\bbvarepsilon}}(\btheta^1,\bgamma)\}$. Then obviously, we have $\mathcal{E}^{\bbvarepsilon}_{2,0}(X,\tau,\bgamma_{-X})\subseteq\mathcal{E}^{\bbvarepsilon}_{4,0}(X,\tau,\bgamma_{-X})$ and if $\mathcal{E}^{\bbvarepsilon}_{2,0}(X,\tau,\bgamma_{-X})\neq\emptyset$, for $i\leq \tau-1$, $\phi_i^{\bbvarepsilon}(\mathcal{E}^{\bbvarepsilon}_{2,0}(X,\tau,\bgamma_{-X}))=\phi_i^{\bbvarepsilon}(\mathcal{E}^\epsilon_{4,0}(X,\tau,\bgamma_{-X}))$. Therefore, we can relax Eq.~\eqref{equation.est_of_events} by\begin{align*}
    \Pr_{\gamma_X\sim\mathcal{U}[0,1]}\{\mathcal{E}^{\bbvarepsilon}_{2,1}(X,\tau,\bgamma_{-X})\cup\mathcal{E}^{\bbvarepsilon}_{2,2}(X,\tau,\bgamma_{-X})\}\leq L_{f_X}^{(1)}
    \left|\sum_{X'\in \phi^{\bbvarepsilon}_{\tau-1}(\mathcal{E}^{\bbvarepsilon}_{4,0}(X,\tau,\bgamma_{-X}))\cap N(X)}(\btheta^1_{X',X}-\btheta^2_{X',X})\right|.
\end{align*}
This also holds for $\mathcal{E}^{\bbvarepsilon}_{2,0}(X,\tau,\bgamma_{-X})=\emptyset$.

Now we can deduce that\begin{align*}
    \Pr_{\gamma\sim(\mathcal{U}[0,1])^n}\{\mathcal{E}^\epsilon_1(X)\}
    &=\int_{\bgamma_{-X}\in[0,1]^{n-1}}\sum_{\tau=1}^n\Pr_{\gamma_X\sim\mathcal{U}[0,1]}\{\mathcal{E}^{\bbvarepsilon}_{2,1}(X,\tau,\bgamma_{-X})
    \cup\mathcal{E}^{\bbvarepsilon}_{2,2}(X,\tau,\bgamma_{-X})\}\text{d}\bgamma_{-X}\\
    &=\sum_{\tau=1}^n\int_{\bgamma_{-X}\in[0,1]^{n-1}}\Pr_{\gamma_X\sim\mathcal{U}[0,1]}\{\mathcal{E}^{\bbvarepsilon}_{2,1}(X,\tau,\bgamma_{-X})
    \cup\mathcal{E}^{\bbvarepsilon}_{2,2}(X,\tau,\bgamma_{-X})\}\text{d}\bgamma_{-X}\\
    &\leq \sum_{\tau=1}^n\int_{\bgamma_{-X}\in[0,1]^{n-1}}\left|\sum_{X'\in \phi^{\bbvarepsilon}_{\tau-1}(\mathcal{E}^{\bbvarepsilon}_{4,0}(X,\tau,\bgamma_{-X}))\cap N(X)}
    (\theta^1_{X',X}-\theta^2_{X',X})\right|L_{f_X}^{(1)}\text{d}\bgamma_{-X}\\
    &=\sum_{\tau=1}^n\mathbb{E}_{\bgamma_{-X}\sim(\mathcal{U}[0,1])^{n-1}}\left[\left|\sum_{X'\in \phi^{\bbvarepsilon}_{\tau-1}(\mathcal{E}^{\bbvarepsilon}_{4,0}(X,\tau,\bgamma_{-X}))\cap N(X)}
    (\theta^1_{X',X}-\theta^2_{X',X})\right|\right]L_{f_X}^{(1)}.
\end{align*}

Combining this with the fact $\mathcal{E}_0^{\bbvarepsilon}\subseteq\cup_{X\in {\bf X}_{{\bf S}, Y}}\mathcal{E}_1^{\bbvarepsilon}(X)$, we have\begin{align*}
    &\left|\sigma({\bf S},\btheta^1)-\sigma({\bf S},\btheta^2)\right|\\
    &\leq \mathbb{E}_{\bbvarepsilon}\left[\sum_{X\in {\bf X}_{{\bf S}, Y}}\sum_{\tau=1}^n\mathbb{E}_{\bgamma_{-X}\sim(\mathcal{U}[0,1])^{n-1}}
    \left[\left|\sum_{X'\in \phi^{\bbvarepsilon}_{\tau-1}(\mathcal{E}^{\bbvarepsilon}_{4,0}(X,\tau,\bgamma_{-X}))\cap N(X)}(\theta^1_{X',X}
    -\theta^2_{X',X})\right|\right]L_{f_X}^{(1)}\right]\\
    &=\mathbb{E}\left[\sum_{X\in {\bf X}_{{\bf S}, Y}}\left|\bV_{X}(\btheta^1_X-\btheta^2_X)\right|L_{f_X}^{(1)}\right]
\end{align*}
which is what we want.
\end{proof}

\subsection{Proof of Regret Bound for Algorithm \ref{alg:glm-ucb} (Theorem \ref{thm.regret_glm})}
\label{app:regret_glm}

Before the proof, we propose a lemma in order to bound the sum of $\left\|\bV_{t,X}\right\|_{M_{t-1,X}^{-1}}$ at first. This lemma holds for general sequences of vectors, so can be used not only in this particular proof.

\begin{lemma}
\label{lemma.bound_V}
Let $\{\bW_t\}_{t=1}^\infty$ be a sequence in $\mathbb{R}^d$ satisfying $\left\|\bW_t\right\|\leq\sqrt{d}$. Define $\bW_0={\bf 0}$ and $M_t=\sum_{i=0}^{t}\bW_i\bW_i^\intercal$. Suppose there is an integer $t_1$ such that $\lambda_{\min}(M_{t_1+1})\geq 1$, then for all $t_2>0$,\begin{align*}
    \sum_{t=t_1}^{t_1+t_2}\left\|\bW_t\right\|_{M_{t-1}^{-1}}\leq\sqrt{2t_2d\log(t_2d+t_1)}.
\end{align*}
\end{lemma}

\begin{proof}
As a direct application of Lemma 11 in \cite{abbasi2011improved}, we have\begin{align*}
    \sum_{t=t_1+1}^{t_1+t_2}\left\|\bW_t\right\|_{M_{t-1}^{-1}}^2\leq2\log\frac{\det M_{t_1+t_2}}{\det M_{t_1}}
    \leq 2d\log\left(\frac{\text{tr}(M_{t_1})+t_2d^2}{d}\right)-2\log\det M_{t_1}.
\end{align*}

Note that $\text{tr}(M_{t_1})=\sum_{t=1}^{t_1}\text{tr}(\bW_t\bW_t^\intercal)=\sum_{t=1}^{t_1}\left\|\bW_t\right\|^2\leq t_1d$ and that $\det(M_{t_1})=\prod_{i=1}^d\lambda_i\geq\lambda_{\min}^d(M_{t_1})\geq 1$ where $\{\lambda_i\}$ are the eigenvalues of $M_{t_1}$. Applying Cauchy-Schwarz inequality yields\begin{align*}
    \sum_{t=t_1+1}^{t_1+t_2}\left\|\bW_t\right\|_{M_{t-1}^{-1}}\leq\sqrt{t_2\sum_{t=t_1+1}^{t_1+t_2}\left\|\bW_t\right\|_{M_{t-1}^{-1}}^2}
    \leq\sqrt{2t_2d\log(t_2d+t_1)}
\end{align*}
which is exactly what we want.
\end{proof}

Meanwhile, for the sake of clarity, we reclaim Lecu\'{e} and Mendelson's inequality here, which is presented in \cite{nie2021matrix}. 

\begin{lemma}[Lecu\'{e} and Mendelson's Inequality]\label{lemma:LMInequality}
There exists an absolute constant $c > 0$ such that the following statement holds.
Let $\bv_1,\cdots,\bv_k$ be independent copies of a random vector $\bv\in\mathbb{R}^d$. Suppose that $\alpha\geq 0$ and $0<\beta\leq 1$ are two real numbers, such that the small-ball probability\begin{align*}
    \Pr\left\{\left|\bv^\intercal \bz\right|>\alpha^{\frac{1}{2}}\right\}\geq \beta
\end{align*}
holds for any $\bz$ in $\Sphere(d)$. Suppose that\begin{align*}
    k\geq \frac{cd}{\beta^2}.
\end{align*}
Then we have\begin{align*}
    \Pr\left\{\lambda_{\min}\left(\frac{1}{n}\sum_{i=1}^k\bv_i\bv_i^\intercal\right)\leq\frac{\alpha\beta}{2}\right\}\leq\exp\left(-\frac{k\beta^2}{c}\right).
\end{align*}
\end{lemma}

Then with Lemma~\ref{lemma.bound_V} and this powerful inequality, we can prove Theorem \ref{thm.regret_glm} to bound the regret of Algorithms \ref{alg:glm-ucb} and \ref{alg:glm-est}.

\thmregretglm*

\begin{proof}
Let $H_t$ be the history of the first $t$ rounds and $R_t$ be the regret in the $t^{th}$ round. By the definition of BGLM and our Algorithm \ref{alg:glm-ucb}, we can deduce that for any $t\leq T_0$, $R_t\leq 1$. Now we consider the case of $t>T_0$. When $t>T_0$, we have\begin{align}
    \mathbb{E}[R_t|H_{t-1}]=\mathbb{E}[\sigma({\bf S}^{\text{opt}},\btheta^*)-\sigma({\bf S}_t,\btheta^*)|H_{t-1}]
\end{align}
such that the expectation is taken over the randomness of ${\bf S}_t$. Then for $T_0<t\leq T$, we can define $\xi_{t-1,X}$ for $X\in {\bf X}\cup \{Y\}$ as $\xi_{t-1,X}=\{\left|\bv^T(\hat{\btheta}_{t-1,X}-\btheta_X^*)\right|\leq\rho\cdot\left\|\bv\right\|_{M_{t-1,X}^{-1}},\forall \bv\in \mathbb{R}^{|\Pa(X)|}\}$. According to the settings in Algorithm \ref{alg:glm-ucb}, we can deduce that $\lambda_{\min}(M_{t-1,X})\geq\lambda_{\min}(M_{T_0,X})$ and by Lecu\'{e} and Mendelson's inequality \cite{nie2021matrix} (conditions of this inequality satisfied according to Lemma \ref{lemma:init_condition}), we have $\Pr\left\{\lambda_{\min}(M_{T_0,X})<R\right\}\leq\exp(-\frac{T_0\zeta^2}{c})$ where $c$ is a constant. Then we can define $\xi_{t-1}=\wedge_{X\in{\bf X}\cup\{Y\}}\xi_{t-1,X}$ and let $\overline{\xi_{t-1}}$ be its complement so by Lemma \ref{thm.learning_glm} we have $\Pr\left\{\overline{\xi_{t-1}}\right\}\leq \left(3\delta+\exp\left(-\frac{T_0\zeta^2}{c}\right)+3\delta \exp\left(-\frac{T_0\zeta^2}{c}\right)\right)n$.

Because under $\xi_{t-1}$, for any $X\in{\bf X}\cup\{Y\}$ and $\bv\in \mathbb{R}^{|\Pa(X)|}$, we have $\left|\bv^T(\hat{\btheta}_{t-1,X}-\btheta_X^*)\right|\leq\rho\cdot\left\|\bv\right\|_{M_{t-1,X}^{-1}}$. Therefore, by the definition of $\tilde{\btheta}_t$, we have $\sigma({\bf S}_t,\tilde{\btheta}_t)\geq \sigma({\bf S}^{\text{opt}},\btheta^*)$ because $\btheta^*$ is in our confidence ellipsoid. Therefore,\begin{align*}
    \mathbb{E}[R_t]&\leq \Pr\left\{\xi_{t-1}\right\}\cdot\mathbb{E}[\sigma({\bf S}^{\text{opt}},\btheta^*)-\sigma({\bf S}_t,\btheta^*)] + \Pr(\overline{\xi_{t-1}})\\
    &\leq \mathbb{E}[\sigma({\bf S}^{\text{opt}},\btheta^*)-\sigma({\bf S}_t,\btheta^*)]+\left(3\delta+\exp\left(-\frac{T_0\zeta^2}{c}\right)+3\delta \exp\left(-\frac{T_0\zeta^2}{c}\right)\right)n\\
    &\leq \mathbb{E}[\sigma({\bf S}_t,\tilde{\btheta}_t)-\sigma({\bf S}_t,\btheta^*)]+\left(3\delta+\exp\left(-\frac{T_0\zeta^2}{c}\right)+3\delta \exp\left(-\frac{T_0\zeta^2}{c}\right)\right)n.
\end{align*}
Then we need to bound $\sigma({\bf S}_t,\tilde{\btheta}_t)-\sigma({\bf S}_t,\btheta^*)$ carefully. 

Therefore, according to Lemma~\ref{thm.learning_glm} and Lemma~\ref{thm.gom_glm}, we can deduce that\begin{align*}
    \mathbb{E}[R_t]
    &\leq\mathbb{E}\left[\sum_{X\in {\bf X}_{{\bf S}_t,Y}}\left|\bV_{t,X}(\tilde{\btheta}_{t,X}-\btheta^*_X)\right|L_{f_X}^{(1)}\right]
    +\left(3\delta+\exp\left(-\frac{T_0\zeta^2}{c}\right)+3\delta \exp\left(-\frac{T_0\zeta^2}{c}\right)\right)n\\
    &\leq\mathbb{E}\left[\sum_{X\in {\bf X}_{{\bf S}_t,Y}}\left\|\bV_{t,X}\right\|_{M_{t-1,X}^{-1}}
    \left\|\tilde{\btheta}_{t,X}-\btheta^*_X\right\|_{M_{t-1,X}}L_{f_X}^{(1)}\right]
    +\left(3\delta+\exp\left(-\frac{T_0\zeta^2}{c}\right)+3\delta \exp\left(-\frac{T_0\zeta^2}{c}\right)\right)n\\
    &\leq 2\rho\cdot\mathbb{E}\left[\sum_{X\in {\bf X}_{{\bf S}_t,Y}}\left\|\bV_{t,X}\right\|_{M_{t-1,X}^{-1}}L_{f_X}^{(1)}\right]
    +\left(3\delta+\exp\left(-\frac{T_0\zeta^2}{c}\right)+3\delta \exp\left(-\frac{T_0\zeta^2}{c}\right)\right)n.
\end{align*}
The last inequality holds because $\left\|\tilde{\btheta}_{t,X}-\btheta^*_X\right\|_{M_{t-1,X}}\leq \left\|\tilde{\btheta}_{t,X}-\hat{\btheta}_{t-1,X}\right\|_{M_{t-1,X}}+\left\|\hat{\btheta}_{t-1,X}-\btheta^*_X\right\|_{M_{t-1,X}}\leq 2\rho$.

Therefore, the total regret can be bounded as\begin{align*}
    R(T)&\leq 2\rho\cdot\mathbb{E}\left[\sum_{t=T_0+1}^T\sum_{X\in {\bf X}_{{\bf S}_t,Y}}\left\|\bV_{t,X}\right\|_{M_{t-1,X}^{-1}}L_{f_X}^{(1)}\right]
    +\left(6\delta+\exp\left(-\frac{T_0\zeta^2}{c}\right)\right)n(T-T_0)+T_0.
\end{align*}

For convenience, we define $\bW_{t,X}$ as a vector such that if $X\in S_t$, $\bW_{t,X}={\bf 0}^{|\Pa(X)|}$; if $X\not\in S_t$, $\bW_{t,X}=\bV_{t,X}$.
Using Lemma \ref{lemma.bound_V}, we can get the result:\begin{align*}
    R(T)&\leq 2\rho\mathbb{E}\left[\sum_{t=T_0+1}^T\sum_{X\in {\bf X}_{{\bf S}_t,Y}}\left\|\bV_{t,X}\right\|_{M_{t-1,X}^{-1}}L_{f_X}^{(1)}\right]
    +\left(6\delta+\exp\left(-\frac{T_0\zeta^2}{c}\right)\right)n(T-T_0)+T_0\\
    &\leq 2\rho\mathbb{E}\left[\sum_{t=T_0+1}^T\sum_{X\in\bX\cup\{Y\}}\left\|\bW_{t,X}\right\|_{M_{t-1,X}^{-1}}L_{f_X}^{(1)}\right]
    +\left(6\delta+\exp\left(-\frac{T_0\zeta^2}{c}\right)\right)n(T-T_0)+T_0\\
    &\leq 2\rho\cdot \max_{X\in{\bf X}\cup\{Y\}}(L_{f_X}^{(1)})\mathbb{E}\left[\sum_{X\in \bX\cup\{Y\}}
    \sqrt{2(T-T_0)|\Pa(X)|\log\left((T-T_0)|\Pa(X)|+T_0\right)}\right]\\
    &\quad+\left(6\delta+\exp\left(-\frac{T_0\zeta^2}{c}\right)\right)n(T-T_0)+T_0\\
    &=O\left(\frac{1}{\kappa}n\sqrt{TD}L^{(1)}_{\max}\ln T\right)=\tilde{O}\left(\frac{1}{\kappa}n\sqrt{TD}L^{(1)}_{\max}\right)
    \end{align*}
because $\rho=\frac{3}{\kappa}\sqrt{\log(1/\delta)}$.

\end{proof}

\section{Proofs of Lemmas for BLM (Section \ref{sec.hidden})}
\label{app:property_of_linear_model}

\thmlinearproperty*

\begin{proof}

The lemma can be extended to three more detailed equations as below:
\begin{align}
    \mathbb{E}[Y|\doi({\bf S}={\bf s})]
    &=\sum_{X\in {\bf S}} s_X \sum_{P\in \mathcal{P}_{X,Y},P\cap {\bf S}=\{X\}}\prod_{e\in P}\theta^*_e \ + 
    	\sum_{P\in \mathcal{P}_{U_0,Y},P\cap {\bf S}=\emptyset}\prod_{e\in P}\theta^*_e   \label{eq:EYold} \\
    &=\sum_{X\in {\bf S}} s_X \sum_{P\in \mathcal{P}'_{X,Y},P\cap {\bf S}=\{X\}}\prod_{e\in P}\theta^{*'}_e \ +
    	\sum_{P\in \mathcal{P}'_{X_1,Y},P\cap {\bf S}=\emptyset}\prod_{e\in P}\theta^{*'}_e  \label{eq:EYold2new} \\
    &=\mathbb{E}'[Y|\doi({\bf S}={\bf s})], \label{eq:EYnew}
\end{align}
where the notation $P\cap {\bf S}$ means that intersection of the node set of the path $P$ and the node set $\bS$.

In order to prove these three equations, we firstly need to prove that for an arbitrary node $X\in\bU \cup {\bf X}\cup\{Y\}$, we have
\begin{align}
    \label{equation.property_linear}
    \mathbb{E}[X|\doi({\bf S}= {\bf s})]=
    \begin{cases}
		\sum_{Z\in\Pa(X)}\mathbb{E}[Z|\doi({\bf S}={\bf s})]\cdot\theta^*_{Z,X} & X \notin \bS, \\
		s_X  & X \in \bS,
    \end{cases}
\end{align}
where the case of $X\notin \bS$ is because 
\begin{align*}
    \mathbb{E}[X|\doi({\bf S}={\bf s})]&=\Pr\left\{X=1|\doi({\bf S}={\bf s})\right\}\\
    &=\mathbb{E}[\pa(X)\cdot \theta^*_X|\doi({\bf S}={\bf s})]\\
    &=\sum_{Z\in\Pa(X)}\mathbb{E}[Z|\doi({\bf S}={\bf s})]\cdot\theta^*_{Z,X}.
\end{align*}
By recursively using Eq.~\eqref{equation.property_linear} to replace the expectation of a node by the expectations of its parents 
	in the expression of $\mathbb{E}[Y|\doi({\bf S}={\bf s})]$, we can get Eq.~\eqref{eq:EYold}. 
Similarly, the Eq.~\eqref{eq:EYnew} can be obtained.

Finally, the Eq.~\eqref{eq:EYold2new} can be proved by 
\begin{align}
& \sum_{X\in {\bf S}} s_X \sum_{P\in \mathcal{P}_{X,Y},P\cap {\bf S}=\{X\}}\prod_{e\in P}\theta^*_e \ + 
	\sum_{P\in \mathcal{P}_{U_0,Y},P\cap {\bf S}=\emptyset}\prod_{e\in P}\theta^*_e  \nonumber\\
&=\sum_{X\in {\bf S}} s_X \sum_{P\in \mathcal{P}'_{X,Y},P\cap {\bf S}=
		\{X\}}\prod_{e\in P}\theta^{*'}_e + 
		\sum_{X\in{\bf X}\cup\{Y\}\setminus {\bf S}}\Pr\left\{X=1|\doi\left({\bf X} \cup \{Y\}\setminus \{X\}={\bf 0}\right)\right\}\sum_{P\in \mathcal{P}'_{X,Y},P\cap {\bf S}=\emptyset}\prod_{e\in P}\theta^{*'}_e \label{eq.transformedtheta}\\
&=\sum_{X\in {\bf S}} s_X \sum_{P\in \mathcal{P}'_{X,Y},P\cap {\bf S}=\{X\}}\prod_{e\in P}\theta^{*'}_e \ +
\sum_{P\in \mathcal{P}'_{X_1,Y},P\cap {\bf S}=\emptyset}\prod_{e\in P}\theta^{*'}_e.  \label{eq.transformedmarkovian}
\end{align}
Now we explain why Eq.~\eqref{eq.transformedtheta} and Eq.~\eqref{eq.transformedmarkovian} hold. 
Firstly, we illustrate why for every $X \in \bS$,
\begin{align}
    \sum_{P\in \mathcal{P}_{X,Y},P\cap {\bf S}=\{X\}}\prod_{e\in P}\theta^*_e=
    \sum_{P\in \mathcal{P}'_{X,Y},P\cap {\bf S}=\{X\}}\prod_{e\in P}\theta^{*'}_e. \label{eq.basicthetatransform}
\end{align}
To show the above equality, we first look at the right-hand side, and notice that for each 
	path $P' \in \mathcal{P}'_{X,Y}$, it corresponds to a collection of paths from $X$ to $Y$ in $G$,
	that is, $\cP_{X,Y}$.
Then each $\theta^{*'}_{P'} = \prod_{e\in P'}\theta^{*'}_e$ is the sum of the $\theta^*_{P}$
	values for some paths $P \in \cP_{X,Y}$ according to our $G'$ construction.
This means that the RHS $\sum_{P'\in \mathcal{P}'_{X,Y},P'\cap {\bf S}=\{X\}}\prod_{e\in P'}\theta^{*'}_e$ is 
	a summation of $\theta^*_{P}$ values for some paths $P \in \cP_{X,Y}$.
Moreover, these paths only intersect with $\bS$ at the starting node $X$, in both $G$ and $G'$.

Next, Suppose $P_0$ is an arbitrary path in $\mathcal{P}_{X,Y}$ such that $P_0\cap \bS=\{X\}$. 
With the above observation, we only need to show that in $\sum_{P\in \mathcal{P}'_{X,Y},P\cap {\bf S}=\{X\}}\prod_{e\in P}\theta^{*'}_e$, 
	$\theta^*_{P_0}=\prod_{e\in P}\theta^*_e$ has been counted exactly once.
Suppose $P_0$ has the form of $X_{i_1}\rightarrow U_{1,1}\rightarrow\cdots\rightarrow U_{1,j_1}\rightarrow X_{i_2}\rightarrow U_{2,1}\rightarrow\cdots\rightarrow U_{2,j_2}\rightarrow\cdots \rightarrow X_{i_k}$, where $X$'s represent observed variables and $U$'s represent unobserved variables. 
Then according to our transformation, $\theta^*_{P_0}$ is contained in the expansion of $\theta^{*'}_{X_{i_1},X_{i_2}}\cdots \theta^{*'}_{X_{i_{k-1}},X_{i_k}}$, 
	namely $(\sum_{P\in \mathcal{P}_{X_{i_1},X_{i_2}}}\theta^*_P)\cdots(\sum_{P\in \mathcal{P}_{X_{i_{k-1}},X_{i_k}}}\theta^*_P)$. 
Moreover, we know that the paths in $\mathcal{P}'_{X,Y}$ only contain observed nodes, so in other terms of the expansion of 
	$\sum_{P\in \mathcal{P}'_{X,Y},P\cap {\bf S}=\{X\}}\prod_{e\in P}\theta^{*'}_e$, we can not find another $\theta^*_{P_0}$. 
Therefore, Eq.~\eqref{eq.basicthetatransform} holds.

Furthermore, according to Eq.~\eqref{eq:EYold}, 
	we have for every $X \in \bX \cup \{Y\}$,
\begin{align*}
	\Pr\left\{X=1|\doi\left({\bf X} \cup \{Y\} \setminus \{X\}={\bf 0}\right)\right\} = \sum_{P\in \mathcal{P}_{U_0,X}, P\cap (\bX\cup\{Y\})=\{X\}}\theta^*_P. 	
\end{align*}
For an arbitrary path $P_0$ in $\mathcal{P}_{U_0,Y}$ and $P_0 \cap \bS = \emptyset$, we claim that it will be added exactly once in 
\begin{align}
\sum_{X\in{\bf X}\cup\{Y\}\setminus {\bf S}}\Pr\left\{X=1|\doi\left({\bf X} \cup \{Y\}\setminus \{X\}={\bf 0}\right)\right\}\sum_{P\in \mathcal{P}'_{X,Y},P\cap {\bf S}=\emptyset}\prod_{e\in P}\theta^{*'}_e.\label{eq.U0pathstransform}
\end{align} 
In fact, suppose $X_i$ is the first observed node in $P_0$.
Then $X_i \in \bX \cup \{Y\} \setminus \bS$, and
	$P_0$ will be added in $\Pr\left\{X_i=1|\doi\left({\bf X} \cup\{Y\}\setminus \{X_i\}=
	{\bf 0}\right)\right\}\sum_{P\in \mathcal{P}'_{X_i,Y},P\cap {\bf S}=\emptyset}\prod_{e\in P}\theta^{*'}_e$ 
	for the similar reason of the proof of Eq.~\eqref{eq.basicthetatransform} by expansions. 
Therefore, up to now , we have proved Eq.~\eqref{eq.transformedtheta}.

Now we consider how Eq.~\eqref{eq.U0pathstransform} equals to $\sum_{P\in \mathcal{P}'_{X_1,Y},P\cap S=\emptyset}\theta^{*'}_P$. 
Actually, for a path $P_0$ in $G'$ from $X_1$ to $Y$, it must have the form $X_1=X_{i_1}\rightarrow X_{i_2}\rightarrow\cdots\rightarrow Y$. 
Therefore, $\theta^{*'}_{P_0}$ is equal to $\theta^{*'}_{X_1,X_{i_2}}\theta^{*'}_{X_{i_2}\rightarrow\cdots Y}$. 
We already know that $\theta^{*'}_{X_1,X_{i_2}}=\Pr\left\{X_{i_2}=1|\doi\left({\bf X} \cup \{Y\}\setminus \{X_{i_2}\}={\bf 0}\right)\right\}$ by the definition of our transformation. Simultaneously, $\theta^{*'}_{X_{i_2}\rightarrow\cdots Y}$ is included exactly once in 
	$\sum_{P\in \mathcal{P}'_{X_{i_2},Y},P\cap {\bf S}=\emptyset}\prod_{e\in P}\theta^{*'}_e$, so we have Eq.~\eqref{eq.transformedmarkovian} proved.
\end{proof}

\lemmahidden*

\begin{proof}
In order to prove this, we firstly prove that 
\begin{align}
\mathbb{E}\left[X|\doi\left(\Pa'(X)\setminus\{X_1\}=\pa'(X)\setminus\{x_1\}, \bS=\bs\right)\right]=\mathbb{E}'\left[X|\doi\left(\Pa'(X)=\pa'(X), \bS=\bs\right)\right],\label{eq.middledoparents}
\end{align}
In fact, this can be seen as a direct application of Lemma~\ref{thm:property_of_linear_model} if we replace the $Y$ in Lemma~\ref{thm:property_of_linear_model} by the node $X$ here.

Next, we want to apply the second rule of $\doi$-calculus \cite{Pearl09,pearl2012calculus} to show the following:
\begin{align}
    &\mathbb{E}'\left[X|\doi\left(\Pa'(X)=\pa'(X),\bS=\bs\right)\right] \nonumber \\
    &=\Pr{'}\left\{X=1|\doi\left(\Pa'(X)=\pa'(X),\bS=\bs\right)\right\} \nonumber \\
    &=\Pr{'}\left\{X=1|\Pa'(X)=\pa'(X),\doi(\bS=\bs)\right\}, \label{eq:primegraphdocal}
\end{align}
where the second equality applies the $\doi$-calculus rule.
According to the rule, the equality holds when $X$ and $\Pa'(X)$ is independent in the causal graph $G'$ after removing all the outgoing edges of $\Pa(X)$ and all the incoming
	edges of $\bS$.
Since $G'$ is Markovian with only observed nodes, after removing all outgoing edges of the parent nodes of $X$, it is certainly true that $X$ is independent of $\Pa(X)$.

Finally, we want to show the following, again by the the second rule of $\doi$-calculus:
\begin{align}
&\mathbb{E}\left[X|\doi\left(\Pa'(X) \setminus \{X_1\}=\pa'(X) \setminus \{x_1\},\bS=\bs\right)\right] \nonumber \\
&=\Pr\left\{X=1|\doi\left(\Pa'(X) \setminus \{X_1\}=\pa'(X) \setminus \{x_1\},\bS=\bs\right)\right\}\nonumber \\
&=\Pr\left\{X=1|\Pa'(X) \setminus \{X_1\}=\pa'(X) \setminus \{x_1\},\doi(\bS=\bs)\right\}. \label{eq:originalgraphdocal}
\end{align}
Note that the above is on the original graph $G$.
To show the above equality according to the second rule of $\doi$-calculus, we need to show that $X$ and $\Pa'(X) \setminus \{X_1\}$ are independent
	in the graph $G$ after removing the outgoing edges of $\Pa'(X) \setminus \{X_1\}$ and the incoming edges of $\bS$.
Call this trimmed graph $\tilde{G}$.
According to the properties of the causal diagram~\cite{Pearl09}, if $X$ and $\Pa'(X) \setminus \{X_1\}$ are not independent in $\tilde{G}$,
	there must exist an active path $P$ between $X$ and one of the nodes $Z\in \Pa'(X) \setminus \{X_1\}$.
Since $Z$'s outgoing edges are cut off, the edge connecting to $Z$ in path $P$ must be pointing towards $Z$.
On the path $P$, it cannot be that all the edges pointing in the direction from $X$ to $Z$, since this would violate the DAG assumption of $G$.
Then, there is a first node $W$ on the path from $Z$ to $X$ such that from $W$ to $Z$ all the edges pointing in the direction from $W$ to $Z$, and on the segment
	from $W$ to $X$, the edge connecting to $W$ is also leaving $W$ and pointing towards $X$, i.e. $W$ is a fork (Definition 1.2.3 of \cite{Pearl09}).
If on the segment from $W$ to $X$, there is a node $V$ with two incoming edges pointing to $V$ on the path, this means $V$ is a collider \cite{Pearl09}.
Since $V$ cannot be in $\bS$, otherwise, the incoming edges of $V$ would have been trimmed in $\tilde{G}$, then this collider $V$ would block the path making $P$ not
	an active path.
Therefore, no such collider exists on the path from $W$ to $X$, and all edges on $W$ to $X$ point in the direction from $W$ to $X$.
If there is an observed variable $X_i$ on the path from $W$ to $X$ other than $X$ itself, then find the $X_i$ that is closest to $X$ on $P$.
This means that the segment of $P$ from $X_i$ to $X$ is a hidden path, which implies that $X_i$ would become a parent of $X$ in $G'$, i.e. $X_i \in \Pa'(X)\setminus \{X_1\}$.
But this implies that the outgoing edges of $X_i$ should have been cut off, a contradiction.
Therefore all nodes on the path segment from $W$ to $X$ except $X$ are hidden variables.
Then let $U$ be the hidden variable $W$, and on the path segment from $U$ to $Z$, let $X_i$ be the first observed variable.
We thus find an unobserved variable $U$ connecting to both $X_i$ and $X$ with paths consisting of only hidden variables, except $X_i$ and $X$, and $X$ is
	a descendant of $X_i$.
But in Section~\ref{sec.hidden} we already state that we exclude such situation in graph $G$.
Hence, no active path can exist between $Z$ and $W$, and thus Eq.~\eqref{eq:originalgraphdocal} holds.

The lemma is proved with Eqs.~\eqref{eq.middledoparents}, \eqref{eq:primegraphdocal} and \eqref{eq:originalgraphdocal}.
\end{proof}

\section{Proof of Regret Bound for Algorithm~\ref{alg:linear-lr} (Theorem~\ref{thm.blmlr})}
\label{app.blmlr}

We rewrite Lemma 1 in \cite{li2020online} as Lemma~\ref{lemma.learning_glm} as below.
\begin{lemma}[Lemma 1 in \cite{li2020online}]
\label{lemma.learning_glm}
Given $\{\bV_t,X^t\}_{t=1}^\infty$ with $\bV_t\in\{0,1\}^{d}$ and $X^t\in\{0,1\}$ as a Bernoulli random variable with $\mathbb{E}[X^t|\bV_1,X^1,\cdots,\bV_{t-1},X^{t-1},\bV_t]=\bV_t^\intercal\btheta$ where $\btheta\in[0,1]^d$, let $M_t=\mathbf{I}+\sum_{i=1}^t\bV_i\bV_{i}^\intercal$ and $\hat{\btheta}_{t}=M_t^{-1}(\sum_{i=1}^t\bV_iX^i)$ be the linear regression estimator. Then with probability at least $1-\delta$, for all $t\geq 1$, it holds that $\btheta$ lies in the confidence region\begin{small}\begin{align*}
    \left\{\btheta'\in[0,1]^d:\left\|\btheta'-\hat{\btheta}_t\right\|_{M_t}\leq\sqrt{d\log(1+td)+2\log\frac{1}{\delta}}+\sqrt{d}\right\}.
\end{align*}\end{small}
\end{lemma}

With Lemma~\ref{lemma.learning_glm}, we are able to prove the regret bound of Algorithm~\ref{alg:linear-lr}.

\thmlinearlr*

\begin{proof}
In Section~\ref{sec.hidden}, we have already shown Lemmas~\ref{lemma.hidden} and \ref{thm:property_of_linear_model}, therefore, we can deduce that by seeing $G$ as the Markovian graph $G'$, Lemma 1 in \cite{li2020online} and Theorem~\ref{thm.gom_glm} still holds for $\btheta^{*'}$.

With probability at most $n\delta$, event $\left\{\exists t\leq T,x\in\bX\cup\{Y\}:\left\|\btheta^{*'}_X-\hat{\btheta}_{t,X}\right\|>\rho_t\right\}$ occurs. Next we bound the regret conditioned on the absence of this event. In detail, according to Theorem 1 in \cite{li2020online} and Theorem \ref{thm.gom_glm}, we can deduce that \begin{align*}
    \mathbb{E}\left[R_t\right]&=\mathbb{E}\left[\sigma'(\bS^{\text{opt}},\btheta^{*'})-\sigma'(\bS_t,\btheta^{*'})\right]\\
    &\leq \mathbb{E}\left[\sigma'(\bS_t,\tilde{\btheta_{t}})-\sigma'(\bS_t,\btheta^{*'})\right]\\
    &\leq \mathbb{E}\left[\sum_{X\in \bX_{\bS_t,Y}}\left|\bV_{t,X}^\intercal(\tilde{\btheta}_{t,X}-\btheta^{*'}_{X})\right|\right]\\
    &\leq \mathbb{E}\left[\sum_{X\in \bX_{\bS_t,Y}}\left\|\bV_{t,X}\right\|_{M_{t-1,X}^{-1}}\left\|\tilde{\btheta}_{t,X}-\btheta^{*'}_{X}\right\|_{M_{t-1,X}}\right]\\
    &\leq \mathbb{E}\left[\sum_{X\in \bX_{\bS_t,Y}}2\rho_{t-1}\left\|\bV_{t,X}\right\|_{M_{t-1,X}^{-1}}\right],
\end{align*}
since $\tilde{\btheta}_{t,X},\btheta^*_X$ are both in the confidence set. Thus, we have\begin{align*}
    R(T)=\mathbb{E}\left[\sum_{t=1}^TR_t\right]&\leq 2\rho_{T}\cdot\mathbb{E}\left[\sum_{t=1}^T\sum_{X\in\bX_{\bS_t,Y}}\left\|\bV_{t,X}\right\|_{M_{t-1,X}^{-1}}\right].
\end{align*}

For convenience, we define $\bW_{t,X}$ as a vector such that if $X\in S_t$, $\bW_{t,X}={\bf 0}^{|\Pa(X)|}$; if $X\not\in S_t$, $\bW_{t,X}=\bV_{t,X}$. According to Cauchy-Schwarz inequality, we have\begin{align*}
    R(T)&\leq 2\rho_T\cdot\mathbb{E}\left[\sum_{t=1}^T\sum_{X\in\bX\cup\{Y\}}\left\|\bW_{t,X}\right\|_{M_{t-1,X}^{-1}}\right]\\
    &\leq 2\rho_T\cdot\mathbb{E}\left[\sqrt{T}\cdot\sum_{X\in \bX\cup\{Y\}}\sqrt{\sum_{t=1}^T\left\|\bW_{t,X}\right\|^2_{M_{t-1,X}^{-1}}}\right].
\end{align*}

Note that $M_{t,X}=M_{t-1,X}+\bW_{t,X}\bW_{t,X}^\intercal$ and therefore, $\det\left(M_{t,X}\right)=\det(M_{t-1,X})\left(1+\left\|\bW_{t,X}\right\|^2_{M_{t-1,X}^{-1}}\right)$, we have\begin{align*}
    \sum_{t=1}^T \left\|\bW_{t,X}\right\|^2_{M_{t-1,X}^{-1}}&\leq\sum_{t=1}^T\frac{n}{\log(n+1)}\cdot\log\left(1+\left\|\bW_{t,X}\right\|^2_{M_{t-1,X}^{-1}}\right)\\
    &\leq \frac{n}{\log(n+1)}\cdot\log\frac{\det(M_{T,X})}{\det({\bf I})}\\
    &\leq \frac{n|\Pa(X)|}{\log(n+1)}\cdot\log\frac{\text{tr}(M_{T,X})}{|\Pa(X)|}\\
    &\leq \frac{n|\Pa(X)|}{\log(n+1)}\cdot\log\left(1+\sum_{t=1}^T\frac{\left\|\bW_{t,X}\right\|^2_2}{|\Pa(X)|}\right)\\
    &\leq \frac{nD}{\log(n+1)}\log(1+T).
\end{align*}
Therefore, the final regret $R(T)$ is bounded by\begin{align*}
    R(T)&\leq 2\rho_Tn\sqrt{T\frac{nD}{\log(n+1)}\log(1+T)}\\
    &=O\left(n^2\sqrt{DT}\log T\right),
\end{align*}
because $\rho_T=\sqrt{n\log(1+Tn)+2\log\frac{1}{\delta}}+\sqrt{n}$. When $\left\{\exists t\leq T,x\in\bX\cup\{Y\}:\left\|\btheta^{*'}_X-\hat{\btheta}_{t,X}\right\|>\rho_t\right\}$ does occur, the regret is no more than $T$. Therefore, the total regret is still $O\left(n^2\sqrt{DT}\log T\right)$.
\end{proof}

\section{Extensions to Causal Model With Continuous Variables}
\label{app.exttocontinuous}

We consider linear models with continuous variables. In linear models, the propagation is defined as $X=\btheta^*_{X}\cdot\pa(X)+\varepsilon_X$ instead of $P(X=1|\Pa(X)=\pa(X))=\btheta^*_{X}\cdot\pa(X)+\varepsilon_X$. We still need $\varepsilon_X$ to be a zero-mean sub-Gaussian noise that ensures that $X\in[0,1]$. According to Theorem 2 in \cite{abbasi2011improved}, Lemma~\ref{lemma.learning_glm} still holds for continuous $\bV_t$'s and $X^t$'s. Formally, we have Lemma~\ref{lemma.continuous_learning_glm}.
\begin{lemma}
\label{lemma.continuous_learning_glm}
Given $\{\bV_t,X^t\}_{t=1}^\infty$ with $\bV_t\in[0,1]^{d}$ and $X^t\in[0,1]$ as a Bernoulli random variable with $\mathbb{E}[X^t|\bV_1,X^1,\cdots,\bV_{t-1},X^{t-1},\bV_t]=\bV_t^\intercal\btheta$ where $\btheta\in[0,1]^d$, let $M_t=\mathbf{I}+\sum_{i=1}^t\bV_i\bV_{i}^\intercal$ and $\hat{\btheta}_{t}=M_t^{-1}(\sum_{i=1}^t\bV_iX^i)$ be the linear regression estimator. Then with probability at least $1-\delta$, for all $t\geq 1$, it holds that $\btheta$ lies in the confidence region\begin{small}\begin{align*}
    \left\{\btheta'\in[0,1]^d:\left\|\btheta'-\hat{\btheta}_t\right\|_{M_t}\leq\sqrt{d\log(1+td)+2\log\frac{1}{\delta}}+\sqrt{d}\right\}.
\end{align*}\end{small}
\end{lemma}
Simultaneously, Lemma~\ref{lemma.hidden} and Lemma~\ref{thm:property_of_linear_model} also hold for linear models without any modification on their proofs. Therefore, we can still use BLM-LR (Algorithm~\ref{alg:linear-lr}) on continuous linear models and get the same regret guarantee. Formally, we have the following theorem.
\begin{theorem}[Regret Bound of Algorithm~\ref{alg:linear-lr} on Linear Models]{theorem}
\label{thm.lmlr}
The regret of BLM-LR (Algorithm~\ref{alg:linear-lr}) running on linear model with hidden variables is bounded as
\begin{align*}
    R(T)=O\left(n^2\sqrt{DT}\log T\right).
\end{align*}
\end{theorem}
The proof of Theorem~\ref{thm.lmlr} is completely the same as the proof of Theorem~\ref{thm.blmlr}, which is given in Appendix~\ref{app.blmlr}. It is worthy noting that the GOM bounded smoothness condition still holds for linear models. This is because the expectation of reward node $Y$ in a linear model is the same comparing to the BLM with same parameters and skeleton.

\section{Hardness of Offline CCB Problems}
\label{sec.hardness}

In order to illustrate the reasonableness of our offline oracles which are used in the online algorithms, we show the hardness of CCB problems for BGLMs here. As a preparation, we already know that Maximum $k$-Vertex Cover (Max $k$-VC) problem is a NP-hard problem because it contains the vertex covering problem as a sub-problem, which is in Karp’s original list of $21$ NP-complete problems \cite{karp1972reducibility}.  In the following theorem, we prove that the offline version of BGLM CCB problem is NP hard by reducing Max $k$-VC problem to this problem.

\begin{theorem}[NP Hardness of Offline BGLM CCB Problem]
	\label{thm:np_glm}
	With Assumptions~\ref{asm:glm_3} and \ref{asm:glm_2} holds for the BGLM $G$, the offline version of BGLM CCB problem is NP hard. 
\end{theorem}

\begin{proof}
We consider an arbitrary instance of the set covering problem $(\mathcal{S},\bX',k)$. Suppose all the elements are $X'_1,\cdots,X'_l$, all the sets are $\bS_1,\cdots,\bS_r$ and the final target is to find $k$ sets such that the number of covered elements is maximized. Then we can create a two bipartite graph $G=({\bf V},E)$ such that $X_1,\cdots,X_l$ are on the right side and $Z_1,\cdots,Z_r$ are on the left side. Although the model is Markovian, the endogenesis distribution for all of the nodes is set as a zero distribution, i.e. without interventions, all the nodes will always be $0$. For every $X_i$, the function that decides it is set by $f_{X_i}(\pa(X_i)\cdot \btheta^*_{X_i})=1-\frac{1}{\alpha pa(X_i)\cdot \btheta^*_{X_i}+1}$ and the $j^{th}$ entry of $\btheta^*_{X_i}$ is $1$ if and only if $X_i\in \bS_j$ (otherwise, the entry is set by $0$). This function obviously satisfies the two assumptions when $\kappa\leq\frac{\alpha}{(\alpha n+1)^2}$, $L_{f_{X_i}}^{(1)}>=\alpha$ and $L_{f_{X_i}}^{(2)} >= 0$. For the reward node $Y$, we define $f_Y((x_1,\cdots,x_l)\cdot \btheta^*_Y)=\frac{1}{l}(x_1+\cdots+x_l)$, which means that $\btheta^*_Y={\bf 1}^{l}$. Until now, we get an instance $(G,k)$ for the offline BGLM CCB problem.

If the optimal solution of the BGLM CCB problem instance is $X_{i_1},\cdots,X_{i_t}$ and $Z_{j_1},\cdots,Z_{j_{k-t}}$. We choose the sets corresponding to this nodes in $(\mathcal{S},\bX',k)$ and replace $X_{i_1},\cdots,X_{i_t}$ by the corresponding $t$ sets that contain $X'_{i_1},\cdots,X'_{i_t}$ in the set covering instance. Suppose the number of covered elements of this strategy is $K$. Then the expected reward $\mathbb{E}[Y]$ when do interventions on $X_{i_1},\cdots,X_{i_t}$ and $Z_{j_1},\cdots,Z_{j_{k-t}}$ is at most $\frac{K}{l}$. Assume that $K$ is not the optimal solution for the set covering instance $(\mathcal{S},\bX',k)$, then there exist $k$ sets such that the number of covered elements is at least $K+1$. Doing interventions to $1$ on corresponding $Z_j$'s of these sets in the BGLM CCB problem instance $(G,k)$, the expected reward is at least $\frac{1}{l}(K+1)(1-\frac{1}{\alpha+1})$. Therefore, if $\frac{1}{l}(K+1)(1-\frac{1}{\alpha+1})>\frac{K}{l}$, we get a contradiction, which can be realized by setting $\alpha=n+1>K$. Therefore, we can use an oracle of the offline BGLM CCB problem to solve the set covering problem. 
\end{proof}

Finally, we prove that BGLM CCB problem has monotonicity. Formally, we propose Theorem~\ref{thm.monoofbglm}.
\begin{theorem}\label{thm.monoofbglm}
In the offline version of BGLM CCB problem, the expected reward $\mathbb{E}[Y]$ is 
monotonous with respect to the interventions. 
\end{theorem}

\begin{proof}
Because each $f_X$ is monotonous, we can deduce that BGLM is in the family of general cascade model \cite{wu2018general}. Therefore, we can create live-edge graphs.

For an arbitrary live-edge graph, if $\bS$ are intervened to $1$ and $Y$ can be reached from one of these nodes through a directed active path, then $Y$ can be obviously reached from nodes in $\bS\cup\{X\}$ through the same path. Therefore, $\mathbb{E}[Y|\doi(\bS)={\bf 1}^{|\bS|}]\leq \mathbb{E}[Y|\doi(\bS)={\bf 1}^{|\bS|},\doi(X)=1]$. Moreover, if we intervene one node to $0$, it is equivalent to cut off all the paths passing by that node to $Y$, so the value of $Y$ can only decrease. Until now, we have proved the monotonicity of BGLM CCB problem.
\end{proof}

\section{Simulations}
\label{app.simulation}

\subsection{Efficient Pair-Oracle for BLMs}

Before the introduction of our settings, we propose a computational efficient algorithm to replace the pair-oracle (line~\ref{alg:OFUargmax} of Algorithm~\ref{alg:glm-ucb}) for binary linear models without unobserved nodes. For convenience, suppose $X_1,X_2,\cdots,X_{n-1},Y$ is a topological order. The implementation is of $O(\binom{n}{K}D^3n)$ time complexity. See Algorithm~\ref{alg:pairoracle} for details.

\begin{algorithm}[t]
\caption{An Implementation of Pair-Oracle for BLMs}
\label{alg:pairoracle}
\begin{algorithmic}[1]
\STATE {\bfseries Input:}
Graph $G=({\bX}\cup\{Y\},E)$, intervention budget $K\in\mathbb{N}$, estimated parameters $\hat{\btheta}_{t-1}$, current $\rho_{t-1}$, observation matrices $M_{t-1,X}$ for $X\in\bX\cup\{Y\}$.
\STATE {\bfseries Output:}
Intervened set in the $t^{th}$ round $\bS_t$.
\FOR{$\bS\subseteq\bX$ such that $|\bS|=K$}
\FOR{$X=X_2,X_3,\cdots,X_{n-1},Y$}
\STATE $\mathbb{E}_{\widetilde{\btheta}_t}[\Pa(X)|\doi(\bS)]=\left(\mathbb{E}_{\widetilde{\btheta}_t}[X_{i_1}|\doi(\bS)],\cdots,\mathbb{E}_{\widetilde{\btheta}_t}[X_{i_{|\Pa(X)|}}|\doi(\bS)]\right)^\intercal$ where $\Pa(X)=\left\{X_{i_1},\cdots,X_{i_{|\Pa(X)|}}\right\}$.
\IF{$X\not\in \bS$}
\STATE $\mathbb{E}_{\widetilde{\btheta}_{t}}[X|\doi(\bS)]=\rho_{t-1}\sqrt{\left(\mathbb{E}_{\widetilde{\btheta}_t}[\Pa(X)|\doi(\bS)]\right)^\intercal M_{t-1,X}^{-1}\mathbb{E}_{\widetilde{\btheta}_t}[\Pa(X)|\doi(\bS)]}+\left(\mathbb{E}_{\widetilde{\btheta}_t}[\Pa(X)|\doi(\bS)]\right)^\intercal\hat{\btheta}_{t-1,X}$.
\ELSE
\STATE $\mathbb{E}_{\widetilde{\btheta}_{t}}[X|\doi(\bS)]=1$.
\ENDIF
\ENDFOR
\ENDFOR
\RETURN $\bS_t=\argmax_{\bS}\mathbb{E}_{\widetilde{\btheta}_{t}}[Y|\doi(\bS)]$. 
\end{algorithmic}
\end{algorithm}

In order to apply Lagrange multiplier method \cite{de2000mathematical}, we remove the $[0,1]$-boundaries of $\btheta'_{t,X}$ in the definition of ellipsoid $\mathcal{C}_{t,X}$ defined in Algorithm~\ref{alg:linear-lr} in this section. This does not have impact on our regret analysis and in this section, we adopt $\mathcal{C}_{t,X}=\left\{\btheta'_X\in\mathbb{R}^{|\Pa(X)|}:\left\|\btheta'_X-\hat{\btheta}_{t-1,X}\right\|_{M_{t-1,X}}\leq\rho_{t-1}\right\}$. Then we have the following theorem to make sure correctness of Algorithm~\ref{alg:pairoracle}.

\begin{theorem}[Correctness of Algorithm~\ref{alg:pairoracle}]
The intervened set $S_t$ we get in Algorithm~\ref{alg:pairoracle} is exactly what we get from $\argmax_{\bS\subseteq \bX, |\bS|\le K, \btheta'_{t,X} \in \mathcal{C}_{t,X} } \E[Y| \doi({\bS})]$ for Algorithm~\ref{alg:linear-lr}. Moreover, for any $i=1,2,\cdots,n$ and $\bS\subseteq\bX$, we have\begin{align}
    \max_{\btheta_{t,X_j}'\in\mathcal{C}_{t,X_j},j\leq i}\mathbb{E}[X_i|\doi(\bS)]=\mathbb{E}_{\widetilde{\btheta}_{t}}[X_i|\doi(\bS)]\label{eq.argmaxtheta} 
\end{align} for $\widetilde{\btheta}_t$ corresponding to each $\bS$ in the first loop of Algorithm~\ref{alg:pairoracle}. Hence, when $\bS$ equals the selected seed node set $\bS_t$, the corresponding $\widetilde{\btheta}_t$ in the loop of Algorithm~\ref{alg:pairoracle} is a feasible solution of $\tilde{\btheta}_t$ for Algorithm~\ref{alg:linear-lr}.
\end{theorem}

\begin{proof}
Initially, we prove that $\mathbb{E}_{\widetilde{\btheta}_{t}}[X_k|\doi(\bS)]=\max_{\btheta'_{t,X_k}\in\mathcal{C}_{t,X_k}}\mathbb{E}_{\widetilde{\btheta}_{t,X_i},i<k}[X_k|\doi(\bS)]$. Here, the subscript of $\mathbb{E}$ on the right-hand side means that $\widetilde{\btheta}_{t,X_i}$ are already fixed for $i<k$. This part is similar to Appendix B.5 in \cite{li2020online}. In order to determine $\argmax_{\btheta'_{t,X_k}\in\mathcal{C}_{t,X_k}}\mathbb{E}_{\widetilde{\btheta}_{t,X_i},i<k}[X_k|\doi(\bS)]$, we use the method of Lagrange multipliers to solve this optimization problem. The only constraint on $\btheta'_{t,X_k}$ is $\left\|\btheta'_{t,X_k}-\hat{\btheta}_{t-1,X_k}\right\|_{M_{t-1,X_k}}\leq\rho_{t-1}$, therefore, the optimized $\btheta'_{t,X_k}$ we want should be a solution of\begin{align*}
   \frac{\partial \mathbb{E}_{\widetilde{\btheta}_{t,X_i},i\leq k-1, \btheta'_{i,X_k}}[X_k|\doi(\bS)]}{\partial \btheta'_{t,X_k}}-\lambda \frac{\partial \left(\left\|\btheta'_{t,X_k}-\hat{\btheta}_{t-1,X_k}\right\|_{M_{t-1,X_k}}^2-\rho_{t-1}^2\right)}{\partial \btheta'_{t,X_k}}=0, 
\end{align*}
which indicates that\begin{align*}
    \left(\mathbb{E}_{\widetilde{\btheta}_{t,X_i},i\leq k-1}[X_{i_1}|\doi(\bS)],\cdots,\mathbb{E}_{\widetilde{\btheta}_{t,X_i},i\leq k-1}[X_{i_{|\Pa(X_k)|}}|\doi(\bS)]\right)^\intercal=2\lambda M_{t-1,X_k}\left(\btheta'_{t,X_k}-\hat{\btheta}_{t-1,X_k}\right),
\end{align*}
where $X_{i_1},\cdots,X_{i_{|\Pa(X_k)|}}$ are parents of $X_k$. Therefore, we can deduce that \begin{align*}
    \btheta'_{t,X_k}=\frac{1}{2\lambda}M_{t,X_k}^{-1}\left(\mathbb{E}_{\widetilde{\btheta}_{t,X_i},i\leq k-1}\left[X_{i_1}\middle|\doi(\bS)\right],\cdots,\mathbb{E}_{\widetilde{\btheta}_{t,X_i},i\leq k-1}[X_{i_{|\Pa(X_k)|}}|\doi(\bS)]\right)^\intercal+\hat{\btheta}_{t-1,X_k}.
\end{align*}
Meanwhile, we know that when $\btheta'_{t,X_k}$ is optimized, it should be on the boundary of the confidence ellipsoid, so $\lambda$ can be solved out as \begin{align*}
    \frac{1}{2\lambda}=\frac{\rho_{t-1}}{\sqrt{\left(\mathbb{E}_{\widetilde{\theta}_{t,X_i},i\leq k-1}[\Pa(X_k)|\doi(\bS)]\right)^\intercal M_{t-1,X_k}^{-1}\mathbb{E}_{\widetilde{\theta}_{t,X_i},i\leq k-1}[\Pa(X_k)|\doi(\bS)]}}.
\end{align*}
Until now, we have shown that\begin{align}
    \max_{\btheta'_{t,X_k}\in\mathcal{C}_{t,X_k}}\mathbb{E}_{\widetilde{\btheta}_{t,X_i},i<k}[X_k|\doi(\bS)]&=\frac{\rho_{t-1} M_{t-1,X_k}^{-1}\mathbb{E}_{\widetilde{\theta}_{t,X_i},i\leq k-1}[\Pa(X_k)|\doi(\bS)]}{\sqrt{\left(\mathbb{E}_{\widetilde{\theta}_{t,X_i},i\leq k-1}[\Pa(X_k)|\doi(\bS)]\right)^\intercal M_{t-1,X_k}^{-1}\mathbb{E}_{\widetilde{\theta}_{t,X_i},i\leq k-1}[\Pa(X_k)|\doi(\bS)]}}\nonumber\\
    &+\left(\mathbb{E}_{\widetilde{\theta}_{t,X_i},i\leq k-1}[\Pa(X_k)|\doi(\bS)]\right)^\intercal\hat{\btheta}_{t-1,X_k}\nonumber\\
    &=\mathbb{E}_{\widetilde{\btheta}_t}[X_k|\doi(\bS)].\label{eq.maxoftheta}
\end{align}

Next, we prove Eq.~\eqref{eq.argmaxtheta} by induction on $i$. If $i=2$, it is trivial that $\max_{\btheta_{t,X_2}'\in\mathcal{C}_{t,X_2}}\mathbb{E}[X_2|\doi(\bS)]=\mathbb{E}_{\widetilde{\btheta}_{t}}[X_2|\doi(\bS)]$ according to Eq.~\eqref{eq.maxoftheta}. 

Now we suppose that for any $i\leq k-1$, we already know that $\max_{\btheta_{t,X_j}'\in\mathcal{C}_{t,X_j},j\leq i}\mathbb{E}[X_i|\doi(\bS)]=\mathbb{E}_{\widetilde{\btheta}_{t}}[X_i|\doi(\bS)]$. We want to prove that $\max_{\btheta_{t,X_j}'\in\mathcal{C}_{t,X_j},j\leq k}\mathbb{E}[X_k|\doi(\bS)]=\mathbb{E}_{\widetilde{\btheta}_{t}}[X_k|\doi(\bS)]$. Actually, we can deduce that\begin{align*}
    \mathbb{E}_{\btheta'_{t,X_j}\in\mathcal{C}_{t,X_j},j\leq k}[X_k|\doi(\bS)]
    &\leq \max_{\btheta_{t,X_j}'\in\mathcal{C}_{t,X_j},j\leq k}\sum_{Z\in\Pa(X_k)}\mathbb{E}[Z|\doi(\bS)]\theta'_{t,Z,X_k}\\
    &\leq\max_{\btheta_{t,X_k}'\in\mathcal{C}_{t,X_k}}\sum_{Z\in\Pa(X_k)}\left(\max_{\btheta_{t,X_j}'\in\mathcal{C}_{t,X_j},j\leq k-1}\mathbb{E}[Z|\doi(\bS)]\right)\theta'_{t,Z,X_k}\\
    &=\max_{\btheta_{t,X_k}'\in\mathcal{C}_{t,X_k}}\sum_{Z\in\Pa(X_k)}\mathbb{E}_{\widetilde{\btheta}_t}[Z|\doi(\bS)]\theta'_{t,Z,X_k}\\
    &\leq\sum_{Z\in\Pa(X_k)}\widetilde{\theta}_{t,Z,X_k}\mathbb{E}_{\widetilde{\btheta}_t}[Z|\doi(\bS)]\\
    &=\mathbb{E}_{\widetilde{\btheta}_{t}}[X_k|\doi(\bS)],
\end{align*}
which is exactly what we want while $\mathbb{E}_{\widetilde{\btheta}_{t}}[X_k|\doi(\bS)]$ can be easily achieved because $\widetilde{\btheta}_{t,X_k}\in\mathcal{C}_{t,X_k}$ for any $X_k\in\bX\cup\{Y\}$. Therefore, Eq.~\eqref{eq.argmaxtheta} holds and because $\bS_t$ in Algorithm~\ref{alg:pairoracle} is selected according to the largest $\mathbb{E}_{\widetilde{\btheta}_t}[Y|\doi(\bS)]=\max_{\btheta'_{t,X_k}\in\mathcal{C}_{t,X_k},k=2,3,\cdots,n}\mathbb{E}[Y|\doi(\bS)]$, it is exactly the intervened set $\bS_t$ we define in Algorithm~\ref{alg:glm-ucb}.
\end{proof}

When conducting experiments on BLMs, Algorithm~\ref{alg:pairoracle} works well in practice. It runs $30$ times faster than the $\epsilon$-net method proposed in \cite{li2020online} when adopting $\epsilon=0.01$.

\subsection{Simulations on BLMs}

In this section, we run simulations on some typical BLMs with our two algorithms based on maximum likelihood estimation, linear regression (BGLM-OFU and BLM-LR) and two baseline algorithms (UCB and $\epsilon$-greedy). We call BGLM-OFU as BLM-OFU when it is adopted on BLMs. We show that our algorithms have much smaller best-in-hindsight regrets than the baseline algorithms (UCB and $\epsilon$-greedy). This is consistent with our theoretical analysis that BLM-OFU and BLM-LR promise better regrets which are polynomial with respect to the size of graph in combinatorial settings. Also, we further show this by compare the performances of these algorithms when adopted on graphs with different sizes. When the graph becomes larger, the performance gaps between our algorithms and the baselines become larger.

Because our round number is limited, we adopt $\rho_t, \rho$ to be $\frac{1}{10}$ times of our original parameters setting in BLM-OFU and BLM-LR. We use the implementation of pair-oracle introduced in the former section (Algorithm~\ref{alg:pairoracle}). About the initialization phase of BLM-OFU, we set $T_0=\frac{1}{100}T$ for convenience (second order derivative of a linear function is $0$, so $L_{f_X}^{(2)}$ in BGLM-OFU can be arbitrarily small). For the UCB algorithm, we adopt the common used upper confidence bound $\sqrt{\frac{\ln t}{n_{i,t}}}$ where $t$ is the number of current round and $n_{i,t}$ is the number of playing times of arm $i$ until the $t^{th}$ round. For fairness, we also simulate on UCB with the upper confidence bound $10$ times smaller than the standard one. The result of this heuristic algorithm is labeled by ``UCB (scaled)". For $\epsilon$-greedy algorithm, we adopt $\epsilon=0.1$ and $\epsilon=0.01$. We have tried other settings for these two baselines, and our choices are close to optimal for all tested BLMs. For both of the two baselines, we treat each possible $K$ node intervention set as an arm so there are $\binom{n-2}{K}$ arms in total (one can intervene $X_2,\cdots,X_{n-1}$). For each experiment, we run average regrets of $30$ repeated simulations in the same settings. Then we draw the $95\%$ confidence intervals \cite{fisher1992statistical} of these average regrets by repeating the $30$ repeats for $20$ times. In total, we simulate $600$ times for each setting of each experiment. All of our experiments are run on multi-threadedly on $4$ performance-cores of Intel Core\texttrademark~i7-12700H Processor at 4.30GHz with 32GB DDR5 SDRAM. Code is available in our supplementary material.

\subsubsection{Simulations on Parallel Graphs}

In the first experiment, we set round number $T=10000$, $K=3$ and $n=8$. The simulated model $G_1$ is a parallel BLM such that $X_1$ points to $X_2,X_3,\cdots,X_7$ and $X_2,X_3,\cdots,X_7$ all point to $Y$. $X_1$ is always $1$, all the other paramters are\begin{align*}
    &\theta^*_{X_1,X_2}=0.3, \theta^*_{X_1,X_3}=0.4, \theta^*_{X_1,X_4}=0.2, \theta^*_{X_1,X_5}=0.1, \theta^*_{X_1,X_6}=0.6, \theta^*_{X_1,X_7}=0.5,\\
    &\theta^*_{X_2,Y}=0.1, \theta^*_{X_3,Y}=0.3, \theta^*_{X_4,Y}=0.2, \theta^*_{X_5,Y}=0.2, \theta^*_{X_6,Y}=0.1, \theta^*_{X_7,Y}=0.1.
\end{align*} 
The best intervention for $G_1$ is $\{\doi(X_3=1),\doi(X_4=1),\doi(X_5=1)\}$. One can find the graph structure and parameters on $G_1$ in Figure~\ref{Fig.parallel1}. The total running time of this experiment is $3714$ seconds.

\begin{figure}[htb] 
\centering 
\includegraphics[width=0.40\textwidth]{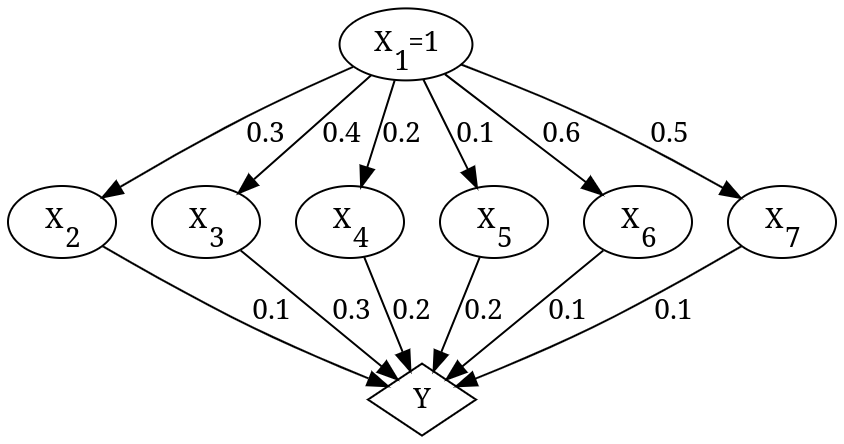}
\caption{Structure and Parameters of $G_1$.}
\label{Fig.parallel1} 
\end{figure}

\begin{figure}[htb] 
\centering 
\includegraphics[width=0.99\textwidth]{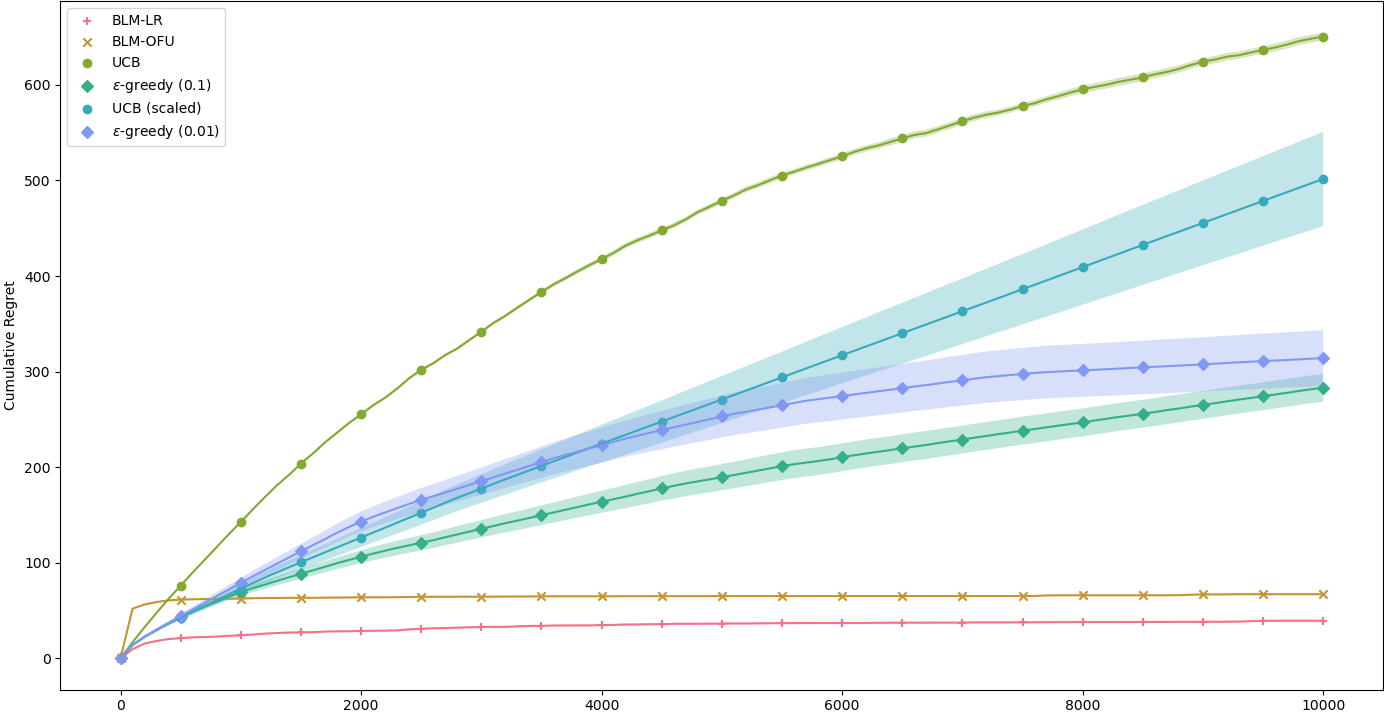}
\caption{Regrets of Algorithms Run on $G_1$.}
\label{Fig.simulation1} 
\end{figure}

Figure~\ref{Fig.simulation1} shows that regrets of BLM-LR and BLM-OFU are much smaller than all settings of UCB and $\epsilon$-greedy algorithms. Moreover, regrets of BLM-LR and BLM-OFU remain nearly unchanged after some rounds. That may be because parameters are already estimated well so deviations from the best intervention are not likely to occur. The large regrets of UCB and $\epsilon$-greedy are due to the large amount of arms ($20$ arms) and the hardness of estimating expected payoffs.

In the second experiment, we set round number $T=2000$ and $K=2$. We run this experiment on three parallel graphs $G_2,G_3$ and $G_4$. $G_2$ has $10$ nodes in total. $X_1$ points to $X_2,X_3,\cdots,X_9$ and $X_2,X_3,\cdots,X_9$ point to $Y$. The parameters on $G_2$ are\begin{align*}
    &\theta^*_{X_1,X_2}=0.2, \theta^*_{X_1,X_3}=0.2, \theta^*_{X_1,X_4}=0.6, \theta^*_{X_1,X_5}=0.6, \theta^*_{X_1,X_6}=0.6, \theta^*_{X_1,X_7}=0.6, \theta^*_{X_1,X_8}=0.6, \theta^*_{X_1,X_9}=0.6,\\
    &\theta^*_{X_2,Y}=0.2, \theta^*_{X_3,Y}=0.2, \theta^*_{X_4,Y}=0.1, \theta^*_{X_5,Y}=0.1, \theta^*_{X_6,Y}=0.1, \theta^*_{X_7,Y}=0.1, \theta^*_{X_8,Y}=0.1, \theta^*_{X_9,Y}=0.1.
\end{align*} 
The best intervention for $G_2$ is $\{\doi(X_2=1),\doi(X_3=1)\}$. One can find the graph structure and parameters on $G_2$ in Figure~\ref{Fig.parallel2}.

\begin{figure}[htb] 
\centering 
\includegraphics[width=0.50\textwidth]{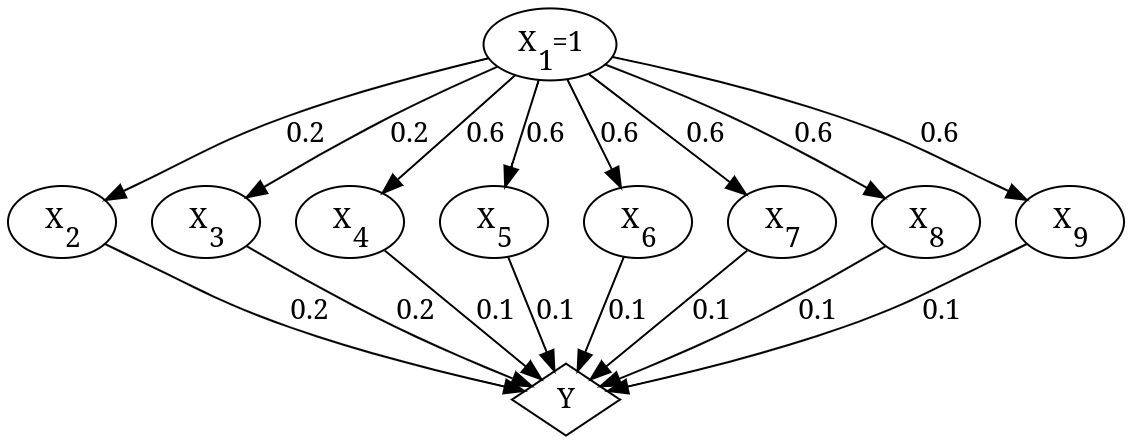}
\caption{Structure and Parameters of $G_2$.}
\label{Fig.parallel2} 
\end{figure}

$G_3$ has $8$ nodes, which is exactly $G_2$ without two nodes $X_8$ and $X_9$. $G_4$ has $6$ nodes, which is exactly $G_2$ without four nodes $X_6,X_7,X_8$ and $X_9$. The best intervention for $G_3$ and $G_4$ is also $\{\doi(X_2=1),\doi(X_3=1)\}$. The total running time of this experiment is $603$ seconds.

\begin{figure}[htb] 
\centering 
\includegraphics[width=0.9\textwidth]{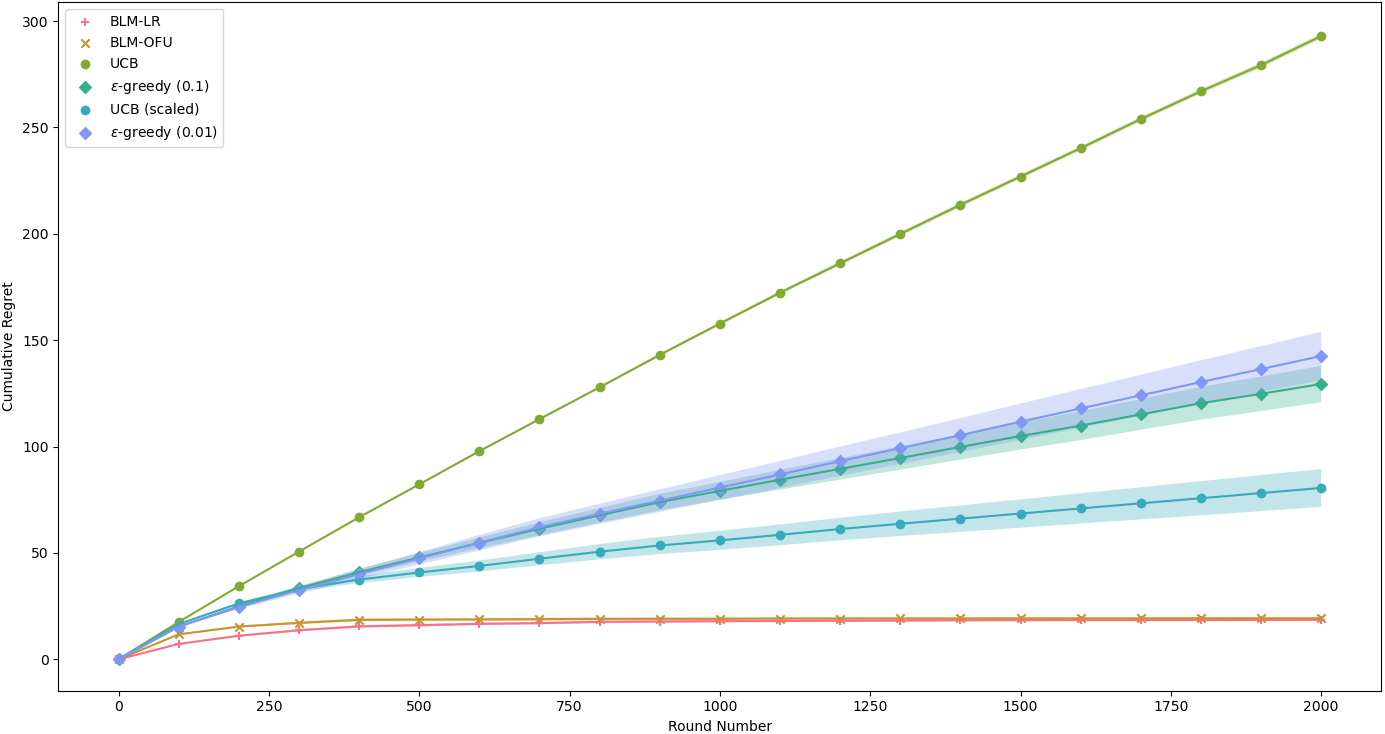}
\caption{Regrets of Algorithms Run on $G_2$.}
\label{Fig.simulation2} 
\end{figure}

\begin{figure}[htb] 
\centering 
\includegraphics[width=0.9\textwidth]{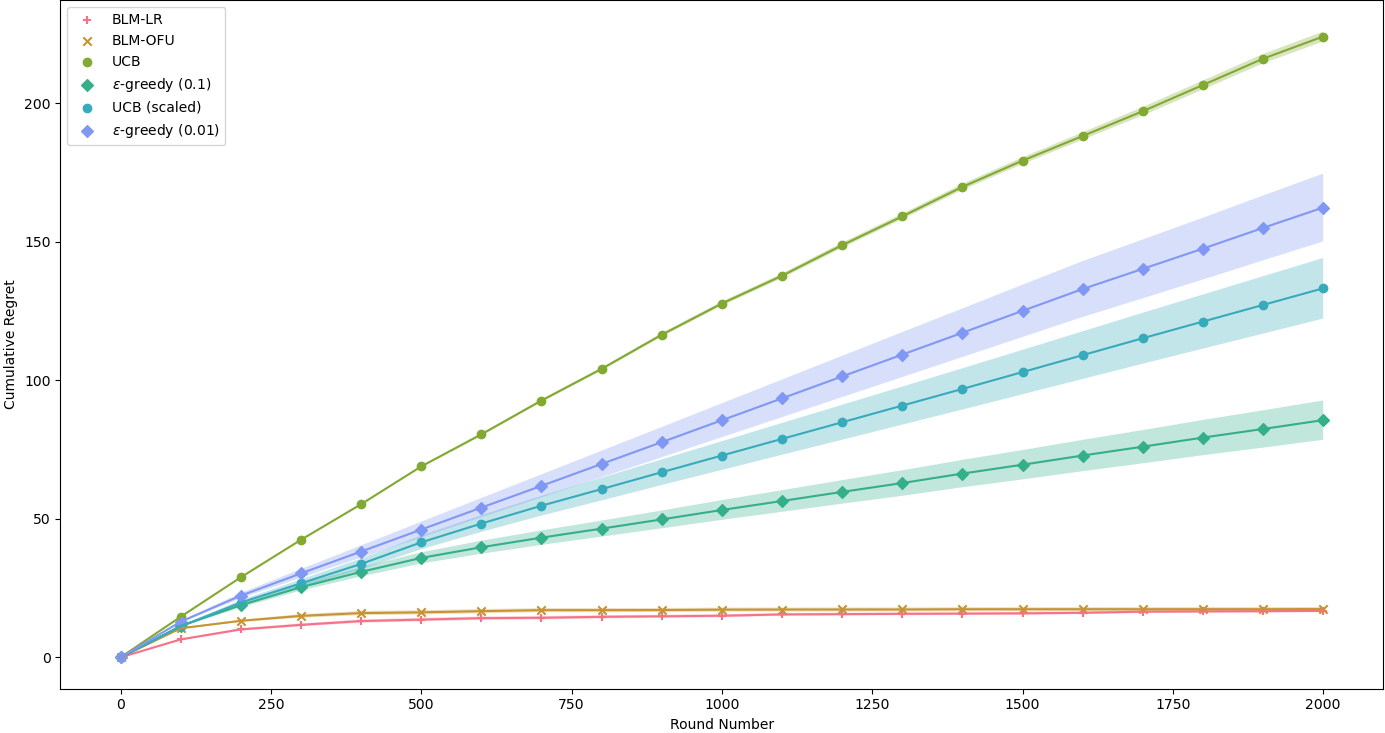}
\caption{Regrets of Algorithms Run on $G_3$.}
\label{Fig.simulation3} 
\end{figure}

\begin{figure}[htb] 
\centering 
\includegraphics[width=0.9\textwidth]{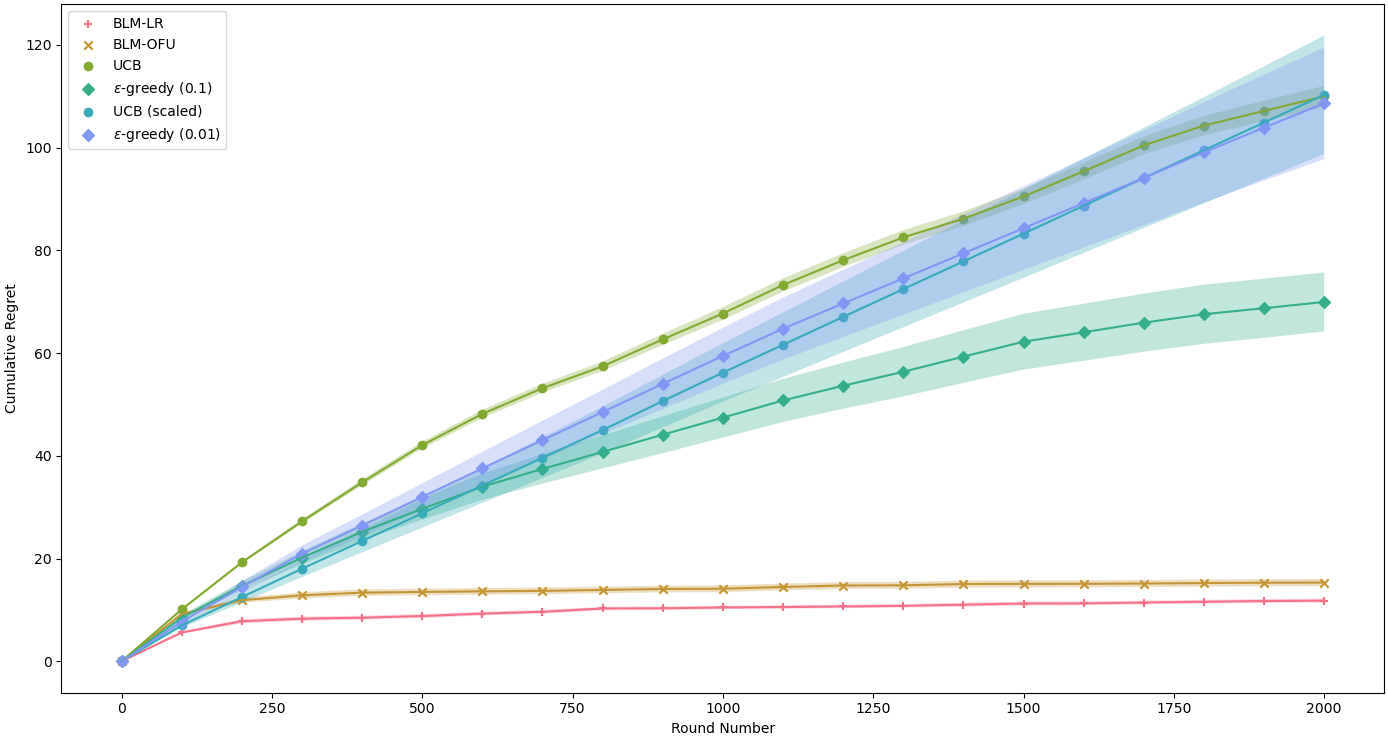}
\caption{Regrets of Algorithms Run on $G_4$.}
\label{Fig.simulation4} 
\end{figure}

Figures~\ref{Fig.simulation2}, \ref{Fig.simulation3} and \ref{Fig.simulation4} also show that regrets of BLM-LR and BLM-OFU are much smaller than all settings of UCB and $\epsilon$-greedy algorithms. Additionally, regrets of BLM-LR and BLM-OFU remain below $20$ on each of the three BLMs $G_2,G_3$ and $G_4$, which are not sensitive to the graph size. However, the regrets of UCB and $\epsilon$-greedy algorithms increase in proportion to the number of arms in each BLM ($G_2,G_3$ and $G_4$ have $6, 15$ and $28$ arms respectively). This comparison shows that BLM-LR and BLM-OFU are able to overcome the exponentially large space of arms in combinatorial settings for causal bandits problem.

\subsubsection{Simulations on Two-Layer Graph}

Besides parallel BLMs, we also conduct an experiment on a two-layer BLM. We set round number $T=10000$, $K=2$ and $n=7$. $G_5$ is a two-layer graph such that $X_1$ points to $X_2,X_3,\cdots,X_6$ and both $X_2,X_3$ point to all of $X_4,X_5,X_6$ and $X_4,X_5,X_6$ all point to $Y$. The parameters on $G_5$ are\begin{align*}
    &\theta^*_{X_1,X_2}=0.1, \theta^*_{X_1,X_3}=0.1, \theta^*_{X_1,X_4}=0.1, \theta^*_{X_1,X_5}=0.1, \theta^*_{X_1,X_6}=0.1,\\
    &\theta^*_{X_2,X_4}=0.1, \theta^*_{X_2,X_5}=0.7, \theta^*_{X_2,X_6}=0.7, \theta^*_{X_3,X_4}=0.2, \theta^*_{X_3,X_5}=0.1, \theta^*_{X_3,X_6}=0.1,\\
    &\theta^*_{X_4,Y}=0.6, \theta^*_{X_5,Y}=0.1, \theta^*_{X_6,Y}=0.1.
\end{align*} 

The best intervention for $G_5$ is $\{\doi(X_2=1),\doi(X_4=1)\}$. One can find the graph structure and parameters on $G_5$ in Figure~\ref{Fig.twolayer}. The total running time of this experiment is $355$ seconds.

\begin{figure}[htb] 
\centering 
\includegraphics[width=0.30\textwidth]{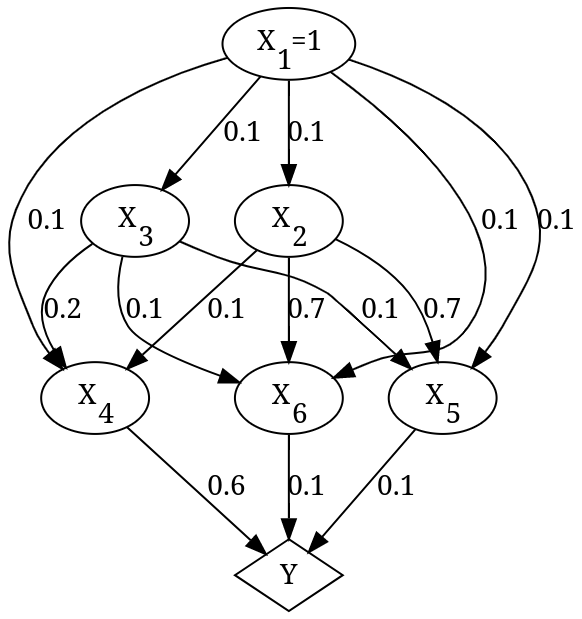}
\caption{Structure and Parameters of $G_5$.}
\label{Fig.twolayer} 
\end{figure}

\begin{figure}[htb] 
\centering 
\includegraphics[width=0.9\textwidth]{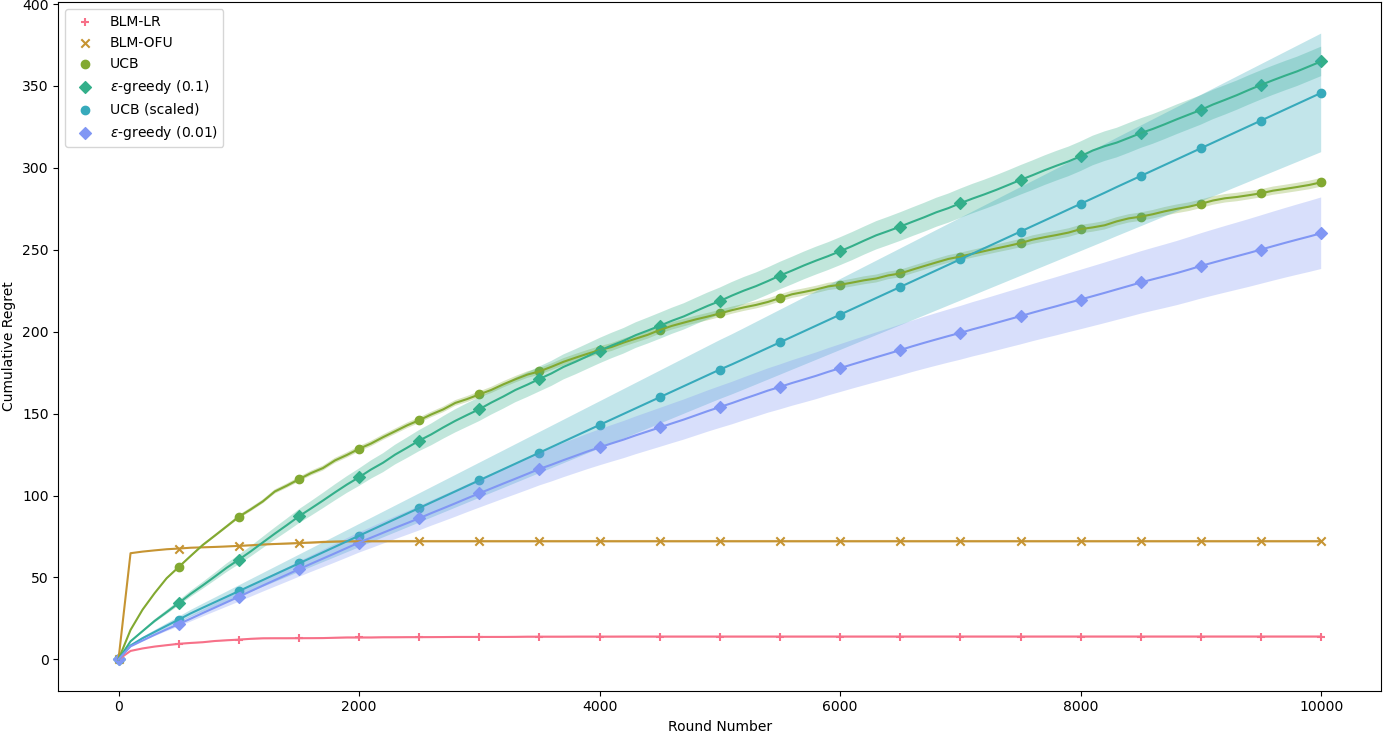}
\caption{Regrets of Algorithms Run on $G_5$.}
\label{Fig.simulation5} 
\end{figure}

Figure~\ref{Fig.simulation5} shows that regrets of BLM-LR and BLM-OFU are much smaller than all settings of UCB and $\epsilon$-greedy algorithms not only on parallel graphs but also on more general graphs. Furthermore, because the difference between expected payoffs of the best intervention and other interventions is larger than that in $G_1$, BLM-LR estimates the parameters accurately enough even faster. Hence, BLM-LR achieves even much smaller regret than BLM-OFU.

\end{document}